%% file: iclr2027_conference.tex
\documentclass{article} 
\usepackage{iclr2027_conference,times}

\input{math_commands.tex}

\usepackage[utf8]{inputenc} 
\usepackage[T1]{fontenc}    
\usepackage{hyperref}       
\usepackage{url}            
\usepackage{booktabs}       
\usepackage{amsfonts}       
\usepackage{nicefrac}       
\usepackage{microtype}      
\usepackage{xcolor}         
\usepackage{pdfpages}
\usepackage{graphicx}
\usepackage{placeins}
\usepackage{float}
\usepackage{standalone}
\usepackage{tikz}
\usepackage{subcaption}
\usepackage{wrapfig}
\usepackage{titletoc}
\usepackage{colortbl}
\definecolor{ourscolor}{gray}{0.93}
\newcommand{\ms}[2]{#1{\scriptsize\,$\pm$#2}}    
\newcommand{\best}[1]{\textbf{#1}}
\newcommand{\second}[1]{\underline{#1}}          
\newlength{\panelw}
\newcommand{\rowlab}[1]{\rotatebox[origin=c]{90}{\scriptsize #1}}
\newcommand{\tsnecell}[2]{\includegraphics[width=\panelw]{figures/plots/tsne_#1_#2.png}}
\newcommand{\denscell}[2]{\includegraphics[width=\panelw]{figures/plots/dens_#1_#2.png}}
\usetikzlibrary{arrows.meta,positioning,fit,backgrounds,calc}
\newcommand{\snrfigure}[2][\linewidth]{\begingroup
  \def\snrdataset{#2}\includestandalone[width=#1]{figures/SNR_figure}\endgroup}

\definecolor{cS}{HTML}{534AB7}
\definecolor{cC}{HTML}{0F6E56}
\definecolor{cN}{HTML}{5F5E5A}
\usepackage{pgfplots}\pgfplotsset{compat=1.18}
\usepackage{multirow}
\definecolor{colOurs}{HTML}{D62728}   
\definecolor{colWave}{HTML}{1F77B4}   
\definecolor{colFlow}{HTML}{2CA02C}   
\definecolor{colFour}{HTML}{FF7F0E}   
\definecolor{colDTS} {HTML}{9467BD}   
\definecolor{colSig} {HTML}{7F7F7F}   

\definecolor{cNeutral}{HTML}{5F5E5A}
\definecolor{cCoef}{HTML}{0F6E56}
\definecolor{cModel}{HTML}{534AB7}
\definecolor{cA}{HTML}{BA7517}

\newcommand{\snrnoy}{\def\snrylabel{}}
\newcommand{\snrnolegend}{\def\snrlegend{0}}
\newcommand{\snrtightlegend}{\def\snrlegendsep{0.12cm}\def\snrlegendsample{0.3cm}}
  
\newcommand{\pmstd}[1]{\,{\scriptsize$\pm$#1}}

\tikzset{
  stage/.style={draw, rounded corners=3pt, line width=0.4pt,
                minimum height=11mm, align=center, font=\small, inner sep=3pt},
  neutral/.style={stage, draw=cNeutral!70, fill=cNeutral!8},
  arr/.style={-{Stealth[length=2mm]}, line width=0.4pt, draw=cNeutral},
  attn/.style={-{Stealth[length=2mm]}, line width=0.6pt, draw=cModel!85},
  zoom/.style={draw=cNeutral!45, dashed, rounded corners=4pt, line width=0.4pt},
  lbl/.style={font=\footnotesize, inner sep=1.5pt},
}

\usepgfplotslibrary{groupplots}
\pgfplotsset{
  plotstyle/.style={
    line width=0.7pt,
    mark size=1.6pt,
    every mark/.append style={solid},
  },
  errbars/.style={
    error bars/y dir=both,
    error bars/y explicit,
    error bars/error bar style={black, line width=0.4pt},
    error bars/error mark options={rotate=90, black, mark size=1pt, line width=0.4pt},
  },
}

\makeatletter
\providecommand{\@trackname}{}
\makeatother

\usepackage{hyperref}
\usepackage{url}
\usepackage{cleveref}
\usepackage{amsmath, amssymb, amsthm, mathtools, bm}
\usepackage{algorithm}
\usepackage{algpseudocode}
\usepackage{cleveref}
\usepackage{enumitem}

\theoremstyle{plain}
\newtheorem{theorem}{Theorem}[section]
\newtheorem{proposition}[theorem]{Proposition}

\theoremstyle{definition}
\newtheorem{definition}[theorem]{Definition}

\theoremstyle{remark}
\newtheorem{remark}[theorem]{Remark}

\newcommand{\N}{\mathcal{N}}
\newcommand{\Z}{\mathbb{Z}}
\newcommand{\pwav}{p_{\mathrm{wav}}}  
\newcommand{\W}{\mathbf{W}}           
         
\newcommand{\vth}{v_\theta}
\newcommand{\dd}{\mathrm{d}}
\newcommand{\diag}{\operatorname{diag}}
\newcommand{\Id}{\mathbf{I}}
\newcommand{\DWT}{\mathrm{DWT}}
\newcommand{\IDWT}{\mathrm{IDWT}}

\title{Wavelet Flow Matching for Time Series}

\author{\normalfont
\begin{tabular*}{\dimexpr\textwidth-2\tabcolsep\relax}[t]{@{\extracolsep{\fill}}lll@{}}
\textbf{Lucas Poinsignon}\thanks{Equal contribution. Correspondence to \texttt{\{lpoinsignon,jogoncalves,sruiperez\}@ethz.ch}.}
  & \textbf{Jorge da Silva Gonçalves}\footnotemark[1] & \textbf{Samuel Ruip\'erez-Campillo}\footnotemark[1] \\
Dept. of Computer Science & Dept. of Computer Science & Dept. of Computer Science \\
ETH Zurich, Switzerland & ETH Zurich, Switzerland & ETH Zurich, Switzerland \\[17pt]
& \textbf{Julia E. Vogt} & \\
& Dept. of Computer Science & \\
& ETH Zurich, Switzerland &
\end{tabular*}}

\newlength{\qpanelsep}
\newlength{\qlabelw}
\newlength{\qpanelw}
\newsavebox{\qdenslab}
\newlength{\qdensh}
\newcommand{\densylabel}[1]{%
  \settoheight{\qdensh}{\includegraphics[width=\qpanelw,clip,trim=20bp 0bp 5bp 0bp]{#1}}%
  \sbox{\qdenslab}{\resizebox{!}{0.40\qdensh}{%
      \rotatebox[origin=c]{90}{Probability Density}}}%
  \makebox[\qlabelw][c]{%
    \raisebox{-\dimexpr 3pt+0.5\qdensh+0.5\ht\qdenslab-0.5\dp\qdenslab\relax}%
      [0pt][0pt]{\usebox{\qdenslab}}}}
\newcommand{\qpanel}[2]{%
  \begin{subfigure}[t]{\qpanelw}%
    \centering
    \includegraphics[width=\linewidth]{figures/plots/tsne_#1.png}\\[2pt]%
    \includegraphics[width=\linewidth,clip,trim=20bp 0bp 5bp 0bp]{figures/plots/dens_#1.png}%
    \caption{#2}%
  \end{subfigure}%
}

\newlength{\rowlabw}
\newlength{\pdlabw}
\newlength{\gridh}
\newsavebox{\gridlabbox}
\newcommand{\setgridh}[1]{%
  \settoheight{\gridh}{\includegraphics[width=\panelw,clip,trim=20bp 0bp 5bp 0bp]{#1}}}
\newcommand{\vcenterlab}[2]{%
  \sbox{\gridlabbox}{#2}%
  \makebox[#1][c]{%
    \raisebox{\dimexpr0.5\gridh-0.5\ht\gridlabbox+0.5\dp\gridlabbox\relax}%
      [0pt][0pt]{\usebox{\gridlabbox}}}}
\renewcommand{\rowlab}[1]{%
  \vcenterlab{\rowlabw}{\rotatebox[origin=c]{90}{\scriptsize #1}}}
\newcommand{\pdlab}{%
  \vcenterlab{\pdlabw}{\resizebox*{!}{0.80\gridh}{%
      \rotatebox[origin=c]{90}{Probability Density}}}}
\renewcommand{\denscell}[2]{%
  \includegraphics[width=\panelw,clip,trim=20bp 0bp 5bp 0bp]{figures/plots/dens_#1_#2.png}}
      
\newcommand{\settsneh}[1]{\settoheight{\gridh}{\includegraphics[width=\panelw]{#1}}}

\iclrfinalcopy 
\begin{document}

\maketitle
\fancyhead{} 

\begin{abstract}

Synthetic time series are increasingly used for data augmentation, privacy-preserving data sharing, and downstream model development, yet faithfully reproducing both multi-scale temporal structure and cross-channel dependencies remains challenging. We study multivariate time-series generation through flow matching in the wavelet domain. By operating on multilevel discrete wavelet coefficients rather than directly in the time domain, the model represents coarse structure and progressively finer details at separate scales. Their naturally different variances further induce an implicit coarse-to-fine generative process without requiring an explicit multi-scale schedule. Since the transform acts independently on each channel, we pair it with a channel-token transformer whose attention directly models cross-channel dependencies. Across seven benchmark datasets and four sequence lengths, our method is best or tied on a majority of dataset–metric combinations, with the largest and most consistent improvements in Context-FID and discriminative score.

\end{abstract}

\section{Introduction}
\label{sec:introduction}

Time series are central to domains including energy, finance, neuroscience, healthcare, and robotics, yet real data can be scarce, expensive to collect, privacy-sensitive, or imbalanced. Generative models address this by learning the data distribution and producing synthetic samples
\citep{yoon2019timegan,desai2021timevae,yuan2024diffusionts}. Multivariate time-series generation is challenging for two structural reasons: (a) signals contain structure at multiple temporal scales, from slow trends to high-frequency transients, and (b) their channels can exhibit strong and heterogeneous dependencies. A successful generative model must capture both.

\paragraph{Representation and generative dynamics.}
Most diffusion models for time series operate directly in the sampled time domain, while recent work has shown benefits from generating in transformed representations such as Fourier coefficients, log-signature embeddings, and wavelet coefficients \citep{crabbe2024fourierdiffusion,barancikova2025sigdiffusion,wang2025waveletdiff,ruiperezcampillo2026cyclostationaryphaseconditioningmedical}. The discrete wavelet transform (DWT) is particularly appealing because it is invertible, localized jointly in time and frequency, and provides an explicit multi-resolution representation of the signal
\citep{mallat1989theory,daubechies1992ten}. In parallel, flow matching \citep{lipman2023flow,liu2023rectified,albergo2023building} provides a simple continuous transport between noise and data and supports deterministic ODE sampling. Existing time-series flow-matching models largely operate in the time domain~\citep{hu2024flowts}, leaving open how a multi-resolution representation changes the structure of the generative flow.

\paragraph{This work.}
We study flow matching in the wavelet domain for unconditional generation of multivariate time series. Each channel is mapped to multilevel DWT coefficients, and a single flow transports isotropic Gaussian noise to their joint distribution. Crucially, although all coefficients follow the same linear interpolation schedule, natural differences in wavelet-level variance induce different signal-to-noise trajectories: coarse, high-variance levels become data-dominated earlier than finer levels. This yields an implicit coarse-to-fine generative process without an explicit scale-dependent schedule. The wavelet representation therefore does more than reorganize the signal: it changes the effective generative dynamics under an otherwise unchanged probability path. Equalizing the level-wise variances collapses these differences in SNR crossing times and substantially degrades generation quality, supporting this mechanism.

The representation also suggests a simple division of labor for the velocity network. The DWT organizes within-channel temporal structure across scales but leaves cross-channel dependencies untouched. We therefore represent each channel by its complete wavelet coefficient stack and apply self-attention across channels. Across seven datasets and four sequence lengths, the resulting model achieves the strongest overall performance among five recent baselines, with particularly consistent gains in Context-FID and discriminative score.
Our contributions are summarized as:
\begin{itemize}[nosep, leftmargin=*]
\item We introduce wavelet flow matching, a multiresolution framework for multivariate time-series generation with a channel-token velocity network tailored to wavelet coefficients.
\item We show that natural differences in wavelet-level variance induce distinct SNR crossing times under a shared linear probability path, which we probe through a targeted variance-equalization intervention.
\item Across seven datasets and four sequence lengths, our model achieves the strongest overall performance among five recent baselines.
\end{itemize}

\begin{figure}[t]
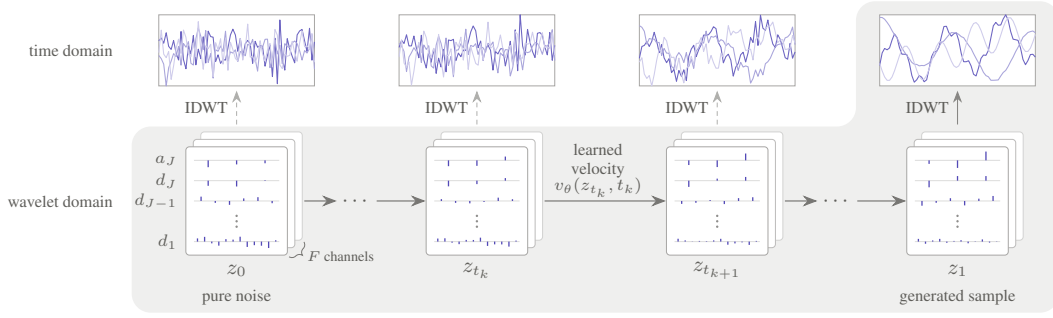

  \centering
  \includestandalone[width=\textwidth]{figures/figure1_new}
  \caption{
Wavelet-domain flow matching. Each channel is mapped to multilevel DWT coefficients, and a single learned flow $\vth$ transports Gaussian noise $z_0$ jointly across all levels and channels. The inverse DWT maps the final coefficients $z_1$ back to the time domain.
}
  \label{fig:pipeline}
\end{figure}


\section{Related work} \label{sec:related_work}
\paragraph{Time-series diffusion and flow matching.}
Denoising diffusion~\citep{sohl2015deep,ho2020ddpm,song2021sde} has become a prominent framework for time-series modeling, with applications spanning forecasting, imputation, generation, and denoising~\citep{tashiro2021csdi,alcaraz2023sssd,kollovieh2023tsdiff,naiman2024imagentime,ruiperez2025physics,ruiperez2026antithetic}. For unconditional generation, Diffusion-TS~\citep{yuan2024diffusionts} combines a transformer denoiser with explicit trend and seasonality decomposition and a Fourier reconstruction loss, but requires iterative reverse-time sampling. Flow matching~\citep{lipman2023flow,liu2023rectified,albergo2023building} instead learns a continuous transport between noise and data without requiring a diffusion noise schedule and supports deterministic ODE sampling~\citep{esser2024sd3}. Recent time-series applications include FlowTS~\citep{hu2024flowts}, which trains a rectified flow directly on the sampled signal, and TSFlow~\citep{kollovieh2025tsflow}, which combines flow matching with Gaussian-process priors for forecasting. Both operate in the time domain. 

\paragraph{Generation in transformed domains.}
Generative models have increasingly operated in transformed representations, including Fourier coefficients~\citep{alaa2021fourierflows,crabbe2024fourierdiffusion}, log-signature embeddings~\citep{barancikova2025sigdiffusion}, and wavelet coefficients. Unlike Fourier coefficients, wavelets are localized jointly in time and scale, making them well suited to transient multi-resolution structure; they have been used in image generation for scale-factorized sampling and accelerated diffusion~\citep{guth2022wsgm,phung2023waveletdiffusion}. For time series, WaveletDiff~\citep{wang2025waveletdiff} applies denoising diffusion with dedicated transformers for each wavelet level and cross-level attention. In contrast, we use a single shared flow over all levels and a channel-token velocity network with self-attention across the original channels. DSFM~\citep{tew2026dsfm} also combines wavelets with flow matching, but further applies a blockwise DCT and processes the coefficients as an image with a vision backbone for fMRI. Our framework instead operates directly on multilevel DWT coefficients and is evaluated across heterogeneous multivariate time-series domains.

\paragraph{Coarse-to-fine generation.}
Diffusion models have long been observed to recover global structure before local detail~\citep{choi2022perception}, and several approaches impose such hierarchies explicitly through frequency-dependent generation \citep{lee2022progressive}, hierarchical representations \citep{da2025treediffusion}, or multi-scale wavelet constructions \citep{guth2022wsgm}. More recently, spectral analyses have connected this behavior to frequency-dependent signal-to-noise trajectories arising from unequal signal energy~\citep{falck2025fourier}. Our setting differs in that the hierarchy emerges across discrete wavelet resolution levels under a single shared linear probability path: their natural variance differences induce distinct SNR crossing times, which we directly probe by equalizing the level-wise variances.

We provide a broader discussion of related work, including additional time-series generative models, transformer tokenization strategies, and learnable wavelet transforms, in Appendix~\ref{app:ext_related_work}.

\section{Method}
\label{sec:method}

We consider unconditional generation of multivariate time series $x\in\R^{T\times F}$ drawn from an unknown distribution $\pdata$, where $T$ is the sequence length and $F$ the number of channels.
Our design follows a division of labor: a channel-wise wavelet transform exposes temporal structure at multiple scales. The learned velocity field, in turn, models the dependencies between channels that the wavelet transform leaves untouched.

\subsection{Flow Matching in the Wavelet Domain}
\label{sec:wavelet_flowmatching}

\paragraph{Wavelet representation.}
We apply a $J$-level discrete wavelet transform (DWT) \citep{mallat1989theory} independently to each channel. For channel $f$, the transform yields one coarse approximation $a_J$ and detail coefficients $d_j \in \R^{L_j}$ for $j=1,\ldots,J$, whose length halves at each stage, $L_j=T/2^j$. We relabel the final approximation as $d_{J+1}:=a_J$ and set $L_{J+1}:=L_J$, then concatenate all levels as
\begin{equation}
\label{eq:concat_coefs}
c_{\cdot,f}=\operatorname{concat}\!\left(d_{J+1},d_J,\ldots,d_1\right)\in\R^T.
\end{equation}
Because the DWT is invertible and produces exactly $T$ coefficients per channel, it preserves the dimensionality of the input. Let $\W\in\R^{T\times T}$ denote the corresponding linear transform, applied independently to each channel. For the full multivariate sequence, we then write
\begin{equation}
\label{eq:wav_transform}
c=\W x\in\R^{T\times F},
\qquad
x=\W^{-1}c.
\end{equation}
Further details on the DWT, including the filter-bank definitions, boundary handling, and reconstruction, are given in Appendix \ref{sec:background:wavelets}. We additionally investigate learning the wavelet transform jointly with the generative model. To this end, we develop a differentiable parametrization of compactly supported orthonormal wavelet filter banks that preserves perfect reconstruction throughout training. The construction is described in \Cref{sec:method:learnable}, with supporting derivations in Appendix \ref{app:lattice}.

\paragraph{Flow matching in coefficient space.}
Rather than model $\pdata$ directly, we model its pushforward through the wavelet transform, i.e., its induced distribution of coefficients,
\begin{equation}
\label{eq:target}
    \pwav = \W_{\#}\pdata .
\end{equation}
Generation proceeds by transporting Gaussian noise to the wavelet-coefficient distribution through an ordinary differential equation. A velocity field $v:\R^{T\times F}\times[0,1]\to\R^{T\times F}$ defines this transport as
\begin{equation}
\label{eq:ode}
\frac{\dd}{\dd t}z_t=v(z_t,t),
\qquad
z_0\sim\N(0,\Id).
\end{equation}
Integrating \Cref{eq:ode} from $t=0$ to $t=1$ yields coefficients $z_1$ approximately distributed as $\pwav$, which are mapped back to the time domain as $\hat{x}=\W^{-1}z_1$ (see \Cref{fig:pipeline}).

Flow matching~\citep{lipman2023flow,liu2023rectified,albergo2023building} learns the velocity field without solving the ODE in \Cref{eq:ode} during training: one fixes a path from noise to data and trains the network to predict the velocity along it. Given wavelet coefficients $c\sim\pwav$ and independent Gaussian noise $z_0\sim p_0:= \mathcal{N}(0,I)$, we use the linear path
\begin{equation}
\label{eq:linear-path}
z_t=t\,c+(1-t)z_0,
\qquad
\frac{\dd}{\dd t}z_t=c-z_0,
\end{equation}
so each noise-data pair is connected by a straight line with constant conditional velocity $c-z_0$. Thus, $z_t$ is a coefficient array in which every level is partially noised. We therefore train $\vth$ by minimizing
\begin{equation}
\label{eq:rf-loss}
\E_{t\sim\pi,\;c\sim\pwav,\;z_0\sim p_0}
\left[
\left\|
\vth(z_t,t)-(c-z_0)
\right\|_2^2
\right],
\end{equation}
where $\pi$ denotes the training-time distribution on $[0,1]$. Since the target $c-z_0$ is not uniquely determined by $(z_t,t)$, the population minimizer of the squared loss is the conditional expectation $\E[c-z_0\mid z_t,t]$, whose marginal velocity field generates the prescribed probability path~\citep{lipman2023flow}.


\subsection{Coarse-to-Fine Flow}
\label{sec:method:coarse_to_fine}

The linear path of \Cref{eq:linear-path} treats all coefficients with the same interpolation schedule, but their signal-to-noise ratios evolve differently across wavelet levels. Let $I_j$ denote the coefficient indices corresponding to level $j$, such that $c_{\cdot,f}[I_j]=d_j$, and let $\sigma_j^2$ denote the variance of the data coefficients at that level. Restricted to level $j$, the interpolant reads
\begin{equation}
\label{eq:level-path}
z_t[I_j]=t\,c[I_j]+(1-t)\,z_0[I_j],
\end{equation}
where the entries of $z_0[I_j]$ have unit variance while those of $c[I_j]$ have variance $\sigma_j^2$. The signal-to-noise ratio of level $j$ along the path is therefore
\begin{equation}
\label{eq:snr}
\mathrm{SNR}_j(t)=\frac{t^2\sigma_j^2}{(1-t)^2},
\qquad j=1,\ldots,J+1,
\end{equation}
which increases monotonically from $0$ to $\infty$ and crosses one at $t_j=1/(1+\sigma_j)$. Before $t_j$, the level is dominated by noise; after $t_j$, it is dominated by data. \Cref{fig:snr-levels} illustrates how differences in $\sigma_j$ translate into different crossing times across wavelet levels.

\begin{wrapfigure}{r}{0.50\columnwidth}
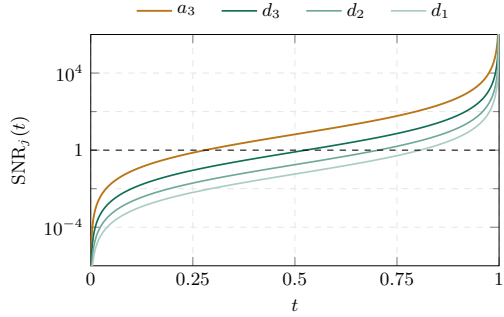

    \vspace{-0.5cm}
    \centering
    \begingroup
  \def\snrdataset{ETTh1}\includestandalone[width=0.97\linewidth]{figures/SNR_figure}\endgroup
    \vspace{-0.2cm}
    \caption{Level-wise SNR along the linear probability path on ETTh1. Coarse levels cross $\mathrm{SNR}=1$ earlier than fine levels.}
    \label{fig:snr-levels}
    \vspace{-0.7cm}
\end{wrapfigure}
Across our datasets, the wavelet coefficients have systematically larger variance at coarser levels. Consequently, coarse levels have larger $\sigma_j$ and become data-dominated earlier along the probability path, whereas fine-scale coefficients remain noise-dominated until later. We observe this ordering across all seven datasets; \Cref{app:coefstats} reports full level-wise coefficient variances, crossing times, and SNR trajectories. Thus, although every coefficient follows the same linear probability path, the wavelet representation induces an implicit coarse-to-fine progression.\looseness-1

Importantly, this progression is not imposed through a scale-dependent noise schedule or separate generative stages. It arises solely from the natural scale-dependent statistics of the wavelet representation under a shared probability path. Standardizing each level to unit variance removes this mechanism precisely: $\sigma_j=1$ for all $j$, so all crossing times collapse to $t_j=1/2$. We test this intervention directly in \Cref{sec:results:ablations}.


\begin{figure}
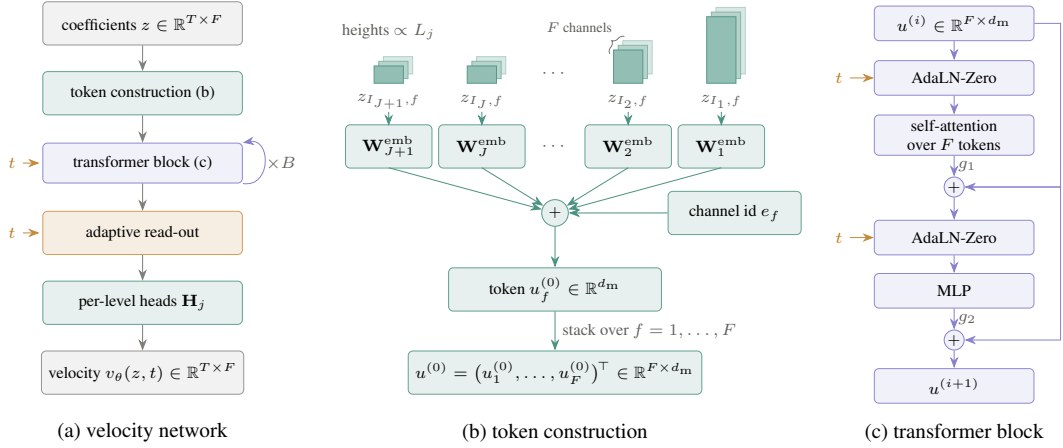

  \centering
  \includestandalone[width=\textwidth]{figures/figure_architecture}
  \caption{
    Velocity-network architecture.
    (a) The network maps wavelet coefficients to a velocity field of the same shape.
    (b) Each channel forms one token by combining level-specific embeddings of its
    wavelet coefficients with a learned channel identity.
    (c) Transformer blocks apply self-attention across channels with AdaLN-Zero
    time conditioning; level-specific heads map the resulting tokens back to
    wavelet coefficients.
    }
  \label{fig:architecture}
\end{figure}

\subsection{Velocity network: channel-token transformer}
\label{sec:method:architecture}

The wavelet transform reorganizes temporal structure within each channel but does not mix information across channels. Motivated by architectures such as iTransformer~\citep{liu2024itransformer}, we therefore introduce a wavelet-aware channel-token transformer whose self-attention operates exclusively across the $F$ channels. Our tokenization is tailored to the wavelet representation: each channel token aggregates level-specific embeddings of its complete multilevel decomposition, and level-specific heads map the resulting representation back to velocity coefficients at each scale. \Cref{fig:architecture} summarizes the architecture, with further architectural details provided in Appendix \ref{app:architecture}.

\paragraph{Token construction.}
The wavelet levels have unequal lengths $L_j$, whereas the transformer operates on tokens of a common width $d_{\mathrm{m}}$. We therefore embed each level separately with a learned linear map $\mathbf{W}^{\mathrm{emb}}_j:\R^{L_j}\to\R^{d_{\mathrm{m}}}$. Writing $z_{I_j,f}\in\R^{L_j}$ for level $j$ of channel $f$, the initial token is
\begin{equation}
\label{eq:token}
u^{(0)}_f=\sum_{j=1}^{J+1}\mathbf{W}^{\mathrm{emb}}_j z_{I_j,f}+e_f\in\R^{d_{\mathrm{m}}},
\qquad f=1,\dots,F,
\end{equation}
where $e_f\in\R^{d_{\mathrm{m}}}$ is a learned channel identity. All remaining parameters are shared across tokens. Thus, without $e_f$, the velocity field is equivariant under permutations of the channel axis and the generated distribution is exchangeable across channels. We ablate this component in \Cref{sec:results:ablations}.

\paragraph{Transformer and time conditioning.}
The $F$ tokens are processed by $B$ pre-norm transformer blocks with self-attention over channels. The flow time $t$ is encoded with a sinusoidal embedding and injected into each block through AdaLN-Zero conditioning following~\citet{peebles2023dit}. Thus, the wavelet representation handles the within-channel multi-scale structure, while the transformer is dedicated to modeling cross-channel dependencies. 

\paragraph{Adaptive read-out.}
After the final transformer block, an adaptive normalization produces channel representations $h=(h_1,\ldots,h_F)^\top$, which are passed through level-specific linear heads $\mathbf{H}_j:\R^{d_{\mathrm{m}}}\to\R^{L_j}$ to reconstruct the velocity coefficients at each scale:
\begin{equation}
\label{eq:readout}
\bigl[\vth(z,t)\bigr]_{I_j,f}=\mathbf{H}_j h_f\in\R^{L_j},
\qquad j=1,\dots,J+1,\quad f=1,\dots,F.
\end{equation}
Concatenating these outputs over the coefficient index sets $I_j$ recovers a velocity array $\vth(z,t)\in\R^{T\times F}$ with the same layout as the input $z$. 


\subsection{Training and Inference}
\label{sec:method:time sampling}

\paragraph{Training-time distribution.}
The distribution $\pi$ of training times controls where along the path the regression objective in \Cref{eq:rf-loss} is evaluated. Following~\citet{esser2024sd3}, adopted for time series by~\citet{hu2024flowts}, we draw
\begin{equation}
\label{eq:logitnorm}
t=\operatorname{sigmoid}\bigl(m+s\varepsilon\bigr),
\qquad
\varepsilon\sim\N(0,1),
\end{equation}
with $\operatorname{sigmoid}$ being the logistic function and $(m,s)=(0,1)$. This yields a symmetric density that concentrates mass on intermediate times and vanishes at the endpoints. The regression is easy near the endpoints, since at $t\approx0$ the input is almost pure noise and at $t\approx1$ it essentially reveals $c$. Meanwhile, intermediate times require resolving $\E[c\mid z_t]$ and dominate the difficulty of the problem. Uniform time sampling is included as an ablation in \Cref{sec:results:ablations}.

\paragraph{Inference.}
At inference, we integrate the learned ODE with the explicit Euler scheme, starting from $z_0$ with i.i.d.\ standard normal entries and stepping along a grid $0=t_0<t_1<\dots<t_N=1$,
\begin{equation}
\label{eq:euler}
  z_{t_{k+1}} \;=\; z_{t_k} \;+\; (t_{k+1}-t_k)\;
  \vth\bigl(z_{t_k},\, t_k\bigr),
  \qquad k = 0,\dots,N-1,
\end{equation}
after which $\hat x = \W^{-1} z_{t_N}$. Since the conditional trajectories in \Cref{eq:linear-path} are straight by construction, a first-order solver with a moderate number of steps is sufficient in practice~\citep{liu2023rectified,esser2024sd3}.
All main experiments use $N=100$ uniformly spaced Euler steps.
We study the effect of the number and placement of integration steps in \Cref{app:sampler}.

\section{Experimental setup}
\label{sec:setup}

\paragraph{Datasets.}
We evaluate on seven standard benchmarks for unconditional time-series generation \citep{yoon2019timegan,yuan2024diffusionts,hu2024flowts,wang2025waveletdiff}: ETTh1 and ETTh2 \citep{zhou2021informer}, hourly load and oil temperature measurements from two electricity transformers; Stocks \citep{yoon2019timegan}, daily Google share prices and trading volume from 2004 to 2019; Exchange \citep{lai2018lstnet}, daily exchange rates of eight currencies against the US dollar; EEG \citep{roesler2013eeg}, a continuous 14-electrode electroencephalogram recording; Energy \citep{candanedo2017energy}, appliance energy consumption together with indoor and outdoor climate readings from a low-energy house; and MuJoCo \citep{todorov2012mujoco}, simulated physics trajectories of a Hopper robot. The datasets span 6--28 channels. Each time-domain channel is standardized, and windows are extracted with stride one at four lengths, $T \in \{24,32,64,128\}$. MuJoCo is the exception, consisting of $10{,}000$ independent rollouts rather than one continuous recording.

\paragraph{Baselines.}
We compare against five recent models. The diffusion-based baselines span different representations: Diffusion-TS \citep{yuan2024diffusionts} operates in the time domain, SigDiffusion \citep{barancikova2025sigdiffusion} in the signature domain, FourierDiffusion \citep{crabbe2024fourierdiffusion} in the frequency domain, and WaveletDiff \citep{wang2025waveletdiff} in the wavelet domain. We additionally compare against the time-domain flow-matching model FlowTS \citep{hu2024flowts}. All baselines are run from their official implementations on the same windows and through the same evaluation pipeline, using the model capacity and sampling budget specified by their released configurations.\looseness-1

\paragraph{Metrics.}
We report four standard benchmark metrics, computed between the real windows and an equally sized set of generated windows, with lower values better throughout. Two metrics assess overall fidelity: the discriminative score \citep{yoon2019timegan} measures how easily real and generated samples can be distinguished by a post-hoc classifier, while Context-FID \citep{jeha2022psagan} measures how closely their distributions match in a learned TS2Vec representation \citep{yue2022ts2vec}. The correlational score \citep{liao2020sigwgan} measures structural fidelity by comparing the cross-channel correlation patterns of real and generated data. The predictive score \citep{yoon2019timegan} measures downstream temporal utility by asking how well a predictor trained on generated data transfers to real data.
For $T=24$, we report mean $\pm$ standard deviation over three independent training runs, evaluating each run five times to account for stochasticity in auxiliary-network training and subsampling. For $T \in \{32, 64, 128\}$, we use one training run with five evaluations.


\paragraph{Implementation and further details.}
We use a three-level periodized DWT ($J=3$), with Daubechies wavelets of order 2 (\texttt{db2}) for $T\leq32$, order 4 (\texttt{db4}) for $T=64$, and order 6 (\texttt{db6}) for $T=128$. The velocity network has token width $d_m=256$, $B=8$ transformer blocks, and 8 attention heads. Further details on the datasets, baseline configurations, evaluation metrics, and implementation are provided in Appendix~\ref{app:setup}.

\begin{table}[t]
\centering
\caption{Unconditional generation on short sequences ($T=24$). Mean$\pm$std over three independent end-to-end training runs with five evaluations each. \textbf{Bold}: best; \underline{underlined}: second best.}
\label{tab:main24}
\small
\setlength{\tabcolsep}{4pt}
\begin{tabular}{@{}l*{7}{c}@{}}
\toprule
 & ETTh1 & ETTh2 & Stocks & Exchange & EEG & Energy & MuJoCo \\
\midrule
\addlinespace[2pt]
\multicolumn{8}{@{}l}{\textit{\hspace{0.2cm}Context-FID} ($\downarrow$)} \\
Diffusion-TS & \ms{0.137}{.012} & \ms{0.056}{.006} & \ms{0.175}{.022} & \ms{0.047}{.005} & \ms{0.018}{.004} & \ms{0.090}{.011} & \ms{0.016}{.002} \\
SigDiffusion & \ms{2.473}{.227} & \ms{1.162}{.051} & \ms{3.285}{.504} & \ms{1.629}{.152} & \ms{0.020}{.003} & \ms{4.241}{.372} & \ms{2.681}{.263} \\
FourierDiffusion & \ms{\second{0.025}}{.003} & \ms{0.026}{.001} & \ms{0.039}{.005} & \ms{0.080}{.021} & \ms{0.015}{.002} & \ms{0.217}{.015} & \ms{0.062}{.006} \\
WaveletDiff & \ms{0.026}{.002} & \ms{0.031}{.002} & \ms{0.020}{.002} & \ms{\second{0.009}}{.000} & \ms{\second{0.008}}{.001} & \ms{0.482}{.042} & \ms{1.180}{.095} \\
FlowTS & \ms{\second{0.025}}{.001} & \ms{\second{0.012}}{.001} & \ms{\second{0.017}}{.005} & \ms{\second{0.009}}{.001} & \ms{\best{0.005}}{.000} & \ms{\second{0.042}}{.004} & \ms{\second{0.012}}{.001} \\
\rowcolor{ourscolor} \textbf{Ours} & \ms{\best{0.005}}{.001} & \ms{\best{0.004}}{.000} & \ms{\best{0.005}}{.002} & \ms{\best{0.004}}{.001} & \ms{0.013}{.002} & \ms{\best{0.010}}{.003} & \ms{\best{0.006}}{.000} \\
\midrule
\multicolumn{8}{@{}l}{\textit{\hspace{0.2cm}Discriminative score} ($\downarrow$)} \\
Diffusion-TS & \ms{0.074}{.005} & \ms{0.040}{.005} & \ms{0.083}{.015} & \ms{0.024}{.002} & \ms{0.312}{.176} & \ms{0.120}{.005} & \ms{0.016}{.001} \\
SigDiffusion & \ms{0.355}{.023} & \ms{0.347}{.093} & \ms{0.364}{.011} & \ms{0.396}{.029} & \ms{0.400}{.197} & \ms{0.500}{.000} & \ms{0.499}{.000} \\
FourierDiffusion & \ms{0.018}{.005} & \ms{0.011}{.007} & \ms{0.018}{.009} & \ms{0.014}{.010} & \ms{0.011}{.008} & \ms{0.118}{.010} & \ms{0.048}{.008} \\
WaveletDiff & \ms{0.014}{.006} & \ms{0.017}{.006} & \ms{\best{0.009}}{.006} & \ms{0.011}{.008} & \ms{\second{0.009}}{.007} & \ms{0.311}{.012} & \ms{0.228}{.010} \\
FlowTS & \ms{\second{0.008}}{.005} & \ms{\second{0.007}}{.005} & \ms{0.014}{.012} & \ms{\second{0.009}}{.002} & \ms{0.106}{.092} & \ms{\best{0.079}}{.015} & \ms{\second{0.011}}{.004} \\
\rowcolor{ourscolor} \textbf{Ours} & \ms{\best{0.005}}{.004} & \ms{\best{0.005}}{.006} & \ms{\second{0.012}}{.009} & \ms{\best{0.005}}{.003} & \ms{\best{0.003}}{.002} & \ms{\second{0.098}}{.018} & \ms{\best{0.007}}{.004} \\
\midrule
\addlinespace[2pt]
\multicolumn{8}{@{}l}{\textit{\hspace{0.2cm}Correlational score} ($\downarrow$)} \\
Diffusion-TS & \ms{0.053}{.012} & \ms{0.110}{.029} & \ms{0.015}{.004} & \ms{0.115}{.042} & \ms{5.164}{.243} & \ms{\second{0.902}}{.044} & \ms{0.283}{.023} \\
SigDiffusion & \ms{0.186}{.006} & \ms{0.413}{.012} & \ms{0.172}{.006} & \ms{0.896}{.039} & \ms{4.777}{.316} & \ms{7.579}{.135} & \ms{1.165}{.025} \\
FourierDiffusion & \ms{0.052}{.006} & \ms{0.088}{.019} & \ms{0.013}{.005} & \ms{0.146}{.097} & \ms{3.585}{.952} & \ms{1.312}{.280} & \ms{0.266}{.026} \\
WaveletDiff & \ms{0.054}{.012} & \ms{0.090}{.021} & \ms{\best{0.005}}{.002} & \ms{0.120}{.013} & \ms{\second{2.270}}{.672} & \ms{1.234}{.168} & \ms{0.284}{.047} \\
FlowTS & \ms{\second{0.044}}{.012} & \ms{\best{0.059}}{.009} & \ms{0.015}{.005} & \ms{\best{0.047}}{.016} & \ms{\best{1.890}}{.636} & \ms{0.957}{.079} & \ms{\best{0.228}}{.032} \\
\rowcolor{ourscolor} \textbf{Ours} & \ms{\best{0.042}}{.015} & \ms{\second{0.066}}{.020} & \ms{\second{0.007}}{.005} & \ms{\second{0.056}}{.017} & \ms{3.870}{.422} & \ms{\best{0.890}}{.176} & \ms{\second{0.246}}{.033} \\
\midrule
\addlinespace[2pt]
\multicolumn{8}{@{}l}{\textit{\hspace{0.2cm}Predictive score} ($\downarrow$)} \\
Diffusion-TS & \ms{0.120}{.003} & \ms{0.108}{.004} & \ms{\best{0.037}}{.000} & \ms{0.045}{.005} & \ms{0.002}{.000} & \ms{\second{0.251}}{.000} & \ms{\best{0.007}}{.000} \\
SigDiffusion & \ms{0.130}{.002} & \ms{0.134}{.004} & \ms{0.041}{.001} & \ms{0.089}{.004} & \ms{\best{0.000}}{.000} & \ms{0.377}{.003} & \ms{0.025}{.001} \\
FourierDiffusion & \ms{\second{0.119}}{.002} & \ms{0.110}{.003} & \ms{\best{0.037}}{.000} & \ms{0.045}{.002} & \ms{\best{0.000}}{.000} & \ms{\second{0.251}}{.000} & \ms{0.010}{.002} \\
WaveletDiff & \ms{\best{0.117}}{.003} & \ms{\best{0.104}}{.003} & \ms{\best{0.037}}{.000} & \ms{0.046}{.003} & \ms{\best{0.000}}{.000} & \ms{\second{0.251}}{.000} & \ms{0.009}{.001} \\
FlowTS & \ms{0.120}{.003} & \ms{\best{0.104}}{.005} & \ms{\best{0.037}}{.000} & \ms{\second{0.043}}{.010} & \ms{0.001}{.000} & \ms{\second{0.251}}{.000} & \ms{0.008}{.001} \\
\rowcolor{ourscolor} \textbf{Ours} & \ms{0.120}{.003} & \ms{0.105}{.004} & \ms{\best{0.037}}{.000} & \ms{\best{0.042}}{.004} & \ms{\best{0.000}}{.000} & \ms{\best{0.250}}{.000} & \ms{\best{0.007}}{.001} \\
\bottomrule
\end{tabular}
\end{table}

\section{Results}
\label{sec:results}
 
\subsection{Time series generation}
\label{sec:results:main}

\paragraph{Overall generation quality.}
\Cref{tab:main24} reports all four metrics on the seven datasets at $T=24$. As our primary benchmark comparison, it reports variability across three independent end-to-end runs, capturing variation from training, generation, and evaluation. Our model is best or tied for best in $18$ of the $28$ dataset--metric combinations, compared with $8$ for FlowTS and $6$ for WaveletDiff.
On both fidelity metrics, the gains are substantial. In terms of Context-FID, our model achieves the lowest score on six of seven datasets; the strongest competing method, FlowTS, has a $2.9\times$ higher score on average. Similarly, for the discriminative score, we are best on five of seven datasets, and the strongest competitor has a $1.7\times$ higher score on average.
Performance on the correlational score, which specifically assesses cross-channel dependencies, is more mixed: we are best on two datasets, while on the remaining five our score is on average $1.4\times$ that of the best-performing method. Finally, the predictive score, our measure of downstream temporal utility, is largely saturated across methods at this length; our model is nevertheless best or tied for best on five of seven datasets. As this metric probes a narrow prediction task rather than general fidelity, we interpret small differences between methods cautiously.

\begin{figure}[t]
  \centering
  \includegraphics[width=\linewidth]{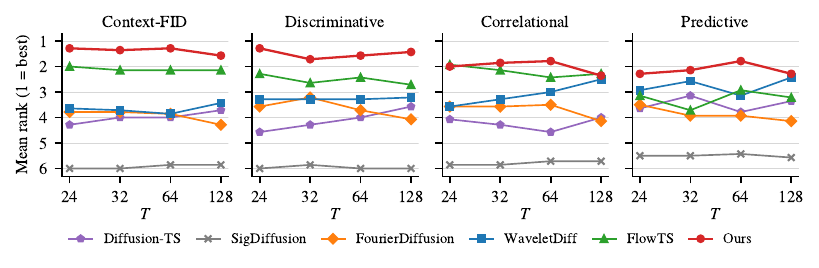}
  \caption{Mean rank of each method over the seven datasets as a function of the window length $T$ \\ (1 = best; tied methods receive the average of their ranks).}
  \label{fig:lengths}
\end{figure}

\begin{figure}
  \centering
\captionsetup[subfigure]{font=scriptsize,skip=1pt}   
  \makebox[\textwidth][c]{%
    \densylabel{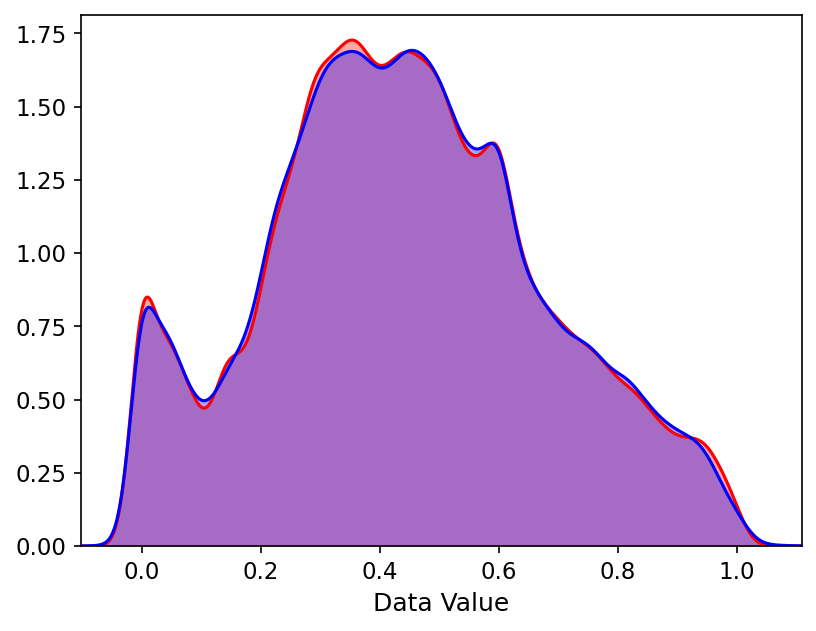}%
    \qpanel{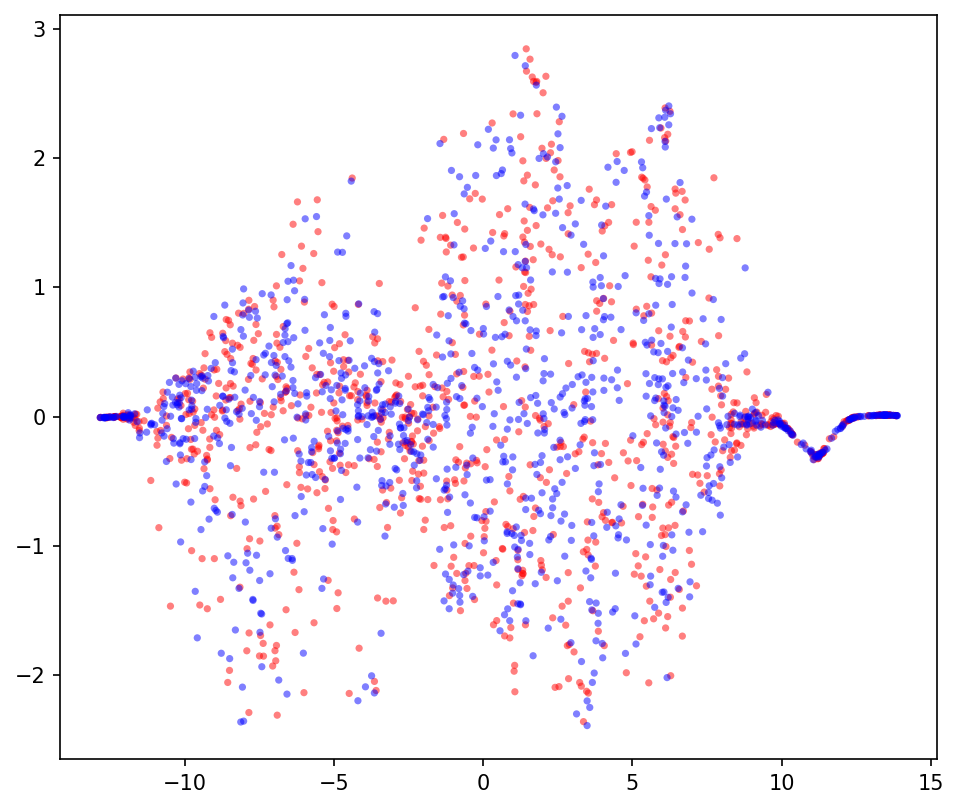}{Ours}\hfill
    \qpanel{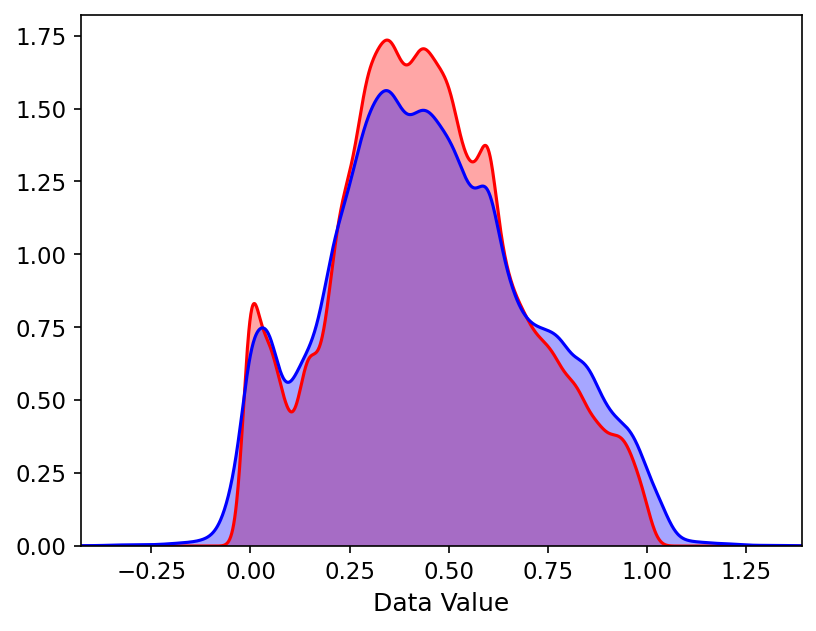}{WaveletDiff}\hfill
    \qpanel{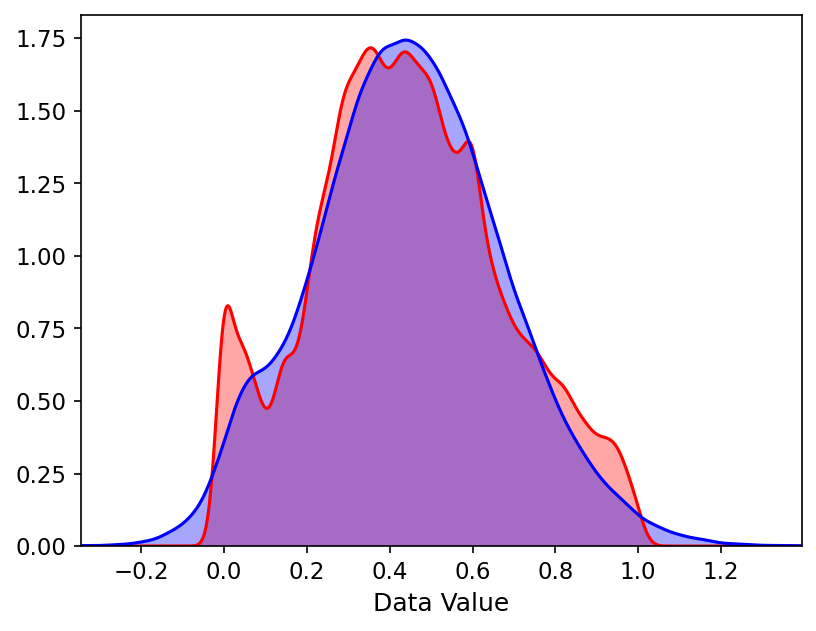}{FlowTS}\hfill
    \qpanel{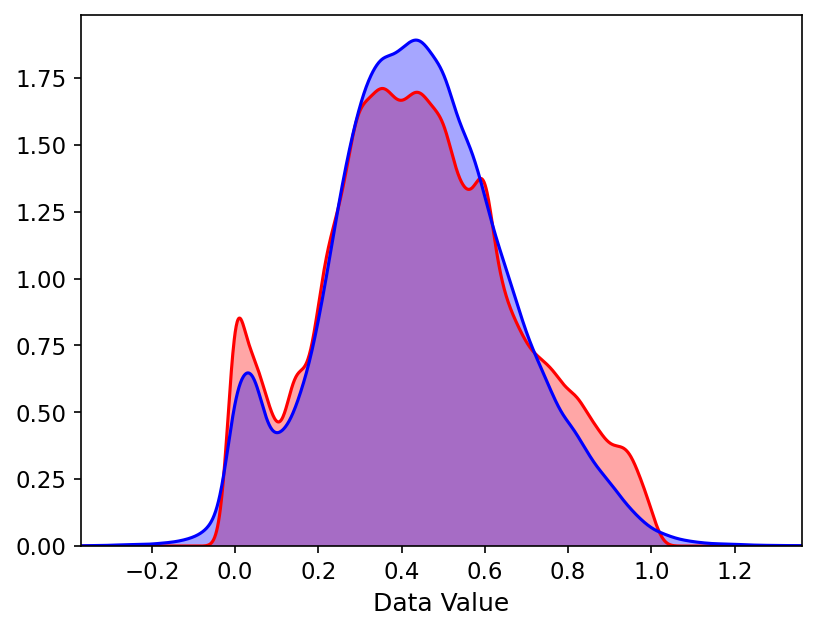}{FourierDiffusion}\hfill
    \qpanel{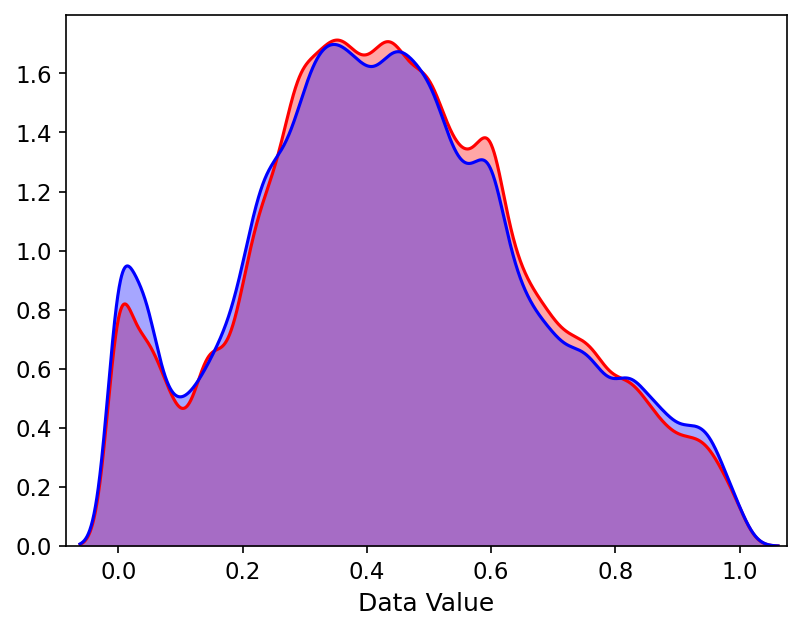}{Diffusion-TS}\hfill
    \qpanel{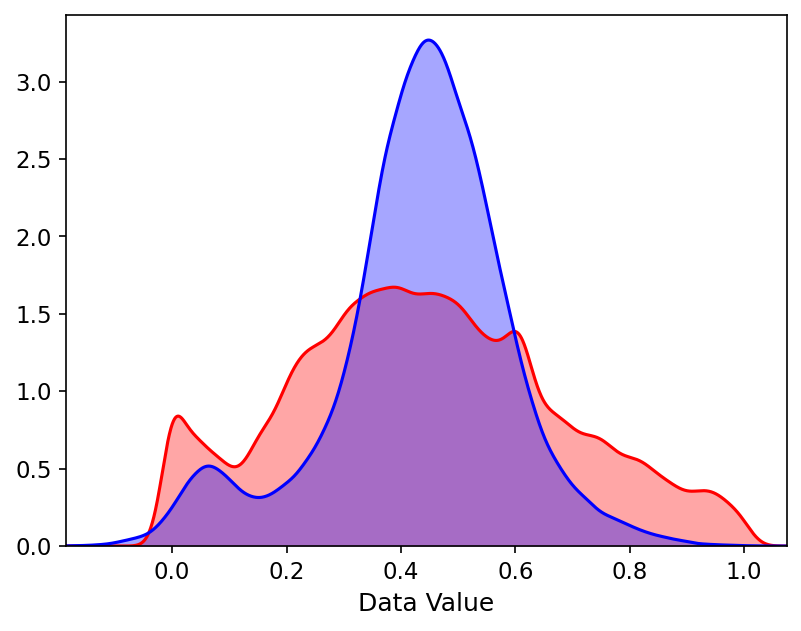}{SigDiffusion}%
  }
  \caption{t-SNE visualization and probability distributions on Energy. Red is for real data, and blue for generated data.}
  \label{fig:qualitative}
\end{figure}

\begin{figure}[t]
  \centering
  \includegraphics[width=\textwidth]{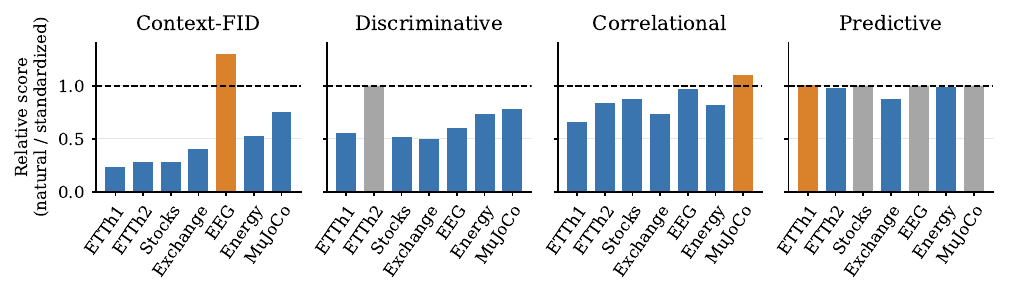}
  \caption{
  Effect of removing the implicit coarse-to-fine progression at $T=24$.
    Bars show the ratio of the score under natural coefficient scaling to that
    under per-level standardization. Values below $1$ favor natural scaling and
    values above $1$ favor standardization; the dashed line denotes equal
    performance.
  }
  \label{fig:ablation-standardize}
\end{figure}

\paragraph{Performance across sequence lengths.}
We next assess whether performance is maintained as the sequence length increases, repeating the full comparison at $T\in\{32,64,128\}$.
\Cref{fig:lengths} summarizes the results by the mean rank of each method across the seven datasets, computed separately for each metric and sequence length. This aggregation avoids directly averaging scores whose scales differ substantially across datasets.
The relative ordering of the methods changes little with sequence length, and our model maintains the best mean rank across all four metrics as $T$ increases.
Full per-dataset results are reported in \Cref{tab:long32,tab:long64,tab:long128}, and \Cref{fig:app-grid} shows the corresponding metric values as a function of sequence length for each dataset.

\paragraph{Qualitative comparison.}
\Cref{fig:qualitative} compares the generated and real distributions on the Energy dataset. Our samples show substantial overlap with the real data in the t-SNE embedding and closely match its empirical density. Corresponding visualizations for all datasets are provided in \Cref{app:qualitative}.



\subsection{Analysis of the Proposed Method}
\label{sec:results:ablations}

We analyze the mechanisms and design choices underlying our method at $T=24$ across all seven datasets. Unless stated otherwise, each experiment changes one component of the default configuration while keeping all other settings fixed.

\paragraph{Coarse-to-fine flow.}
To test whether the implicit coarse-to-fine progression identified in \Cref{sec:method:coarse_to_fine} contributes to generation quality, we standardize each wavelet level to unit variance before flow matching and restore its original scale before the inverse transform (row (b) of \Cref{tab:ablations_main}). This equalizes the level-wise coefficient variances and collapses their SNR crossing times. \Cref{fig:ablation-standardize} shows the ratio between performance under natural scaling and per-level standardization. Natural scaling is better or tied in $25$ of the $28$ dataset--metric combinations. The effect is strongest on the two overall-fidelity metrics: Context-FID improves on six of seven datasets and is $2.5\times$ lower on average, while the discriminative score improves on six datasets and ties on the seventh, with a $1.6\times$ lower score on average. EEG is the main exception: its wavelet levels already have similar variances (see \Cref{tab:level-energy}), so standardization changes their SNR trajectories comparatively little. These results support the hypothesis that the natural scale-dependent coefficient statistics and their induced coarse-to-fine progression improve generation quality.

\paragraph{Method components.}
\Cref{tab:ablations_main} summarizes the remaining one-factor-at-a-time ablations using Context-FID; full results across all four metrics are provided in \Cref{tab:ablations}. Replacing the wavelet transform with the identity (row (c)) generally degrades performance, supporting the benefit of the multiresolution representation itself. 
Removing the channel-identity embedding $e_f$ (row (a)) substantially degrades performance on most datasets, showing that explicit channel identity is important for modeling heterogeneous cross-channel structure. EEG is the main exception, where performance changes little, plausibly because its electrode channels are more homogeneous and thus closer to the exchangeable setting induced by removing $e_f$.
Finally, replacing the logit-normal training-time distribution with uniform sampling (row (d)) worsens Context-FID, while its effect on the other metrics is comparatively small.

\paragraph{Choice of wavelet transform.}
We additionally compare several fixed wavelet families with a learnable orthonormal transform, whose lattice parametrization guarantees an orthonormal filter bank throughout training (\Cref{sec:method:learnable}). No transform consistently dominates across datasets and metrics, and learning the transform provides no systematic improvement over fixed bases. The learned parameters also remain close to their \texttt{db2} initialization. These results suggest that performance depends primarily on the multiresolution representation rather than on the specific wavelet basis. Full results and the evolution of the learned transform are provided in \Cref{tab:ablation-transform} and \Cref{fig:angle-drift}.

Additional sampler sensitivity analyses in \Cref{app:sampler} show that $N=100$ uniformly spaced Euler steps provides a robust default.


\begin{table*}[t]
\centering
\footnotesize
\setlength{\tabcolsep}{4pt}

\resizebox{\textwidth}{!}{%
\begin{tabular}{l *{7}{c}}
\toprule
\textbf{Variant}
& ETTh1 & ETTh2 & Stocks & Exchange & EEG & Energy & MuJoCo \\
\midrule
(a) Without $e_f$
  & 1.488\pmstd{0.099} & 0.977\pmstd{0.168}
  & 1.763\pmstd{0.321} & 0.938\pmstd{0.072}
  & \underline{0.011}\pmstd{0.001} & 2.617\pmstd{0.145}
  & 1.390\pmstd{0.139} \\
(b) Std. levels
  & 0.021\pmstd{0.002} & 0.014\pmstd{0.001}
  & 0.018\pmstd{0.004} & 0.010\pmstd{0.001}
  & \textbf{0.010}\pmstd{0.001} & 0.019\pmstd{0.002}
  & 0.008\pmstd{0.000} \\
(c) Time domain
  & \underline{0.006}\pmstd{0.001} & \underline{0.005}\pmstd{0.001}
  & 0.014\pmstd{0.001} & 0.007\pmstd{0.000}
  & 0.014\pmstd{0.001} & \underline{0.012}\pmstd{0.000}
  & \textbf{0.005}\pmstd{0.000} \\
(d) Uniform $t$
  & 0.007\pmstd{0.000} & 0.008\pmstd{0.001}
  & \underline{0.013}\pmstd{0.001} & \underline{0.005}\pmstd{0.000}
  & 0.014\pmstd{0.002} & 0.016\pmstd{0.001}
  & 0.010\pmstd{0.001} \\
\rowcolor{ourscolor}
Full model
  & \textbf{0.005}\pmstd{0.001} & \textbf{0.004}\pmstd{0.000}
  & \textbf{0.005}\pmstd{0.002} & \textbf{0.004}\pmstd{0.001}
  & 0.013\pmstd{0.002} & \textbf{0.010}\pmstd{0.003}
  & \underline{0.006}\pmstd{0.000} \\
\bottomrule
\end{tabular}%
}

\caption{Context-FID ($\downarrow$) for one-factor-at-a-time ablations at
$T=24$. The full model uses \texttt{db2} coefficients, natural level scaling,
logit-normal time sampling, and channel-identity embeddings $e_f$.
Variants (a)--(d) respectively remove $e_f$, standardize the wavelet levels,
replace the wavelet transform with the identity, and sample $t$ uniformly.
Results are mean $\pm$ standard deviation. Best values are in bold and second-best values are underlined.}
\label{tab:ablations_main}
\end{table*}

\section{Conclusion}
\label{sec:conclusion}

We studied unconditional multivariate time-series generation by flow matching in the wavelet domain. The multilevel wavelet representation separates temporal structure across scales, while the natural differences in level variance induce an implicit coarse-to-fine generation process under a single linear probability path. We pair this representation with a channel-token transformer, so that the transform organizes within-channel temporal structure while attention models the cross-channel dependencies it leaves untouched. Across seven datasets and four sequence lengths, the resulting model matches or improves on the strongest baseline for a majority of dataset-metric pairs, with the largest and most consistent gains on Context-FID. Ablations further show that preserving the natural scale-dependent coefficient statistics is important, supporting the role of the induced coarse-to-fine flow.

\paragraph{Limitations and future work.}
The decomposition depth is currently fixed as a function of window and filter length rather than adapted to the scale structure of each dataset. Moreover, standard time-series generation metrics do not explicitly measure memorization and can saturate on some datasets. Adapting the decomposition to the scale structure of individual datasets and extending the framework to conditional tasks such as forecasting and imputation are natural directions for future work.

\newpage

\subsubsection*{Reproducibility statement}
We provide the information needed to reproduce our results throughout the paper, appendix, and supplementary material. The proposed method, including the training objective, time sampling, and inference procedure, is specified in \Cref{sec:method}, with architectural details and parameter counts in \Cref{app:architecture}. The wavelet construction, boundary handling, and learnable orthonormal transform are described in \Cref{sec:background:wavelets}. Dataset sources and preprocessing are documented in \Cref{sec:setup,app:datasets}, while baseline configurations, evaluation metrics, and implementation details are provided in \Cref{app:baselines,app:metrics,sec:setup:implementation}. Ablation and sampler settings are described in \Cref{sec:results:ablations,app:full_abl,app:sampler}. 

\subsubsection*{Ethics statement}
This work develops a generative model for unconditional multivariate time-series synthesis and evaluates it exclusively on publicly available benchmark datasets. We do not collect new human-subject data, conduct interventions, or deploy the model in real-world decision-making systems.
Synthetic time-series generation can support applications such as data augmentation and data sharing, but generated data should not automatically be assumed to be anonymous, private, or representative of the underlying population. Generative models may memorize training examples or reproduce biases and artifacts present in their training data. Our evaluation focuses on distributional fidelity and downstream utility and does not constitute a formal privacy or memorization analysis. We therefore caution against using the proposed method as a privacy-preserving mechanism without additional safeguards and dedicated privacy evaluation.
Several time-series domains considered in this work, including financial and physiological signals, can arise in high-stakes settings. Synthetic samples produced by the model should not be interpreted as clinically valid measurements, financial advice, or substitutes for domain-specific validation. Any use in such settings should include appropriate expert oversight, validation, and consideration of potential distributional biases and failure modes.

\subsubsection*{AI use statement}

In this work, we used generative AI tools to 
design or provide feedback on research methodology or experiments, 
implement methods,
help develop theoretical models or conceptual frameworks.
We have not used generative AI tools to 
formulate mathematical claims, 
assist in the writing of proofs,
generate synthetic data sets, 
provide critical ingredients for proving mathematical claims, 
propose or refine hypotheses, 
assist with translation, 
clean and reformat datasets,
support qualitative and thematic data analysis, 
interpret results. 
Additionally, we used generative AI tools to 
create or modify scientific figures or images, 
suggest experimental parameters, 
create or edit software code,  
creation of artifacts, 
draft parts of a research paper, 
summarize or analyse existing literature, 
discover research topics or identify gaps, 
brainstorming, sourcing/searching for information, 
edit a research paper to improve readability, 
identify relevant literature, 
suggest a structure for a research paper.
We have reviewed all AI-assisted work. AI-generated code was manually inspected and tested, and AI-generated text and suggested literature were reviewed and verified by the authors before use. We take responsibility for the final content of this work, including text, claims, code, figures, and other artifacts produced with the aid of generative AI.

\subsubsection*{Acknowledgments}
JSG is supported by the StimuLoop grant \#1-007811-002 and the Vontobel Foundation. SRC is supported by the Department of Computer Science at ETH Zurich and reports equity, intellectual property and consulting with Physcade Inc.---no disclosures are conflicting with or related to this work. Computational data analysis was performed at Leonhard Med,\footnote{\footnotesize\url{https://sis.id.ethz.ch/services/sensitiveresearchdata/}} a secure trusted research environment at ETH Zurich.


\clearpage

\bibliographystyle{plainnat}
\bibliography{references}

\clearpage
\appendix

\startcontents[appendices]
\section*{Appendix Contents}
\printcontents[appendices]{}{1}{\setcounter{tocdepth}{2}}


\clearpage

\section{Extended Related Work}
\label{app:ext_related_work}

\paragraph{Adversarial and variational generators.}
The first neural generators of time series operated on the sampled signal and adapted generative adversarial networks~\citep{goodfellow2014gan} to sequential data: C-RNN-GAN~\citep{mogren2016crnngan} and RCGAN~\citep{esteban2017rcgan} pair recurrent generators with recurrent discriminators, while TimeGAN~\citep{yoon2019timegan} adds a learned embedding space and a supervised stepwise loss, and introduced the discriminative and predictive evaluation protocol that remains standard. Later work refined the training signal through progressive growing and self-attention~\citep{jeha2022psagan}, signature-based Wasserstein objectives~\citep{liao2020sigwgan}, or generators built on neural ODEs~\citep{chen2018neuralode} to handle irregular sampling~\citep{jeon2022gtgan}. Variational models~\citep{kingma2014vae} traded sharpness for stability, equipping the decoder with interpretable trend and seasonality blocks~\citep{desai2021timevae}, generating in vector-quantized latent spaces~\citep{lee2023timevqvae}, or imposing Koopman dynamics on the latent process~\citep{naiman2024kovae}. Beyond sample generation, variational autoencoders have also been used in the physiological domain, for instance, to denoise cardiac time series~\citep{ruiperez2024can,ruiperez-campillo_reducing_2026}. 

\paragraph{Choice of tokenization.}
Early transformer forecasters treat each time step as a token and attend over time~\citep{zhou2021informer,wu2021autoformer}. PatchTST~\citep{nie2023patchtst} instead tokenizes contiguous patches and processes channels independently, which improves accuracy and suggests that mixing channels through temporal attention is not where the gains lie. Crossformer~\citep{zhang2023crossformer} attends over time and channels in separate stages, and iTransformer~\citep{liu2024itransformer} inverts the convention entirely, making each channel one token so that attention models only inter-channel dependence while the temporal axis is handled by feed-forward layers. This inversion is a good match for a multi-resolution representation, in which the within-channel structure has already been separated by the transform but the dependence between channels has not been touched. We adopt it here, with each token carrying a channel's full coefficient stack.

\paragraph{Learning the wavelet transform.}
The wavelet transform itself need not be fixed: \citet{recoskie2018sparse} frame the DWT as a modified convolutional network and learn by gradient descent the filters that give a sparse representation of the data, in one dimension and then for images~\citep{recoskie2018filters2d}.  Learned banks have since been placed in pooling layers~\citep{wolter2021adaptive}, distilled from the attributions of a trained network~\citep{ha2021awd}, used for unsupervised monitoring of high-frequency signals~\citep{michau2022despawn} and put inside transformer attention~\citep{zhuang2024wavspa}. For time series, mWDN~\citep{wang2018mwdn} initializes at Daubechies filters and fine-tunes them without constraint; see \citet{ramzi2022waveletsdl} for a survey. These approaches differ in how the admissibility conditions of \Cref{app:wavelets} are imposed. Most add penalty terms to the objective~\citep{recoskie2018sparse,wolter2021adaptive,ha2021awd}, so the conditions are encouraged rather than enforced: \citet{recoskie2018sparse} note that their filters are, in consequence, only approximately wavelet filters, and \citet{zhuang2024wavspa} note that their parametrization preserves the quadrature-mirror relation but not orthogonality. Others build the conditions into the model rather than the objective: \citet{jawali2019wavelet} build orthonormality and vanishing moments into a filter-bank autoencoder and recover the Daubechies filters from Gaussian training data, AdaWaveNet~\citep{yu2024adawavenet} obtains perfect reconstruction structurally from a lifting scheme, though not orthogonality, and \citet{le2024orthlatt} train the rotation angles of a paraunitary lattice~\citep{vaidyanathan1988lattice} inside a ResNet-18, dropping the vanishing moment, which they do not need at a single decomposition level.

\clearpage

\section{Wavelet Transform Background}
\label{sec:background:wavelets}

Here, we provide the wavelet background underlying our method. 
\Cref{app:dwt-background} specifies the periodized DWT, coefficient layout, and reconstruction operator used throughout the paper. 
\Cref{app:wavelets} summarizes the fixed wavelet families considered in our experiments. 
\Cref{sec:method:learnable} then introduces the learnable orthonormal wavelet parametrization used in our ablations, with supporting derivations and proofs deferred to \Cref{app:lattice}.

\subsection{Periodized DWT and coefficient layout} \label{app:dwt-background}
This section provides the discrete wavelet transform details underlying the notation used in \Cref{sec:wavelet_flowmatching}. In the main text, we require only that the transform is linear, channel-wise, critically sampled, and invertible; here we specify the filter-bank construction, boundary handling, coefficient layout, and reconstruction operator.

The discrete wavelet transform (DWT) rewrites a signal as one coarse approximation level, carrying the slow structure, together with a sequence of detail levels, carrying the rapid fluctuations. The transform acts channel by channel, so fix a channel $f$ and set $a_0 = x_{\cdot,f} \in \R^{T}$. A wavelet filter bank consists of a low-pass filter $h$ and a high-pass filter $g$, both supported in $\{0,\dots,q-1\}$. More details about the wavelet families are given in \Cref{app:wavelets}. Mallat's pyramid algorithm~\citep{mallat1989theory} repeats $J$ times a single stage: filter the current approximation $a_j \in \R^{L_j}$ with $h$ and with $g$, then downsample by a factor of two.
Near the end of the window, this filtering reaches up to $q-2$ samples past the last one. We therefore apply the step not to $a_j$ but to its periodic extension
\begin{equation}
\label{eq:per-ext}
  \bar a_j[i] \;=\; a_j\bigl[\,i \bmod L_j\,\bigr], \qquad i \in \mathbb{Z} .
\end{equation}
One step then reads, for $j = 0,\dots,J-1$ and $k = 0,\dots,L_{j+1}-1$,
\begin{equation}
\label{eq:fwt}
  a_{j+1}[k] \;=\; \sum_{m=0}^{q-1} h[m]\, \bar a_j[2k+m],
  \qquad
  d_{j+1}[k] \;=\; \sum_{m=0}^{q-1} g[m]\, \bar a_j[2k+m].
\end{equation}
That is the correlation $\sum_{m} h[m]\,\bar a_j[n+m]$ evaluated at the even shifts $n = 2k$. 
Since the extension is periodic, $k$ runs over a full period and
\begin{equation}
\label{eq:lengths}
  L_{j+1} \;=\; \tfrac{1}{2}\,L_j, \qquad L_0 = T,
  \qquad\text{hence}\qquad L_j \;=\; \frac{T}{2^{\,j}} .
\end{equation}
Exact halving requires $2^{J} \mid T$, and every window length considered here satisfies this.
After $J$ steps, we obtain the full decomposition:
\begin{equation}
\label{eq:dwt}
  \DWT(x_{\cdot,f}) \;=\; (a_J,\, d_J,\, d_{J-1},\, \dots,\, d_1).
\end{equation}
Reconstruction reverses the pyramid: for $j = J,\dots,1$ and $n = 0,\dots,L_{j-1}-1$,
\begin{equation}
\label{eq:iwt}
  a_{j-1}[n] \;=\; \sum_{k=0}^{L_j-1} \Bigl(
    \tilde h\bigl[(n-2k) \bmod L_{j-1}\bigr]\, a_j[k] \;+\;
    \tilde g\bigl[(n-2k) \bmod L_{j-1}\bigr]\, d_j[k] \Bigr),
\end{equation}
with $(\tilde h, \tilde g)$ the synthesis filters of the filter bank (\Cref{app:wavelets}), taken to vanish outside $\{0,\dots,q-1\}$. \Cref{eq:iwt} evaluates the filters at a wrapped index, which coincides with using their periodizations only while the filters are shorter than the stage, $q \le L_{j-1}$; every configuration in this work satisfies this at all stages, and the binding case is quantified in \Cref{app:periodized}. After $J$ steps this returns
\begin{equation}
\label{eq:idwt}
  x_{\cdot,f} \;=\; a_0 \;=\; \IDWT(a_J,\, d_J,\, d_{J-1},\, \dots,\, d_1).
\end{equation}
The final approximation plays the same role in what follows as the details, so we relabel it $d_{J+1} := a_J$ and set $L_{J+1} := L_J$. Concatenating the $J{+}1$ sequences 
\begin{equation}
\label{eq:coeff-layout}
  c_{\cdot,f} \;=\;
  \begin{bmatrix} d_{J+1} \\ d_J \\ \vdots \\ d_1 \end{bmatrix}
  \;\in\; \R^{D},
  \qquad
  D \;=\; \sum_{j=1}^{J+1} L_j .
\end{equation}
Because the levels halve, we have 
\begin{equation}
\label{eq:critical}
  D \;=\; \frac{T}{2^{J}} \;+\; \sum_{j=1}^{J} \frac{T}{2^{j}}
    \;=\; \frac{T}{2^{J}} \;+\; T\bigl(1-2^{-J}\bigr) \;=\; T ,
\end{equation}
so the transform returns as many coefficients as the series has time steps.
We write $I_j$ for the index set of level $d_j$, so that $c_{\cdot,f}[I_j] = d_j$.
Both \Cref{eq:per-ext} and \Cref{eq:fwt} are linear, as is \Cref{eq:iwt}, so the DWT and the IDWT are each a matrix acting on $\R^{T}$. That the second inverts the first is the perfect-reconstruction property: it holds exactly under the admissibility conditions of \Cref{app:wavelets}, as shown in \Cref{app:periodized}, so we write the DWT as $\W \in \R^{T\times T}$ and the IDWT as $\W^{-1}$. The transform acts on each channel independently, so we write $c = \W x \in \R^{T\times F}$ and $x = \W^{-1} c \in \R^{T\times F}$.

\FloatBarrier
\subsection{Wavelet filter families}
\label{app:wavelets}
The DWT of \Cref{sec:background:wavelets} is built from a filter bank: the filters $h,g$ entering \Cref{eq:fwt} and $\tilde h,\tilde g$ entering \Cref{eq:iwt} cannot be arbitrary. This appendix states the two standard admissibility conditions, under which \Cref{eq:iwt} inverts \Cref{eq:fwt} exactly (\Cref{app:periodized}), and lists the four families compared in \Cref{sec:results}; see~\citet{daubechies1992ten,mallat2008wavelet} for complete treatments. Filter names follow PyWavelets~\citep{lee2019pywavelets}.

\paragraph{Orthonormal filter banks.}
A low-pass filter $h$ generates an orthonormal wavelet basis if it satisfies
\begin{subequations}
\label{eq:cqf}
\begin{align}
  \sum_{n} h[n]\,h[n-2k] &\;=\; \delta_{k,0} \quad \forall k\in\Z,
  \label{eq:cqf-shift}\\
  \sum_{n} h[n] &\;=\; \sqrt{2}.
  \label{eq:cqf-norm}
\end{align}
\end{subequations}
The high-pass filter is then $g[n]=(-1)^{n}h[q-1-n]$, which is again supported on $\{0,\dots,q-1\}$, and the synthesis filters of \Cref{eq:iwt} are the analysis ones, $(\tilde h,\tilde g)=(h,g)$.

\paragraph{Biorthogonal filter banks.}
The low-pass filters of a biorthogonal wavelet must satisfy
\begin{equation}
\label{eq:biorth}
  \sum_{n} \tilde h[n]\,h[n-2k] \;=\; \delta_{k,0}\quad\forall k\in\mathbb{Z},
\end{equation}
and the high-pass filters are then given by $g[n]=(-1)^{n}\tilde h[q-1-n]$ and $\tilde g[n]=(-1)^{n}h[q-1-n]$, the same alternating flip as in the orthonormal case. Here the four filters need not have the same length: $q$ denotes the longest and is taken even, and the shorter ones are placed inside $\{0,\dots,q-1\}$ so that \Cref{eq:biorth} holds. The four filters are then biorthogonal, meaning
\begin{equation}
\label{eq:biorth2}
  \sum_{n} \tilde g[n]\,g[n-2k] \;=\; \delta_{k,0},
  \qquad
  \sum_{n} \tilde h[n]\,g[n-2k] \;=\; \sum_{n} \tilde g[n]\,h[n-2k] \;=\; 0 .
\end{equation}

\paragraph{Symmetry.}
A filter $h$ supported on $\{0,\dots,q-1\}$ is symmetric when its taps form a palindrome,
\begin{equation}
\label{eq:sym}
  h[n] \;=\; h[q-1-n], \qquad 0 \le n \le q-1 ,
\end{equation}
so that it delays every frequency equally and does not distort the shape of a feature.

\paragraph{Vanishing moments.}
The filter bank has $p$ vanishing moments when the high-pass filter annihilates polynomials of degree $<p$, that is
\begin{equation}
\label{eq:vm}
  \sum_{n} n^{m}\, g[n] \;=\; 0 , \qquad 0 \le m < p ,
\end{equation}
so detail coefficients are small wherever the signal is locally well approximated by such a polynomial.

\paragraph{Daubechies (\texttt{db}).}
~\Citet{daubechies1988orthonormal} constructed, for every $p$, the orthonormal filter of minimal length with $p$ vanishing moments; its length is $q=2p$. For $p=2$ (\texttt{db2}, $q=4$),
\begin{equation}
\label{eq:db2}
  h = \frac{1}{4\sqrt{2}}\bigl(1+\sqrt{3},\; 3+\sqrt{3},\; 3-\sqrt{3},\; 1-\sqrt{3}\bigr).
\end{equation}

\paragraph{Coiflets (\texttt{coif}).}
Coiflets are orthonormal filters that additionally impose vanishing moments on the low-pass filter: \texttt{coif}$k$ has $2k$ vanishing moments for $g$ and $2k-1$ for $h$, with length $q=6k$. The extra conditions make the filters nearly symmetric and the approximation coefficients close to samples of the signal. For
\texttt{coif1} ($q=6$),
\begin{equation}
\label{eq:coif1}
  h = \frac{1}{16\sqrt{2}}\bigl(1-\sqrt{7},\; 5+\sqrt{7},\; 14+2\sqrt{7},\;
      14-2\sqrt{7},\; 1-\sqrt{7},\; -3+\sqrt{7}\bigr).
\end{equation}

\paragraph{Biorthogonal splines (\texttt{bior}, \texttt{rbio}).}
Apart from the Haar case, no real filter of finite length satisfies both \Cref{eq:cqf,eq:sym}~\citep{daubechies1988orthonormal}. The spline family therefore gives up orthonormality \Cref{eq:cqf} for symmetry: its four filters satisfy the biorthogonality conditions \Crefrange{eq:biorth}{eq:biorth2}, and all satisfy \Cref{eq:sym} on their support. On top of these two conditions, the high-pass filters carry prescribed numbers of vanishing moments: $g$ has $p$ and $\tilde g$ has $\tilde p$, and the resulting filter bank is denoted \texttt{bior}$p$.$\tilde p$. For \texttt{bior2.2} ($q=6$), the analysis pair is
\begin{equation}
\label{eq:bior22}
  h = \frac{1}{4\sqrt{2}}\bigl(-1,\,2,\,6,\,2,\,-1\bigr) \ \text{ on } \{0,\dots,4\},
  \qquad
  g = \frac{1}{2\sqrt{2}}\bigl(1,\,-2,\,1\bigr) \ \text{ on } \{2,3,4\},
\end{equation}
and the synthesis pair is $\tilde h = \tfrac{1}{2\sqrt{2}}(1,2,1)$ on $\{1,2,3\}$ and $\tilde g = \tfrac{1}{4\sqrt{2}}(1,2,-6,2,1)$ on $\{1,\dots,5\}$, all four vanishing elsewhere.

The reverse family \texttt{rbio}$p$.$\tilde p$ is the same construction with the analysis pair $(h,g)$ and the synthesis pair $(\tilde h,\tilde g)$ exchanged: the filters that \texttt{bior} uses in \Cref{eq:fwt} are used in \Cref{eq:iwt} instead, and conversely.

\subsection{Learnable orthonormal wavelets}
\label{sec:method:learnable}

\subsubsection{Requirements}
The filter bank has so far been fixed before training. We propose a way to learn it instead, jointly with the model, in the hope of finding the representation in which the flow is easiest to learn. We restrict the search to orthonormal filter banks (\Cref{app:wavelets}), which gives two things at once: perfect reconstruction, since $\W^{-1}=\W^{\top}$, so no error is introduced when mapping the coefficients back to the time domain, and energy preservation, $\|\W x\|_F=\|x\|_F$, which prevents the transform from making \Cref{eq:rf-loss} small by scaling its output down rather than by fitting the data. \Citet{recoskie2018sparse} penalize $(\|h\|_2-1)^2$ for the same reason, but using a penalty rather than a guarantee. The only freedom left is how the energy is spread across levels.

\subsubsection{The polyphase matrix}
\label{app:polyphase}
Throughout, $h$ is supported on $\{0,\dots,q-1\}$ with $q=2K$ even, the high-pass filter is $g[n]=(-1)^{n}h[q-1-n]$ as in \Cref{app:wavelets}, and $H(z)=\sum_n h[n]z^{-n}$, $G(z)=\sum_n g[n]z^{-n}$ denote the corresponding transfer functions.

\begin{definition}[Polyphase matrix]
\label{def:polyphase}
Splitting a filter into its even- and odd-indexed coefficients gives its two polyphase components. Collecting those of $h$ in the first row and those of $g$ in the second defines the polyphase matrix
\begin{equation}
  E(z) \;=\;
  \begin{pmatrix}
    \sum_n h[2n]z^{-n} & \sum_n h[2n+1]z^{-n}\\[2pt]
    \sum_n g[2n]z^{-n} & \sum_n g[2n+1]z^{-n}
  \end{pmatrix}
  \;\in\; \R^{2\times2}[z^{-1}],
  \label{eq:polyphase}
\end{equation}
whose entries are polynomials of degree at most $K-1$ in $z^{-1}$.
\end{definition}

\begin{remark}
\label{rem:polyphase-bijection}
By construction,
\begin{equation}
  \begin{pmatrix} H(z)\\ G(z)\end{pmatrix}
  \;=\; E(z^{2})\begin{pmatrix} 1\\ z^{-1}\end{pmatrix}
  \quad\Longrightarrow\quad
  \left\{
  \begin{aligned}
    H(z) &= E_{11}(z^{2}) + z^{-1}E_{12}(z^{2}),\\
    G(z) &= E_{21}(z^{2}) + z^{-1}E_{22}(z^{2}).
  \end{aligned}
  \right.
  \label{eq:polyphase-recover}
\end{equation}

\end{remark}
 
\begin{definition}[Paraunitarity]
\label{def:Paraunitarity}
The matrix $E$ is paraunitary if
\begin{equation}
  E(z^{-1})^{\top}E(z) \;=\; \Id
  \qquad\text{on } |z|=1 ,
  \label{eq:paraunitary-def}
\end{equation}
that is, when $E(z)$ is unitary at every point of the unit circle. Since $E$ is square, \Cref{eq:paraunitary-def} is equivalent to $E(z)E(z^{-1})^{\top}=\Id$, the form used below.
\end{definition}

\subsubsection{Lattice parametrization}
Let $E$ be the $2\times2$ polyphase matrix of $(h,g)$, defined in \Cref{app:polyphase}, $R(\phi)=\left(\begin{smallmatrix}\cos\phi & \sin\phi\\-\sin\phi & \cos\phi\end{smallmatrix}\right)$ the rotation of angle $\phi$, and $\Lambda(z)=\diag(1,z^{-1})$ the delay. For filters of length $q=2K$,
\begin{flalign}
  &(h,g) \ \text{satisfies \Cref{eq:cqf-shift}}
   \;\iff\; E(z^{-1})^{\top}E(z)=\Id \ \text{ on } |z|=1,
  &\label{eq:paraunitary}\\[2pt]
  &\iff\; E(z)=\diag(1,-1)\,R(\phi_{K-1})\,\Lambda(z)\,R(\phi_{K-2})\,\Lambda(z)\cdots\Lambda(z)\,R(\phi_{0}),
  &\label{eq:lattice}
\end{flalign}
for some $\phi\in\R^{K}$. \Citet{le2024orthlatt} take these $K$ angles as free parameters. We constrain them: writing $(h_\phi,g_\phi)$ for the filters obtained from $\phi$ through \Cref{eq:lattice}, the second condition reads
\begin{equation}
  (h_\phi,g_\phi) \ \text{satisfies \Cref{eq:cqf-norm}}
  \;\iff\; \sum_{i=0}^{K-1}\phi_i \;\equiv\; \frac{\pi}{4} \pmod{2\pi}.
  \label{eq:vm-angles}
\end{equation}
See \Cref{app:lattice} for details on these equivalences and for the recursions mapping $\phi$ to $(h_\phi,g_\phi)$ and back.

A wavelet filter bank of length $q=2K$ is therefore described by $K-1$ free angles: we parametrize $\phi_0,\dots,\phi_{K-2}$ freely and set $\phi_{K-1}=\pi/4-\sum_{i<K-1}\phi_i$, so every parameter value is an orthonormal bank with one vanishing moment at every step of training, and the angles can be learned jointly with $\theta$. The vanishing moment is what keeps the approximation and detail coefficients apart in frequency. It forces $\sum_n g[n]=0$, so the filter producing $d_j$ in \Cref{eq:fwt} has no response at zero frequency, implying that it passes none of the signal's slow variation, and by \Cref{eq:cqf} the filter producing $a_j$ is then the one that carries it. Without the constraint both filters respond at every frequency, so neither is low- nor high-pass and the two outputs no longer separate the signal by scale. This does not matter when the bank is applied once, which is why \citet{le2024orthlatt} leave all $K$ angles free. We apply it $J$ times, so each level would inherit the same blurred split.


\clearpage

\FloatBarrier 
\section{Supporting Derivations and Proofs}
\label{app:lattice}
 
This section proves \Cref{eq:paraunitary} (orthonormality is paraunitarity) and \Cref{eq:vm-angles} (the normalization is an affine condition on the angles) of \Cref{sec:method:learnable}, and records the lattice factorization \Cref{eq:lattice} of~\citet{vaidyanathan1988lattice}, proving the direction used during training. It also derives the recursion that turns angles into filters and its inverse, and the orthogonality of the periodized transform.
 
\subsection{Orthonormality is paraunitarity}

\begin{proposition}
\label{prop:paraunitary}
$h$ satisfies \Cref{eq:cqf-shift} $\iff$ $E$ is paraunitary.
\end{proposition}
 
\begin{proof}
Splitting each correlation over $n$ into its even and odd indexed terms, the four entries of $E(z)\,E(z^{-1})^{\top}$ collect the even-lag correlations of the two filters:
\begin{equation}
  E(z)\,E(z^{-1})^{\top}
  \;=\; \sum_{k\in\Z}
  \begin{pmatrix}
    \sum_n h[n]\,h[n-2k] & \sum_n h[n]\,g[n-2k]\\[2pt]
    \sum_n g[n]\,h[n-2k] & \sum_n g[n]\,g[n-2k]
  \end{pmatrix} z^{-k}.
  \label{eq:corr-polyphase}
\end{equation}
Since a Laurent series vanishes if and only if all its coefficients do, $E(z)E(z^{-1})^{\top}=\Id$ holds if and only if
\begin{equation}
  \sum_n h[n]h[n-2k] = \sum_n g[n]g[n-2k] = \delta_{k,0},
  \qquad
  \sum_n h[n]g[n-2k] = 0
  \qquad \forall k\in\Z.
  \label{eq:cqf-full}
\end{equation}

($\Rightarrow$) Assume \Cref{eq:cqf-shift}. Using $g[n]=(-1)^{n}h[q-1-n]$ and $(-1)^{n}(-1)^{n-2k}=1$,
\begin{equation*}
  \sum_n g[n]g[n-2k]
  \;=\; \sum_n h[q-1-n]\,h[q-1-n+2k]
  \;=\; \sum_m h[m]\,h[m+2k]
  \;=\; \delta_{k,0}.
\end{equation*}
For the cross term, write $S_k=\sum_n h[n]g[n-2k]=\sum_n(-1)^{n}h[n]h[q-1+2k-n]$ and substitute $m=q-1+2k-n$. As $q$ is even, $(-1)^{n}=(-1)^{q-1+2k-m}=-(-1)^{m}$, so $S_k=-S_k$ and $S_k=0$, and \Cref{eq:cqf-full} holds.

($\Leftarrow$) Immediate, \Cref{eq:cqf-shift} being part of \Cref{eq:cqf-full}.

\end{proof}
 
\subsection{The lattice factorization}
\label{app:lattice-fact}

Throughout, $h$ is supported on $\{0,\dots,2K-1\}$, $g$ is its alternating flip $g[n]=(-1)^{n}h[2K-1-n]$, and $E$ is their polyphase matrix (\Cref{app:polyphase}). For $\phi=(\phi_0,\dots,\phi_{K-1})\in\R^{K}$, let
\begin{equation*}
  E_\phi(z) \;=\; R(\phi_{K-1})\,\Lambda(z)\,R(\phi_{K-2})\cdots\Lambda(z)\,R(\phi_0)
\end{equation*}
be the lattice product of \Cref{eq:lattice}, and let $h_\phi$ be the filter read off its first row through \Cref{eq:polyphase-recover}.

\begin{proposition}[\citealp{vaidyanathan1988lattice}]
\label{prop:lattice}
$E$ is paraunitary if and only if $E(z)=\diag(1,-1)\,E_\phi(z)$ for some $\phi\in\R^{K}$, equivalently, if and only if $h=h_\phi$ for some $\phi\in\R^{K}$.
\end{proposition}

This is the two-channel real FIR case of the lattice factorization theorem for paraunitary systems; we refer to~\citet{vaidyanathan1988lattice} for the proof and record only the direction used during training.

\begin{proof}[Proof of ($\Leftarrow$), the direction used during training]
Every factor of $\diag(1,-1)\,E_\phi$ is paraunitary:
\begin{equation*}
  R(\phi_i)^{\top}R(\phi_i)=\Id,
  \qquad
  \diag(1,-1)^{\top}\diag(1,-1)=\Id,
  \qquad
  \Lambda(z^{-1})^{\top}\Lambda(z)=\diag(1,\,z\cdot z^{-1})=\Id ,
\end{equation*}
the first two being constant orthogonal matrices. 

Paraunitarity is stable under products: if $A(z^{-1})^{\top}A(z)=B(z^{-1})^{\top}B(z)=\Id$, then $(AB)(z^{-1})^{\top}(AB)(z)=B(z^{-1})^{\top}A(z^{-1})^{\top}A(z)B(z)=\Id$. Hence $\diag(1,-1)\,E_\phi$ is paraunitary, so every $\phi\in\R^{K}$ yields an admissible filter bank.
\end{proof}

The converse ($\Rightarrow$), that every admissible bank is reached by some $\phi$, is proved in~\citet{vaidyanathan1988lattice}. Beyond the statement, we use two constructive ingredients.

\paragraph{From angles to filters.}
For $i=0,\dots,K-1$, let $E^{(i)}(z)=R(\phi_i)\,\Lambda(z)\,R(\phi_{i-1})\cdots\Lambda(z)\,R(\phi_0)$ be the partial products of \Cref{eq:lattice}, so $E^{(K-1)}=E_\phi$, and let $(h^{(i)},g^{(i)})$ be the filter pair of $E^{(i)}$ through \Cref{eq:polyphase-recover}. The entries of $E^{(i)}$ have degree at most $i$ in $z^{-1}$, so $h^{(i)}$ and $g^{(i)}$ have length $2(i{+}1)$. Write $c_i=\cos\phi_i$, $s_i=\sin\phi_i$. At $i=0$,
\begin{equation*}
  E^{(0)}(z) \;=\; R(\phi_0) \;=\;
  \begin{pmatrix} c_0 & s_0\\ -s_0 & c_0 \end{pmatrix}
  \quad\xRightarrow{\ \text{\Cref{eq:polyphase-recover}}\ }\quad
  \begin{aligned}
    H^{(0)}(z) &= c_0 + s_0\,z^{-1},\\
    G^{(0)}(z) &= -s_0 + c_0\,z^{-1},
  \end{aligned}
  \qquad
  \begin{aligned}
    h^{(0)} &= (c_0,\,s_0),\\
    g^{(0)} &= (-s_0,\,c_0).
  \end{aligned}
\end{equation*}
For $i\ge1$, peeling one factor off the product,
\begin{equation*}
  E^{(i)}(z)
  \;=\; R(\phi_i)\,\Lambda(z)\,E^{(i-1)}(z)
  \;=\; \begin{pmatrix} c_i & s_i\,z^{-1}\\ -s_i & c_i\,z^{-1}\end{pmatrix}
  \begin{pmatrix}
    E^{(i-1)}_{11}(z) & E^{(i-1)}_{12}(z)\\[2pt]
    E^{(i-1)}_{21}(z) & E^{(i-1)}_{22}(z)
  \end{pmatrix},
\end{equation*}
and feeding the first row through \Cref{eq:polyphase-recover}, where the substitution $z\mapsto z^{2}$ turns the factor $z^{-1}$ into $z^{-2}$,
\begin{align*}
  H^{(i)}(z)
  &= E^{(i)}_{11}(z^{2}) + z^{-1}E^{(i)}_{12}(z^{2})\\
  &= c_i\bigl[E^{(i-1)}_{11}(z^{2}) + z^{-1}E^{(i-1)}_{12}(z^{2})\bigr]
     \;+\; s_i\,z^{-2}\bigl[E^{(i-1)}_{21}(z^{2}) + z^{-1}E^{(i-1)}_{22}(z^{2})\bigr]\\
  &= c_i\,H^{(i-1)}(z) \;+\; s_i\,z^{-2}\,G^{(i-1)}(z),
\end{align*}
and identically for the second row, $G^{(i)}(z) = -s_i\,H^{(i-1)}(z) + c_i\,z^{-2}\,G^{(i-1)}(z)$. Since $z^{-2}$ delays by two samples, matching coefficients of $z^{-n}$ gives for $i=1,\dots,K-1$ 
\begin{equation}
  h^{(0)}=(c_0,\,s_0),
  \quad
  g^{(0)}=(-s_0,\,c_0),
  \qquad
  \begin{aligned}
    h^{(i)}[n] &= c_i\,h^{(i-1)}[n] + s_i\,g^{(i-1)}[n-2],\\
    g^{(i)}[n] &= -s_i\,h^{(i-1)}[n] + c_i\,g^{(i-1)}[n-2].
  \end{aligned}
  \label{eq:lattice-taps}
\end{equation}
and $h_\phi=h^{(K-1)}$. By induction, $g^{(i)}$ is minus the alternating flip of $h^{(i)}$, which is where the $\diag(1,-1)$ in \Cref{prop:lattice} comes from: with $g_\phi$ the alternating flip of $h_\phi$, the convention of \Cref{app:wavelets}, the pair $(h_\phi,g_\phi)$ has polyphase matrix $\diag(1,-1)E_\phi$. 

\begin{remark}
At $K=2$,
\begin{equation}
  h_\phi = (c_1c_0,\; c_1s_0,\; -s_1s_0,\; s_1c_0),
  \label{eq:lattice-taps-K2}
\end{equation}
which at $(\phi_0,\phi_1)=(\frac{\pi}{3},-\frac{\pi}{12})$ returns the Daubechies-2 filter \Cref{eq:db2}, using $\cos\frac{\pi}{12}=\frac{\sqrt6+\sqrt2}{4}$ and $\sin\frac{\pi}{12}=\frac{\sqrt6-\sqrt2}{4}$.
\end{remark}

\paragraph{From filters to angles.}
Fix $i\in\{1,\dots,K-1\}$, the case $i=0$ is treated separately at the end. 
For each $n$, \Cref{eq:lattice-taps} rotates the stage-$(i{-}1)$ pair by $R(\phi_i)$, which is undone by $R(\phi_i)^{-1}=R(\phi_i)^{\top}$:
\begin{equation*}
  \begin{pmatrix} h^{(i)}[n]\\ g^{(i)}[n]\end{pmatrix}
  \;=\; R(\phi_i)
  \begin{pmatrix} h^{(i-1)}[n]\\ g^{(i-1)}[n-2]\end{pmatrix},
  \qquad
  \begin{pmatrix} h^{(i-1)}[n]\\ g^{(i-1)}[n-2]\end{pmatrix}
  \;=\; R(\phi_i)^{\top}
  \begin{pmatrix} h^{(i)}[n]\\ g^{(i)}[n]\end{pmatrix},
\end{equation*}
componentwise
\begin{equation}
  h^{(i-1)}[n] = c_i\,h^{(i)}[n] - s_i\,g^{(i)}[n],
  \qquad
  g^{(i-1)}[n-2] = s_i\,h^{(i)}[n] + c_i\,g^{(i)}[n].
  \label{eq:lattice-inv}
\end{equation}
As functions of $n$, the right-hand sides of \Cref{eq:lattice-inv} can be nonzero for $n\in\{0,\dots,2i+1\}$, hence define $h^{(i-1)}$ on $\{0,\dots,2i+1\}$ and $g^{(i-1)}$ on $\{-2,\dots,2i-1\}$; but a stage-$(i{-}1)$ pair is supported on $\{0,\dots,2i-1\}$, so $\phi_i$ is determined by the requirement that the four overflow taps vanish: $h^{(i-1)}[2i]=h^{(i-1)}[2i+1]=0$ and $g^{(i-1)}[-2]=g^{(i-1)}[-1]=0$. Substituting $g^{(i)}[n]=-(-1)^{n}h^{(i)}[2i+1-n]$ into \Cref{eq:lattice-inv}, these four conditions reduce to
\begin{equation}
  \tan\phi_i
  \;=\; \frac{h^{(i)}[2i+1]}{h^{(i)}[0]}
  \;=\; -\,\frac{h^{(i)}[2i]}{h^{(i)}[1]},
  \qquad i=K-1,\dots,1,
  \label{eq:lattice-peel}
\end{equation}
the two ratios agreeing because $h^{(i)}[0]h^{(i)}[2i]+h^{(i)}[1]h^{(i)}[2i+1] =\langle h^{(i)},\,h^{(i)}[\,\cdot-2i\,]\rangle=0$ for any admissible filter and any lag $2i\neq0$. We take the representative $\phi_i\in(-\frac{\pi}{2},\frac{\pi}{2}]$: the alternative $\phi_i+\pi$ merely negates $h^{(i-1)}$, a sign that propagates through the remaining steps and is absorbed by $\phi_0$ below. Eliminating $g^{(i)}$ from \Cref{eq:lattice-inv} by the same flip, the down-step reads
\begin{equation}
  h^{(i-1)}[n]
  \;=\; c_i\,h^{(i)}[n] + s_i\,(-1)^{n}\,h^{(i)}[2i+1-n],
  \qquad n=0,\dots,2i-1,
  \label{eq:lattice-stepdown}
\end{equation}
and $g^{(i-1)}$ is again minus the alternating flip of $h^{(i-1)}$ (the downward analogue of the induction above), so the recursion closes on $h$ alone. Starting from $h^{(K-1)}=h$ and iterating $i=K-1,\dots,1$ leaves the length-two filter $h^{(0)}$, to which \Cref{eq:lattice-peel} does not extend: at $i=0$ the relevant correlation is the lag-zero one, $\langle h^{(0)},h^{(0)}\rangle=1\neq0$, and the two ratios contradict each other. None is needed: the base case of \Cref{eq:lattice-taps} gives $h^{(0)}=(\cos\phi_0,\sin\phi_0)$ directly, a unit vector whose both taps carry the angle, hence
\begin{equation*}
  \phi_0 \;=\; \operatorname{atan2}\!\bigl(h^{(0)}[1],\,h^{(0)}[0]\bigr),
\end{equation*}
determined modulo $2\pi$, unlike the $\phi_i$, $i\geq 1$, of which only the tangent is pinned.
\begin{remark}
For the Daubechies-2 filter \Cref{eq:db2}, $\tan\phi_1=\frac{h[3]}{h[0]}=\frac{1-\sqrt3}{1+\sqrt3}=\sqrt3-2$, so $\phi_1=-\frac{\pi}{12}$, and \Cref{eq:lattice-stepdown} leaves $h^{(0)}=(\cos\frac{\pi}{3},\,\sin\frac{\pi}{3})$, recovering the angles of \Cref{eq:lattice-taps-K2}.
\end{remark}

\subsubsection{The normalization is an affine condition on the angles}
\label{app:angles}

\begin{proposition}
\label{prop:moments}
Let $\Phi=\sum_{i=0}^{K-1}\phi_i$. Then $\sum_n h_\phi[n]=\sqrt2\,\sin\bigl(\Phi+\frac{\pi}{4}\bigr)$. In particular, \Cref{eq:vm-angles} holds: $h_\phi$ satisfies \Cref{eq:cqf-norm} if and only if $\Phi\equiv\pi/4 \pmod{2\pi}$.
\end{proposition}

\begin{proof}
Since $\Lambda(1)=\Id$ and $R(\alpha)R(\beta)=R(\alpha+\beta)$, evaluating \Cref{eq:lattice} at $z=1$ gives $E_\phi(1)=R(\Phi)$. Hence, by \Cref{eq:polyphase-recover} at $z=1$,
\begin{equation*}
  \sum_n h_\phi[n]
  \;=\; H_\phi(1)
  \;=\; \bigl[E_\phi\bigr]_{11}(1)+\bigl[E_\phi\bigr]_{12}(1)
  \;=\; \cos\Phi+\sin\Phi
  \;=\; \sqrt2\,\sin\Bigl(\Phi+\frac{\pi}{4}\Bigr),
\end{equation*}
which equals $\sqrt2$ if and only if $\sin\bigl(\Phi+\frac{\pi}{4}\bigr)=1$, i.e.\ $\Phi\equiv\pi/4 \pmod{2\pi}$.
\end{proof}

\subsection{Orthogonality of the periodized transform}
\label{app:periodized}

The properties invoked in \Cref{sec:method:learnable} concern the finite matrix $\W$ of the periodized transform rather than the filters themselves. For any $h$ satisfying \Cref{eq:cqf-shift}, one periodized analysis stage on $\R^{L}$, $L$ even, is an orthogonal matrix: the even circular shifts of the periodized low and high-pass filters together form an orthonormal basis of $\R^{L}$~\citep{mallat2008wavelet}. The $J$-level matrix $\W$, a product of such stages, is therefore orthogonal as well, giving $\W^{-1}=\W^{\top}$ and $\|\W x\|_F=\|x\|_F$. Since \Cref{prop:lattice} produces filters satisfying \Cref{eq:cqf-shift} at every $\phi\in\R^{K}$, this holds at every value of the learned angles. The conclusion is specific to periodization: under symmetric extension the analysis operator of the same filters is a rectangular, merely expansive frame, and both properties fail.

The stage is a genuine multiresolution step only when the periodized filters do not wrap around. In the notation of \Cref{sec:background:wavelets}, stage $j$ filters the approximation $a_{j-1}\in\R^{L_{j-1}}$, so the condition reads $q\le L_{j-1}$; with $2^{J}\mid T$ and $L_j=T/2^{j}$, the lengths halve at every stage and the constraint binds at the coarsest one, $q\le L_{J-1}=T/2^{J-1}$.


\clearpage
\section{Velocity Network Architecture}
\label{app:architecture}

This section provides implementation details for the channel-token velocity network introduced in \Cref{sec:method:architecture}, including the time-conditioning pathway, AdaLN-Zero transformer blocks, adaptive read-out, and computational cost.

\subsection{Time conditioning}
\label{app:architecture:time}

We condition the velocity network on the continuous flow time $t\in[0,1]$ using a DiT-style pathway~\citep{peebles2023dit}. A sinusoidal featurization of $t$ is passed through a two-layer MLP with SiLU activations, producing a time embedding $\psi(t)\in\R^{d_{\mathrm{t}}}$. This representation is then projected to the transformer width, yielding a conditioning vector $\eta(t)\in\R^{d_{\mathrm{m}}}$. The resulting vector is shared across all channel tokens and supplies the conditioning signal to every transformer block as well as to the final adaptive normalization.

\subsection{AdaLN-Zero transformer blocks}
\label{app:architecture:blocks}

Starting from the channel tokens $u^{(0)}\in\R^{F\times d_{\mathrm{m}}}$ defined in \Cref{eq:token}, the network applies $B$ pre-norm transformer blocks. Block $i\in\{0,\ldots,B-1\}$ maps $u^{(i)}\in\R^{F\times d_{\mathrm{m}}}$ to $u^{(i+1)}\in\R^{F\times d_{\mathrm{m}}}$. Following AdaLN-Zero~\citep{peebles2023dit}, each block derives six modulation vectors from the time conditioning:
\begin{equation}
\label{eq:app-adaln-parameters}
\bigl(\beta^{(i)}_1,\gamma^{(i)}_1,g^{(i)}_1,\beta^{(i)}_2,\gamma^{(i)}_2,g^{(i)}_2\bigr)
=
\mathbf{A}_i\!\left(\eta(t)\right)
\in\R^{6d_{\mathrm{m}}},
\qquad
\mathbf{A}_i=\mathrm{Linear}\circ\mathrm{SiLU}.
\end{equation}
These modulation vectors are shared across the $F$ channel tokens. The attention and feed-forward branches are given by \Cref{eq:block-attn} and \Cref{eq:block-mlp}, respectively:
\begin{align}
h^{(i)}_1 &= \mathrm{LN}\bigl(u^{(i)}\bigr)\odot\bigl(1+\gamma^{(i)}_1\bigr) + \beta^{(i)}_1, &
\tilde u^{(i)} &= u^{(i)} + g^{(i)}_1\odot \mathrm{MHSA}\bigl(h^{(i)}_1,h^{(i)}_1,h^{(i)}_1\bigr), \label{eq:block-attn}\\
h^{(i)}_2 &= \mathrm{LN}\bigl(\tilde u^{(i)}\bigr)\odot\bigl(1+\gamma^{(i)}_2\bigr) + \beta^{(i)}_2, &
u^{(i+1)} &= \tilde u^{(i)} + g^{(i)}_2\odot \mathrm{MLP}\bigl(h^{(i)}_2\bigr). \label{eq:block-mlp}
\end{align}

Here, $\mathrm{LN}$ denotes layer normalization without learned affine parameters, $\mathrm{MHSA}$ is multi-head self-attention over the $F$ channel tokens, and $\odot$ denotes element-wise multiplication with broadcasting over the token dimension. The feed-forward network is a two-layer MLP with GELU activation and hidden width $4d_{\mathrm{m}}$. Because attention is performed over channels, rather than over the $T$ wavelet coefficients, its sequence length is $F$, so its cost $O(F^2 d_{\mathrm{m}})$ per block is small next to the $O(F d_{\mathrm{m}}^2)$ of the token-wise projections and MLP.

The final linear layer of every modulation network $\mathbf{A}_i$ is initialized to zero. Consequently, the residual gates $g^{(i)}_1$ and $g^{(i)}_2$ are zero at initialization, so both residual branches are initially inactive and each transformer block implements the identity map. This AdaLN-Zero initialization provides the stable starting point used in DiT~\citep{peebles2023dit}.

\subsection{Adaptive read-out}
\label{app:architecture:readout}

Let $u=u^{(B)}$ denote the output of the final transformer block. Before applying the level-specific heads of \Cref{eq:readout}, we use a final adaptive normalization whose scale and shift are obtained from the same time-conditioning representation:
\begin{equation}
\label{eq:app-output-conditioning}
\bigl(\beta_{\mathrm{o}},\gamma_{\mathrm{o}}\bigr)
=
\mathbf{A}_{\mathrm{out}}\!\left(\eta(t)\right),
\qquad
h=
\mathrm{LN}(u)\odot\left(1+\gamma_{\mathrm{o}}\right)+\beta_{\mathrm{o}}.
\end{equation}
The modulation parameters are shared across all channel tokens, and the final affine layer of $\mathbf{A}_{\mathrm{out}}$ is zero-initialized following the AdaLN-Zero initialization used in the transformer blocks. The resulting channel representations $h_f$ are projected back to the individual wavelet levels by the heads $\mathbf{H}_j$ as defined in \Cref{eq:readout}.

\subsection{Parameter count and computational cost }
\label{app:cost:breakdown}

\Cref{tab:param-breakdown} decomposes the $9{,}779{,}352$ trainable parameters of the default configuration of \Cref{sec:setup:implementation} into the components of \Cref{sec:method:architecture}. We write $d_{\mathrm{m}}$ for the token width, $B$ for the number of blocks, $d_{\mathrm{t}}$ for the time-embedding width, $J$ for the decomposition depth and $D=\sum_{j=1}^{J+1} L_j = T$ for the number of coefficients per channel. 

\begin{table}[!htbp]
  \centering
  \small
  \setlength{\tabcolsep}{5pt}
  \renewcommand{\arraystretch}{1.1}
  \begin{tabular}{|l|l|r|r|}
    \hline
    \textbf{Component} & \textbf{Formula} & \textbf{Parameters} & \textbf{Share} \\
    \hline
    MLP, all blocks              & $B\,(8d_{\mathrm{m}}^{2}+5d_{\mathrm{m}})$   & $4{,}204{,}544$ & $42.99\%$ \\
    AdaLN-Zero, all blocks       & $B\,(6d_{\mathrm{m}}^{2}+6d_{\mathrm{m}})$   & $3{,}158{,}016$ & $32.29\%$ \\
    Channel attention, all blocks& $B\,(4d_{\mathrm{m}}^{2}+4d_{\mathrm{m}})$   & $2{,}105{,}344$ & $21.53\%$ \\
    \hline
    Time embedding               & $8d_{\mathrm{t}}^{2}+5d_{\mathrm{t}}$        & $131{,}712$     & $1.35\%$ \\
    Output modulation            & $2d_{\mathrm{m}}^{2}+2d_{\mathrm{m}}$        & $131{,}584$     & $1.35\%$ \\
    Conditioning projection      & $d_{\mathrm{m}}d_{\mathrm{t}}+d_{\mathrm{m}}$& $33{,}024$      & $0.34\%$ \\
    \hline
    Level embedders $\mathbf{W}^{\mathrm{emb}}_j$ & $d_{\mathrm{m}}D+(J{+}1)d_{\mathrm{m}}$ & $7{,}168$ & $0.07\%$ \\
    Level heads $\mathbf{H}_j$   & $d_{\mathrm{m}}D+D$                          & $6{,}168$       & $0.06\%$ \\
    Channel identity             & $F d_{\mathrm{m}}$                           & $1{,}792$       & $0.02\%$ \\
    \hline
    \textbf{Total}               & \Cref{eq:paramcount}                         & $\mathbf{9{,}779{,}352}$ & $100\%$ \\
    \hline
  \end{tabular}
  \caption{Trainable parameters of the default configuration on ETTh1 at $T=24$ ($d_{\mathrm{m}}=256$, $B=8$, $d_{\mathrm{t}}=128$, MLP ratio $4$, $J=3$, $D=T=24$ coefficients per channel, $F=7$).}
  \label{tab:param-breakdown}
\end{table}

Summing the second column of \Cref{tab:param-breakdown} gives
\begin{equation}
\begin{aligned}
  \#\text{parameters} \;=\;&
  \underbrace{B\bigl(18d_{\mathrm{m}}^{2}+15d_{\mathrm{m}}\bigr)}_{\text{blocks}}
  \;+\;\underbrace{8d_{\mathrm{t}}^{2}+5d_{\mathrm{t}}}_{\text{time embedding}}
  \;+\;\underbrace{2d_{\mathrm{m}}^{2}+3d_{\mathrm{m}}+d_{\mathrm{m}}d_{\mathrm{t}}}_{\text{conditioning}}
  \\[3ex]
  &+\;\underbrace{2d_{\mathrm{m}}D+(J{+}1)d_{\mathrm{m}}+D}_{\text{level embedders and heads}}
  \;+\;\underbrace{F d_{\mathrm{m}}}_{\text{channel identity}} .
\end{aligned}
\label{eq:paramcount}
\end{equation}

Three points follow. The $B$ blocks are $96.8\%$ of the total, while the wavelet-specific level embedders $\mathbf{W}^{\mathrm{emb}}_j$ and heads $\mathbf{H}_j$ are $0.14\%$: the multi-resolution structure is almost free in parameters, since a token carries all $J{+}1$ levels of a channel and the per-level maps are shared across channels. Time conditioning is a third of the model, the AdaLN-Zero projections costing $6d_{\mathrm{m}}^{2}$ per block against $4d_{\mathrm{m}}^{2}$ for the attention itself. Finally, $F$ appears in only one term of \Cref{eq:paramcount}, so each additional channel adds exactly $d_{\mathrm{m}}=256$ parameters.

\clearpage
\section{Reproducibility and Experimental Details} \label{app:setup}
\FloatBarrier
\subsection{Datasets}
\label{app:datasets}

\Cref{tab:datasets} lists properties of the datasets used in our experiments.

\begin{table}[!htbp]
  \centering
  \small
  \begin{tabular}{l l r r l}
    \toprule
    \textbf{Dataset} & \textbf{Type} & \textbf{Length} & \textbf{Channels} $F$ & \textbf{Source} \\
    \midrule
    ETTh1         & Electricity transformer & $17{,}420$          & $7$  & \citet{zhou2021informer} \\
    ETTh2         & Electricity transformer & $17{,}420$          & $7$  & \citet{zhou2021informer} \\
    Stocks        & Daily Google prices     & $3{,}685$           & $6$  & \citet{yoon2019timegan} \\
    Exchange & Daily exchange rates    & $7{,}588$           & $8$  & \citet{lai2018lstnet} \\
    EEG           & 14-electrode recording  & $14{,}980$          & $14$ & \citet{roesler2013eeg} \\
    Energy        & Appliance energy        & $19{,}735$          & $28$ & \citet{candanedo2017energy} \\
    MuJoCo        & Simulated Hopper        & $T\times 10{,}000$ & $14$ & \citet{todorov2012mujoco} \\
    \bottomrule
  \end{tabular}
  \caption{Datasets used in our experiments. Length is the number of time steps in
    the raw series; for MuJoCo, $10{,}000$ trajectories of length $T$ are simulated.}
  \label{tab:datasets}
\end{table}

\paragraph{ETTh1 and ETTh2.}
These two datasets come from~\citet{zhou2021informer} and record two electricity transformers at separate sites in China. Each is sampled once per hour and covers about two years. Six channels give the electrical load, split into three magnitude levels with a useful and a useless part each, and the seventh is the temperature of the transformer oil. 

\paragraph{Stocks.}
Daily Google share data from $2004$ to $2019$, used as a benchmark since~\citet{yoon2019timegan}. The six channels are the open, high, low, close and adjusted close prices, plus the traded volume.

\paragraph{Exchange.}
Daily exchange rates of eight currencies against the US dollar between $1990$ and $2016$, from the collection of~\citet{lai2018lstnet}. The currencies are those of Australia, United Kingdom, Canada, Switzerland, China, Japan, New Zealand and Singapore. 

\paragraph{EEG.}
A single continuous EEG recording of one person, made with a $14$-electrode consumer headset and distributed by~\citet{roesler2013eeg}. It lasts under two minutes, and the $14{,}980$ samples correspond to about $128$ Hz. The $14$ channels are the electrode signals, one per position on the scalp. 

\paragraph{Energy.}
Measurements from one low-energy house over about four and a half months, taken every ten minutes, collected by~\citet{candanedo2017energy}. Two channels give the electricity used by the appliances and by the lights. Eighteen more come from nine sensor nodes, eight placed in rooms and one on the north facade, each giving a temperature and a relative humidity. Six channels are outdoor conditions from a nearby weather station, put on the same time grid: outside temperature, pressure, humidity, wind speed, visibility and dew point. The remaining two channels are the pair of uniform random variables that the original authors added as a control for feature selection. They are numerically identical to each other and we keep them so that the channel count matches prior work.

\paragraph{MuJoCo.}
Trajectories of the Hopper simulated in the MuJoCo physics engine~\citep{todorov2012mujoco}. We use $10{,}000$ separate rollouts, each started from a random initial configuration, and the $14$ channels are the positions and velocities of the body.

\FloatBarrier
\subsection{Baselines} \label{app:baselines}

\paragraph{Diffusion-TS}
Diffusion-TS~\citep{yuan2024diffusionts} is a denoising diffusion model that operates directly on the time-domain window $x\in\R^{T\times F}$. Its denoiser is an encoder--decoder transformer. The decoder is interpretable and writes its output as the sum of a trend component, modeled by a low-degree polynomial regression in time, and a seasonal component, modeled by a small set of trigonometric (Fourier) bases. The network predicts the clean sample rather than the injected noise. Its training loss combines a time-domain reconstruction error with the same error measured between the discrete Fourier transforms of the target and the prediction. Sampling uses ancestral DDPM or DDIM steps (\Cref{tab:nfe}). Official implementation: \url{https://github.com/Y-debug-sys/Diffusion-TS}.

\paragraph{SigDiffusion}
SigDiffusion~\citep{barancikova2025sigdiffusion} generates time series through their log-signature embedding. Each channel is treated as a continuous path, augmented with deterministic time-dependent channels, and mapped to its truncated log-signature. The per-channel log-signatures are then concatenated. The log-signatures live in a free nilpotent Lie algebra, which is a finite-dimensional vector space, so a standard score-based model can be trained on them. Here the model is a transformer score network under a variance-preserving SDE, sampled by integrating the probability-flow ODE with a Tsit5 solver. Generated log-signatures are mapped back to the time domain by closed-form inversion formulae, which express the truncated Fourier coefficients of the path as polynomial functions of the log-signature. The reconstructed series are therefore truncated Fourier expansions, and the signature truncation level bounds how many frequencies they contain. Official implementation: \url{https://github.com/BarbOra/SigDiffusions}.

\paragraph{FourierDiffusion}
FourierDiffusion~\citep{crabbe2024fourierdiffusion} trains a score-based diffusion model on the discrete Fourier transform (DFT) of each window, computed along time for every channel. The unitary DFT is linear, so applying it to the time-domain forward SDE yields a dual SDE in the frequency domain. In this dual SDE the Brownian motion is replaced by a mirrored Brownian motion: for a real signal the coefficients at frequencies $\kappa$ and $T-\kappa$ are complex conjugates, so only the non-redundant real and imaginary parts carry independent noise. Denoising score matching is adapted to this non-redundant representation, and samples are mapped back with the inverse DFT. The authors motivate the approach by showing that many real-world series are more localized in frequency than in time. Official implementation: \url{https://github.com/JonathanCrabbe/FourierDiffusion}.

\paragraph{WaveletDiff}
WaveletDiff~\citep{wang2025waveletdiff} trains a denoising diffusion model directly on the coefficients of a multilevel DWT of each window. Its denoiser assigns a dedicated transformer to each decomposition level, that is, to the approximation band and to every detail band. The levels exchange information through cross-level attention modulated by adaptive gating. Sampling uses DDPM or DDIM steps. Our model differs in two ways. It uses flow matching rather than diffusion. It also uses a single velocity network shared across all levels, with attention across channels, instead of level-specific networks. Official implementation: \url{https://github.com/GarlicWang/WaveletDiff}.

\paragraph{FlowTS}
FlowTS~\citep{hu2024flowts} trains a rectified flow~\citep{liu2023rectified} directly in the time domain, using the linear path of \Cref{eq:linear-path} and the loss of \Cref{eq:rf-loss} with the window $x$ in place of the coefficients $c$ (equivalently, with $\W=\Id$). Its velocity network is a transformer that keeps a trend--seasonality decomposition similar to that of Diffusion-TS, and adds attention registers for global context aggregation and rotary position embeddings. Training times are drawn from a logit-normal distribution~\citep{esser2024sd3}, and samples are generated by Euler integration of the learned ODE. The comparison with FlowTS therefore keeps the training objective essentially fixed while varying the representation and the network architecture. Official implementation: \url{https://github.com/UNITES-Lab/FlowTS}.

All six models generate a sample by integrating an ODE or SDE driven by the trained network, so sampling cost is naturally measured by the number of function evaluations (NFE), the number of network calls per sample. \Cref{tab:nfe} reports it at $T=24$ for each method's default sampler.

\begin{table}[t]
\centering
\small
\setlength{\tabcolsep}{6pt}
\renewcommand{\arraystretch}{1.08}
\begin{tabular}{lll}
\toprule
\textbf{Model} & \textbf{Sampler} & \textbf{NFE per sample} \\
\midrule
Diffusion-TS
  & DDPM / DDIM
  & 100--1000\textsuperscript{$\dagger$} \\
SigDiffusion
  & Tsit5 (128 steps)
  & 768 \\
FourierDiffusion
  & VP-SDE (1000 steps)
  & 1000 \\
WaveletDiff
  & DDPM (1000 steps)
  & 1000 \\
FlowTS
  & Euler (100 steps)
  & 100 \\
\rowcolor{ourscolor}
Ours
  & Euler (100 steps)
  & 100 \\
\bottomrule
\end{tabular}

\caption{Network function evaluations (NFE) per generated sample at $T=24$.
\textsuperscript{$\dagger$}Diffusion-TS uses 500 evaluations on ETTh1, ETTh2,
Stocks, and Exchange, 100 on EEG, and 1000 on Energy and MuJoCo.}
\label{tab:nfe}
\end{table}

NFE is per sample and independent of the batch size. Every method calls its network once per integration step except SigDiffusion, whose sampler is a Tsitouras $5(4)$ Runge--Kutta method with six evaluations per step, so its $128$ configured steps cost $768$ evaluations. Diffusion-TS varies because its configuration files set a different number of diffusion steps per dataset and switch to DDIM with $100$ steps on EEG. These are the settings released by each method's authors. 
\FloatBarrier
\subsection{Evaluation metrics}
\label{app:metrics}

All four metrics compare the set of real windows against an equally sized set of generated windows, and lower is better throughout.

\paragraph{Discriminative score.}
A post-hoc classifier is trained to separate real from generated windows. The score is $|\mathrm{acc} - 0.5|$ on a held-out split, so $0$ means the two sets are indistinguishable and $0.5$ means perfectly separable. Following TimeGAN~\citep{yoon2019timegan}, the classifier is a single-layer GRU~\citep{cho2014gru} of hidden width $\max\bigl(\lfloor F/2 \rfloor,\, 2\bigr)$ followed by a linear layer, trained for $2000$ iterations with Adam and a binary cross-entropy loss at batch size $128$.

\paragraph{Predictive score.}
A one-step-ahead forecaster is trained on the generated windows and evaluated on the real ones, measuring how useful the synthetic data is as a substitute for real training data. The score is the mean absolute error on the real windows. Again following~\citet{yoon2019timegan}, the forecaster is a single-layer GRU of hidden width $\max\bigl(\lfloor F/2 \rfloor,\, 2\bigr)$, trained for $5000$ iterations at batch size $128$.

\paragraph{Context-FID.}
The Fréchet Inception Distance~\citep{heusel2017fid} between two Gaussians fitted to a feature representation,
\begin{equation}
\label{eq:fid}
  \mathrm{FID}(X,Y) \;=\; \lVert \mu_X - \mu_Y \rVert^2
  \;+\; \operatorname{Tr}\bigl( \Sigma_X + \Sigma_Y
  - 2(\Sigma_X \Sigma_Y)^{1/2} \bigr),
\end{equation}
with $\mu$ and $\Sigma$ the mean and covariance of the features. Context-FID~\citep{jeha2022psagan} replaces the Inception features of the image version by embeddings from a TS2Vec encoder~\citep{yue2022ts2vec} trained on the real windows. We use the implementation of~\citet{yuan2024diffusionts} with output dimension $320$.

\paragraph{Correlational score.}
Writing $\rho^{\mathrm{real}}_{i,j}$ and $\rho^{\mathrm{gen}}_{i,j}$ for the lag-zero correlation between channels $i$ and $j$ in the real and generated data, the score is
\begin{equation}
\label{eq:corr-score}
  \frac{1}{10} \sum_{1 \le j \le i \le F}
  \bigl\lvert \rho^{\mathrm{real}}_{i,j} - \rho^{\mathrm{gen}}_{i,j} \bigr\rvert.
\end{equation}
We keep the factor $1/10$ of~\citet{liao2020sigwgan}, used by all the baselines, so that our numbers are directly comparable. Note that it does not normalize by $F$, so the score is not comparable across datasets.

\paragraph{Repeats.}
All four metrics are stochastic through the auxiliary networks they train, or through subsampling, so each is recomputed five times on the same generated set and we report mean $\pm$ standard deviation. This measures the variability of the metric, not of training.

\subsection{Training and implementation details}
\label{sec:setup:implementation}

Unless stated otherwise, all experiments use the following configuration. The analysis transform uses periodic boundary handling, with \texttt{db2} for $T\le32$, \texttt{db4} for $T=64$ and \texttt{db6} for $T=128$ (filter length $q=4,8,12$), at depth $J=\min\bigl(\max\bigl(\lfloor\log_2\tfrac{T}{q-1} \rfloor,3\bigr),7\bigr)$, i.e.\ $J=3$ in all configurations. The velocity network uses token width $d_{\mathrm{m}}=256$, depth $B=8$, $8$ attention heads, MLP ratio $4$, dropout $0.1$, and time-embedding width $d_{\mathrm{t}}=128$ ($\approx 9.78$M parameters, decomposed in \Cref{tab:param-breakdown}). Training minimizes \Cref{eq:rf-loss} with logit-normal time sampling \Cref{eq:logitnorm} ($m=0$, $s=1$), using AdamW~\citep{loshchilov2019adamw} (weight decay $10^{-5}$, $\beta=(0.9,0.999)$, $\varepsilon=10^{-8}$), a one-cycle learning-rate schedule~\citep{smith2019onecycle} with peak learning rate $6\times 10^{-4}$ and $6\%$ warm-up, batch size $512$, gradient-norm clipping at $1.0$, for $2400$ epochs per dataset. An EMA of all weights with decay $0.999$ is updated every optimizer step and used for sampling and evaluation. Sampling uses $N=100$ Euler steps on the uniform grid ($\alpha=1$), and no validation-based model selection is performed so the final EMA weights are evaluated. The learnable wavelet of \Cref{sec:method:learnable} is disabled in the main configuration and studied separately in the ablations.
All models were trained and sampled in the same environment on a single NVIDIA GeForce RTX 4090 (24\,GB), in \texttt{float32} with TF32/bf16 matmuls enabled, under PyTorch 2.13~\citep{paszke2019pytorch}.


\clearpage
\section{Extended Results}

\FloatBarrier
\subsection{Coefficient statistics}
\label{app:coefstats}

\Cref{tab:level-energy} reports, for each dataset and each level, the standard deviation $\sigma$ of the coefficients and the time $t=1/(1+\sigma)$ at which that level's signal-to-noise ratio \Cref{eq:snr} crosses one (called crossing time).

\begin{table}[H]
  \centering
  \small
  \setlength{\tabcolsep}{6pt}
  \begin{tabular}{l *{8}{c}}
    \toprule
    & \multicolumn{2}{c}{$a_3$}
    & \multicolumn{2}{c}{$d_3$}
    & \multicolumn{2}{c}{$d_2$}
    & \multicolumn{2}{c}{$d_1$} \\
    \cmidrule(lr){2-3} \cmidrule(lr){4-5} \cmidrule(lr){6-7} \cmidrule(lr){8-9}
    \textbf{Dataset}
    & $\sigma$ & $t$ & $\sigma$ & $t$ & $\sigma$ & $t$ & $\sigma$ & $t$ \\
    \midrule
    ETTh1         & 2.57 & 0.28 & 0.90 & 0.53 & 0.42 & 0.70 & 0.24 & 0.81 \\
    ETTh2         & 2.76 & 0.27 & 0.40 & 0.71 & 0.24 & 0.81 & 0.17 & 0.85 \\
    Stocks        & 2.75 & 0.27 & 0.30 & 0.77 & 0.23 & 0.81 & 0.17 & 0.85 \\
    Exchange & 2.82 & 0.26 & 0.12 & 0.89 & 0.07 & 0.93 & 0.04 & 0.96 \\
    EEG           & 1.41 & 0.41 & 0.95 & 0.51 & 0.93 & 0.52 & 0.92 & 0.52 \\
    Energy        & 2.66 & 0.27 & 0.51 & 0.66 & 0.37 & 0.73 & 0.31 & 0.76 \\
    MuJoCo        & 2.67 & 0.27 & 0.66 & 0.60 & 0.38 & 0.72 & 0.21 & 0.83 \\
    \bottomrule
  \end{tabular}
  \caption{Per-level coefficient standard deviation $\sigma$ and SNR crossing time
    $t = 1/(1+\sigma)$ at $T=24$, under the \texttt{db2} wavelet.
    Coarser levels have larger $\sigma$ and therefore cross earlier.}
  \label{tab:level-energy}
\end{table}

\Cref{fig:snr-levels-all} plots the corresponding SNR trajectories for the six datasets not shown in \Cref{fig:snr-levels}: the crossing order $a_3 \to d_3 \to d_2 \to d_1$ holds on all of them, although on EEG the three detail levels nearly coincide.

\begin{figure}[H]
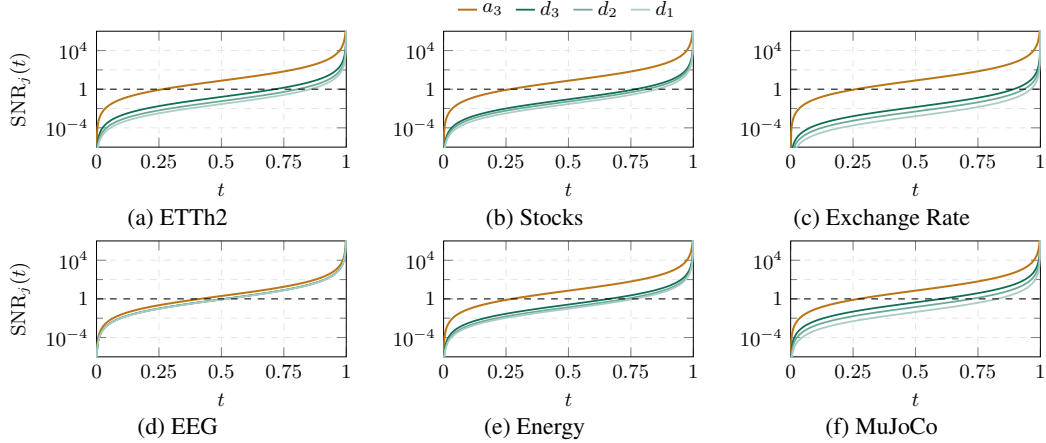

  \centering
  \captionsetup[subfigure]{font=small, skip=1pt}
  \begin{subfigure}[b]{0.3427\linewidth}
    \centering
    \snrnolegend
    \begingroup
  \def\snrdataset{ETTh2}\def\snrwidth{5.5cm}\def\snrheight{3.4cm}\includestandalone[width=\linewidth]{figures/SNR_figure}\endgroup
    \caption{ETTh2}
  \end{subfigure}\hfill
  \begin{subfigure}[b]{0.3236\linewidth}
    \centering
    \snrnoy\snrtightlegend
    \begingroup
  \def\snrdataset{Stocks}\def\snrwidth{5.5cm}\def\snrheight{3.4cm}\includestandalone[width=\linewidth]{figures/SNR_figure}\endgroup
    \caption{Stocks}
  \end{subfigure}\hfill
  \begin{subfigure}[b]{0.3236\linewidth}
    \centering
    \snrnoy\snrnolegend
    \begingroup
  \def\snrdataset{Exchange Rate}\def\snrwidth{5.5cm}\def\snrheight{3.4cm}\includestandalone[width=\linewidth]{figures/SNR_figure}\endgroup
    \caption{Exchange Rate}
  \end{subfigure}

  \vspace{0.2em}
  \begin{subfigure}[b]{0.3427\linewidth}
    \centering
    \snrnolegend
    \begingroup
  \def\snrdataset{EEG}\def\snrwidth{5.5cm}\def\snrheight{3.4cm}\includestandalone[width=\linewidth]{figures/SNR_figure}\endgroup
    \caption{EEG}
  \end{subfigure}\hfill
  \begin{subfigure}[b]{0.3236\linewidth}
    \centering
    \snrnoy\snrnolegend
    \begingroup
  \def\snrdataset{Energy}\def\snrwidth{5.5cm}\def\snrheight{3.4cm}\includestandalone[width=\linewidth]{figures/SNR_figure}\endgroup
    \caption{Energy}
  \end{subfigure}\hfill
  \begin{subfigure}[b]{0.3236\linewidth}
    \centering
    \snrnoy\snrnolegend
    \begingroup
  \def\snrdataset{MuJoCo}\def\snrwidth{5.5cm}\def\snrheight{3.4cm}\includestandalone[width=\linewidth]{figures/SNR_figure}\endgroup
    \caption{MuJoCo}
  \end{subfigure}
  \caption{Level-wise SNR along the linear probability path on the other six datasets. Coarse levels cross $\mathrm{SNR}=1$ earlier than fine levels.}
  \label{fig:snr-levels-all}
\end{figure}

\clearpage

\FloatBarrier
\subsection{Full results on longer sequences}
\label{app:long}
 
\Cref{tab:long32,tab:long64,tab:long128} give the per-dataset values summarized in \Cref{fig:lengths} for all four metrics, seven datasets and six models at $T=32$, $64$ and $128$, and \Cref{fig:app-grid} plots them, together with those of \Cref{tab:main24} at $T=24$, as a function of $T$ in one plot per dataset--metric pair.
 
\begin{table}[H]
\centering
\caption{Unconditional generation at $T=32$. Mean {\scriptsize$\pm$std} over five recomputations of each metric on the same generated set. \textbf{Bold}: best; \underline{underlined}: second best.}
\label{tab:long32}
\small
\setlength{\tabcolsep}{3.4pt}
\begin{tabular}{@{}l*{7}{c}@{}}
\toprule
 & ETTh1 & ETTh2 & Stocks & Exchange & EEG & Energy & MuJoCo \\
\midrule
\multicolumn{8}{@{}l}{\textit{\hspace{0.2cm}Context-FID} ($\downarrow$)} \\
Diffusion-TS & \ms{0.187}{.006} & \ms{0.073}{.004} & \ms{0.238}{.050} & \ms{0.048}{.005} & \ms{0.024}{.002} & \ms{0.106}{.019} & \ms{\second{0.021}}{.002} \\
SigDiffusion & \ms{2.972}{.112} & \ms{1.413}{.181} & \ms{4.086}{.784} & \ms{1.702}{.101} & \ms{0.032}{.004} & \ms{4.730}{.547} & \ms{3.117}{.150} \\
FourierDiffusion & \ms{0.039}{.005} & \ms{0.034}{.004} & \ms{0.055}{.010} & \ms{0.058}{.009} & \ms{0.022}{.001} & \ms{0.296}{.018} & \ms{0.099}{.012} \\
WaveletDiff & \ms{0.059}{.005} & \ms{0.074}{.005} & \ms{0.024}{.002} & \ms{\second{0.007}}{.001} & \ms{\second{0.011}}{.002} & \ms{0.541}{.022} & \ms{1.269}{.149} \\
FlowTS & \ms{\second{0.030}}{.002} & \ms{\second{0.013}}{.001} & \ms{\second{0.019}}{.003} & \ms{0.010}{.001} & \ms{\best{0.008}}{.001} & \ms{\second{0.061}}{.004} & \ms{0.026}{.001} \\
\rowcolor{ourscolor} \textbf{Ours} & \ms{\best{0.007}}{.000} & \ms{\best{0.003}}{.000} & \ms{\best{0.010}}{.001} & \ms{\best{0.003}}{.001} & \ms{0.022}{.003} & \ms{\best{0.012}}{.000} & \ms{\best{0.006}}{.000} \\
\midrule
\addlinespace[2pt]
\multicolumn{8}{@{}l}{\textit{\hspace{0.2cm}Discriminative score} ($\downarrow$)} \\
Diffusion-TS & \ms{0.084}{.006} & \ms{0.039}{.004} & \ms{0.103}{.012} & \ms{0.026}{.007} & \ms{0.246}{.174} & \ms{0.115}{.003} & \ms{0.032}{.007} \\
SigDiffusion & \ms{0.299}{.052} & \ms{0.347}{.040} & \ms{0.351}{.013} & \ms{0.340}{.084} & \ms{0.327}{.216} & \ms{0.500}{.000} & \ms{0.466}{.051} \\
FourierDiffusion & \ms{0.023}{.007} & \ms{0.017}{.009} & \ms{\best{0.005}}{.004} & \ms{0.020}{.012} & \ms{0.009}{.004} & \ms{0.133}{.005} & \ms{0.050}{.006} \\
WaveletDiff & \ms{0.023}{.013} & \ms{0.019}{.003} & \ms{0.016}{.010} & \ms{\best{0.005}}{.006} & \ms{\best{0.005}}{.004} & \ms{0.366}{.021} & \ms{0.195}{.036} \\
FlowTS & \ms{\best{0.006}}{.004} & \ms{\best{0.004}}{.002} & \ms{0.025}{.010} & \ms{0.010}{.008} & \ms{0.385}{.035} & \ms{\best{0.094}}{.029} & \ms{\second{0.022}}{.008} \\
\rowcolor{ourscolor} \textbf{Ours} & \ms{\best{0.006}}{.006} & \ms{\second{0.005}}{.002} & \ms{\second{0.008}}{.007} & \ms{\second{0.008}}{.004} & \ms{\best{0.005}}{.002} & \ms{\second{0.101}}{.007} & \ms{\best{0.005}}{.003} \\
\midrule
\addlinespace[2pt]
\multicolumn{8}{@{}l}{\textit{\hspace{0.2cm} Correlational score} ($\downarrow$)} \\
Diffusion-TS & \ms{0.061}{.018} & \ms{0.087}{.023} & \ms{0.019}{.003} & \ms{0.103}{.040} & \ms{4.672}{.150} & \ms{1.162}{.124} & \ms{\second{0.239}}{.021} \\
SigDiffusion & \ms{0.201}{.015} & \ms{0.399}{.017} & \ms{0.152}{.010} & \ms{1.067}{.025} & \ms{4.439}{.293} & \ms{7.274}{.090} & \ms{0.844}{.020} \\
FourierDiffusion & \ms{0.050}{.010} & \ms{0.109}{.005} & \ms{0.015}{.004} & \ms{0.076}{.019} & \ms{3.623}{.270} & \ms{1.321}{.219} & \ms{0.267}{.022} \\
WaveletDiff & \ms{0.054}{.017} & \ms{0.083}{.018} & \ms{\best{0.006}}{.004} & \ms{0.086}{.025} & \ms{\best{2.119}}{.675} & \ms{1.337}{.255} & \ms{0.277}{.014} \\
FlowTS & \ms{\second{0.041}}{.011} & \ms{\second{0.065}}{.016} & \ms{\second{0.012}}{.003} & \ms{\second{0.056}}{.019} & \ms{\second{2.698}}{.601} & \ms{\second{1.152}}{.222} & \ms{0.249}{.022} \\
\rowcolor{ourscolor} \textbf{Ours} & \ms{\best{0.040}}{.008} & \ms{\best{0.064}}{.018} & \ms{0.018}{.004} & \ms{\best{0.052}}{.006} & \ms{3.972}{.542} & \ms{\best{0.722}}{.089} & \ms{\best{0.237}}{.031} \\
\midrule
\addlinespace[2pt]
\multicolumn{8}{@{}l}{\textit{\hspace{0.2cm}Predictive score} ($\downarrow$)} \\
Diffusion-TS & \ms{\second{0.119}}{.003} & \ms{0.105}{.004} & \ms{\best{0.037}}{.000} & \ms{0.043}{.004} & \ms{0.001}{.000} & \ms{\second{0.251}}{.000} & \ms{\best{0.008}}{.001} \\
SigDiffusion & \ms{0.128}{.003} & \ms{0.127}{.005} & \ms{0.041}{.003} & \ms{0.085}{.010} & \ms{\best{0.000}}{.000} & \ms{0.356}{.007} & \ms{0.024}{.001} \\
FourierDiffusion & \ms{0.121}{.003} & \ms{0.105}{.003} & \ms{\best{0.037}}{.000} & \ms{0.046}{.006} & \ms{\best{0.000}}{.000} & \ms{\second{0.251}}{.000} & \ms{0.011}{.001} \\
WaveletDiff & \ms{\best{0.115}}{.004} & \ms{\best{0.102}}{.004} & \ms{\best{0.037}}{.000} & \ms{0.044}{.005} & \ms{\best{0.000}}{.000} & \ms{\second{0.251}}{.000} & \ms{0.009}{.001} \\
FlowTS & \ms{0.120}{.003} & \ms{0.106}{.003} & \ms{\best{0.037}}{.000} & \ms{\best{0.041}}{.004} & \ms{0.001}{.000} & \ms{\second{0.251}}{.000} & \ms{0.010}{.002} \\
\rowcolor{ourscolor} \textbf{Ours} & \ms{0.120}{.003} & \ms{\second{0.104}}{.003} & \ms{\best{0.037}}{.000} & \ms{\best{0.041}}{.006} & \ms{\best{0.000}}{.000} & \ms{\best{0.250}}{.000} & \ms{\best{0.008}}{.002} \\
\bottomrule
\end{tabular}
\end{table}
 
\begin{table}[t]
\centering
\caption{Unconditional generation at $T=64$. Same conventions as \Cref{tab:long32}.}
\label{tab:long64}
\small
\setlength{\tabcolsep}{3.4pt}
\begin{tabular}{@{}l*{7}{c}@{}}
\toprule
 & ETTh1 & ETTh2 & Stocks & Exchange & EEG & Energy & MuJoCo \\
\midrule
\multicolumn{8}{@{}l}{\textit{\hspace{0.2cm}Context-FID} ($\downarrow$)} \\
Diffusion-TS & \ms{0.267}{.023} & \ms{0.112}{.010} & \ms{0.354}{.080} & \ms{0.056}{.003} & \ms{0.058}{.005} & \ms{\second{0.085}}{.007} & \ms{\second{0.035}}{.002} \\
SigDiffusion & \ms{5.948}{.465} & \ms{1.581}{.174} & \ms{3.851}{.744} & \ms{1.986}{.125} & \ms{0.056}{.004} & \ms{6.403}{.271} & \ms{3.621}{.181} \\
FourierDiffusion & \ms{0.089}{.005} & \ms{0.068}{.005} & \ms{0.111}{.012} & \ms{0.081}{.008} & \ms{0.045}{.005} & \ms{0.446}{.030} & \ms{0.167}{.009} \\
WaveletDiff & \ms{0.104}{.008} & \ms{0.072}{.003} & \ms{0.059}{.011} & \ms{0.151}{.015} & \ms{\second{0.025}}{.003} & \ms{0.438}{.013} & \ms{0.347}{.021} \\
FlowTS & \ms{\second{0.052}}{.003} & \ms{\second{0.027}}{.002} & \ms{\second{0.038}}{.005} & \ms{\second{0.016}}{.001} & \ms{\best{0.015}}{.003} & \ms{0.172}{.016} & \ms{0.093}{.005} \\
\rowcolor{ourscolor} \textbf{Ours} & \ms{\best{0.012}}{.001} & \ms{\best{0.009}}{.001} & \ms{\best{0.014}}{.003} & \ms{\best{0.003}}{.000} & \ms{0.041}{.006} & \ms{\best{0.079}}{.002} & \ms{\best{0.011}}{.001} \\
\midrule
\addlinespace[2pt]
\multicolumn{8}{@{}l}{\textit{\hspace{0.2cm}Discriminative score} ($\downarrow$)} \\
Diffusion-TS & \ms{0.090}{.004} & \ms{0.040}{.013} & \ms{0.107}{.007} & \ms{0.035}{.006} & \ms{0.308}{.179} & \ms{0.203}{.049} & \ms{\second{0.030}}{.004} \\
SigDiffusion & \ms{0.385}{.066} & \ms{0.151}{.039} & \ms{0.316}{.021} & \ms{0.305}{.016} & \ms{0.395}{.198} & \ms{0.499}{.003} & \ms{0.484}{.002} \\
FourierDiffusion & \ms{0.043}{.014} & \ms{0.024}{.010} & \ms{0.017}{.007} & \ms{0.075}{.014} & \ms{0.013}{.003} & \ms{\second{0.178}}{.014} & \ms{0.098}{.025} \\
WaveletDiff & \ms{0.024}{.012} & \ms{0.017}{.004} & \ms{\best{0.003}}{.003} & \ms{0.065}{.018} & \ms{\second{0.007}}{.005} & \ms{0.435}{.023} & \ms{0.155}{.036} \\
FlowTS & \ms{\second{0.017}}{.009} & \ms{\second{0.006}}{.002} & \ms{0.011}{.005} & \ms{\second{0.010}}{.007} & \ms{0.221}{.135} & \ms{\best{0.116}}{.070} & \ms{0.088}{.030} \\
\rowcolor{ourscolor} \textbf{Ours} & \ms{\best{0.005}}{.003} & \ms{\best{0.005}}{.004} & \ms{\second{0.010}}{.008} & \ms{\best{0.007}}{.007} & \ms{\best{0.005}}{.004} & \ms{0.312}{.009} & \ms{\best{0.008}}{.004} \\
\midrule
\addlinespace[2pt]
\multicolumn{8}{@{}l}{\textit{\hspace{0.2cm}Correlational score} ($\downarrow$)} \\
Diffusion-TS & \ms{0.068}{.018} & \ms{0.110}{.019} & \ms{0.018}{.002} & \ms{0.125}{.042} & \ms{5.136}{.128} & \ms{\second{0.877}}{.154} & \ms{0.215}{.028} \\
SigDiffusion & \ms{0.198}{.017} & \ms{0.337}{.021} & \ms{0.115}{.007} & \ms{1.005}{.020} & \ms{3.706}{.198} & \ms{6.615}{.090} & \ms{0.688}{.017} \\
FourierDiffusion & \ms{0.053}{.009} & \ms{0.103}{.012} & \ms{\second{0.010}}{.004} & \ms{0.115}{.019} & \ms{2.755}{.527} & \ms{1.631}{.264} & \ms{0.208}{.020} \\
WaveletDiff & \ms{0.048}{.012} & \ms{\second{0.061}}{.020} & \ms{\best{0.004}}{.003} & \ms{0.218}{.030} & \ms{\second{2.608}}{.128} & \ms{1.071}{.103} & \ms{0.212}{.023} \\
FlowTS & \ms{\second{0.042}}{.010} & \ms{0.068}{.017} & \ms{0.012}{.003} & \ms{\second{0.082}}{.036} & \ms{\best{1.744}}{.773} & \ms{1.057}{.170} & \ms{\second{0.196}}{.015} \\
\rowcolor{ourscolor} \textbf{Ours} & \ms{\best{0.034}}{.003} & \ms{\best{0.056}}{.009} & \ms{\second{0.010}}{.005} & \ms{\best{0.067}}{.020} & \ms{3.788}{.364} & \ms{\best{0.789}}{.104} & \ms{\best{0.191}}{.021} \\
\midrule
\addlinespace[2pt]
\multicolumn{8}{@{}l}{\textit{\hspace{0.2cm}Predictive score} ($\downarrow$)} \\
Diffusion-TS & \ms{0.117}{.006} & \ms{0.119}{.017} & \ms{0.037}{.000} & \ms{0.041}{.005} & \ms{0.001}{.000} & \ms{\second{0.250}}{.000} & \ms{\second{0.008}}{.001} \\
SigDiffusion & \ms{0.125}{.002} & \ms{0.125}{.003} & \ms{0.038}{.001} & \ms{0.068}{.003} & \ms{\best{0.000}}{.000} & \ms{0.310}{.003} & \ms{0.017}{.001} \\
FourierDiffusion & \ms{0.115}{.005} & \ms{0.104}{.002} & \ms{0.037}{.000} & \ms{0.044}{.002} & \ms{\best{0.000}}{.000} & \ms{0.251}{.000} & \ms{0.009}{.001} \\
WaveletDiff & \ms{\best{0.113}}{.011} & \ms{0.102}{.001} & \ms{0.037}{.000} & \ms{0.043}{.003} & \ms{0.002}{.000} & \ms{\second{0.250}}{.001} & \ms{\best{0.006}}{.001} \\
FlowTS & \ms{0.117}{.003} & \ms{\best{0.099}}{.002} & \ms{\best{0.036}}{.000} & \ms{\best{0.040}}{.007} & \ms{0.001}{.000} & \ms{0.251}{.000} & \ms{\second{0.008}}{.000} \\
\rowcolor{ourscolor} \textbf{Ours} & \ms{\best{0.113}}{.002} & \ms{\second{0.100}}{.001} & \ms{\best{0.036}}{.000} & \ms{\best{0.040}}{.006} & \ms{\best{0.000}}{.000} & \ms{\best{0.249}}{.000} & \ms{\second{0.008}}{.002} \\
\bottomrule
\end{tabular}
\end{table}
 
\begin{table}[t]
\centering
\caption{Unconditional generation at $T=128$. Same conventions as \Cref{tab:long32}.}
\label{tab:long128}
\small
\setlength{\tabcolsep}{3.4pt}
\begin{tabular}{@{}l*{7}{c}@{}}
\toprule
 & ETTh1 & ETTh2 & Stocks & Exchange & EEG & Energy & MuJoCo \\
\midrule
\multicolumn{8}{@{}l}{\textit{\hspace{0.2cm}Context-FID} ($\downarrow$)} \\
Diffusion-TS & \ms{0.994}{.025} & \ms{0.146}{.018} & \ms{0.611}{.068} & \ms{0.059}{.004} & \ms{0.139}{.023} & \ms{\best{0.074}}{.008} & \ms{\second{0.073}}{.004} \\
SigDiffusion & \ms{11.619}{.739} & \ms{2.087}{.284} & \ms{4.743}{.381} & \ms{2.005}{.267} & \ms{0.130}{.012} & \ms{10.471}{.272} & \ms{3.738}{.269} \\
FourierDiffusion & \ms{0.319}{.020} & \ms{0.200}{.013} & \ms{0.272}{.023} & \ms{0.235}{.040} & \ms{0.086}{.014} & \ms{0.837}{.034} & \ms{0.224}{.021} \\
WaveletDiff & \ms{0.158}{.023} & \ms{0.088}{.006} & \ms{0.112}{.010} & \ms{0.154}{.015} & \ms{\second{0.059}}{.005} & \ms{0.501}{.037} & \ms{0.277}{.018} \\
FlowTS & \ms{\second{0.097}}{.003} & \ms{\second{0.058}}{.004} & \ms{\second{0.063}}{.010} & \ms{\second{0.035}}{.004} & \ms{\best{0.027}}{.003} & \ms{0.260}{.021} & \ms{0.156}{.009} \\
\rowcolor{ourscolor} \textbf{Ours} & \ms{\best{0.037}}{.001} & \ms{\best{0.018}}{.003} & \ms{\best{0.043}}{.006} & \ms{\best{0.013}}{.001} & \ms{0.114}{.012} & \ms{\second{0.123}}{.013} & \ms{\best{0.024}}{.002} \\
\midrule
\addlinespace[2pt]
\multicolumn{8}{@{}l}{\textit{\hspace{0.2cm}Discriminative score} ($\downarrow$)} \\
Diffusion-TS & \ms{0.165}{.008} & \ms{0.047}{.013} & \ms{0.118}{.019} & \ms{0.035}{.005} & \ms{0.211}{.220} & \ms{\best{0.231}}{.054} & \ms{\second{0.068}}{.003} \\
SigDiffusion & \ms{0.258}{.178} & \ms{0.153}{.087} & \ms{0.321}{.011} & \ms{0.245}{.081} & \ms{0.419}{.123} & \ms{0.499}{.001} & \ms{0.481}{.007} \\
FourierDiffusion & \ms{0.095}{.039} & \ms{0.061}{.006} & \ms{0.041}{.042} & \ms{0.093}{.018} & \ms{\second{0.011}}{.008} & \ms{0.301}{.041} & \ms{0.168}{.068} \\
WaveletDiff & \ms{0.010}{.005} & \ms{0.019}{.007} & \ms{\second{0.015}}{.006} & \ms{0.059}{.011} & \ms{\second{0.011}}{.001} & \ms{0.463}{.023} & \ms{0.078}{.048} \\
FlowTS & \ms{\second{0.009}}{.005} & \ms{\second{0.007}}{.006} & \ms{0.020}{.017} & \ms{\second{0.010}}{.008} & \ms{0.176}{.118} & \ms{\second{0.296}}{.021} & \ms{0.112}{.042} \\
\rowcolor{ourscolor} \textbf{Ours} & \ms{\best{0.007}}{.004} & \ms{\best{0.003}}{.002} & \ms{\best{0.007}}{.004} & \ms{\best{0.007}}{.005} & \ms{\best{0.008}}{.006} & \ms{0.411}{.022} & \ms{\best{0.008}}{.004} \\
\midrule
\addlinespace[2pt]
\multicolumn{8}{@{}l}{\textit{\hspace{0.2cm}Correlational score} ($\downarrow$)} \\
Diffusion-TS & \ms{0.092}{.009} & \ms{0.086}{.020} & \ms{0.026}{.003} & \ms{0.106}{.022} & \ms{5.141}{.128} & \ms{\best{0.619}}{.051} & \ms{0.229}{.021} \\
SigDiffusion & \ms{0.229}{.017} & \ms{0.393}{.018} & \ms{0.105}{.004} & \ms{1.015}{.039} & \ms{3.141}{.032} & \ms{5.903}{.068} & \ms{0.596}{.021} \\
FourierDiffusion & \ms{0.080}{.012} & \ms{0.160}{.011} & \ms{0.034}{.006} & \ms{0.155}{.023} & \ms{\second{1.394}}{.451} & \ms{1.369}{.369} & \ms{0.201}{.012} \\
WaveletDiff & \ms{\second{0.040}}{.010} & \ms{\second{0.062}}{.020} & \ms{\best{0.005}}{.002} & \ms{0.153}{.037} & \ms{1.572}{.313} & \ms{0.905}{.062} & \ms{\best{0.177}}{.013} \\
FlowTS & \ms{\best{0.035}}{.004} & \ms{0.068}{.010} & \ms{\second{0.009}}{.005} & \ms{\second{0.068}}{.015} & \ms{\best{0.871}}{.097} & \ms{0.837}{.087} & \ms{0.204}{.026} \\
\rowcolor{ourscolor} \textbf{Ours} & \ms{\second{0.040}}{.013} & \ms{\best{0.059}}{.020} & \ms{0.020}{.001} & \ms{\best{0.067}}{.027} & \ms{3.348}{.603} & \ms{\second{0.770}}{.174} & \ms{\second{0.182}}{.022} \\
\midrule
\addlinespace[2pt]
\multicolumn{8}{@{}l}{\textit{\hspace{0.2cm}Predictive score} ($\downarrow$)} \\
Diffusion-TS & \ms{0.120}{.002} & \ms{0.109}{.009} & \ms{\best{0.036}}{.000} & \ms{0.040}{.005} & \ms{\best{0.000}}{.000} & \ms{\second{0.249}}{.001} & \ms{\second{0.005}}{.001} \\
SigDiffusion & \ms{0.129}{.002} & \ms{0.123}{.006} & \ms{0.038}{.000} & \ms{0.068}{.003} & \ms{\best{0.000}}{.000} & \ms{0.278}{.002} & \ms{0.011}{.000} \\
FourierDiffusion & \ms{\second{0.112}}{.004} & \ms{0.107}{.003} & \ms{0.037}{.000} & \ms{0.049}{.004} & \ms{\best{0.000}}{.000} & \ms{0.251}{.000} & \ms{0.006}{.000} \\
WaveletDiff & \ms{\best{0.103}}{.005} & \ms{0.104}{.003} & \ms{\best{0.036}}{.000} & \ms{\second{0.039}}{.005} & \ms{\best{0.000}}{.000} & \ms{0.250}{.000} & \ms{\second{0.005}}{.000} \\
FlowTS & \ms{\second{0.112}}{.008} & \ms{\best{0.098}}{.002} & \ms{\best{0.036}}{.000} & \ms{\best{0.036}}{.003} & \ms{0.001}{.000} & \ms{0.251}{.000} & \ms{0.006}{.001} \\
\rowcolor{ourscolor} \textbf{Ours} & \ms{\second{0.112}}{.008} & \ms{\second{0.101}}{.009} & \ms{\best{0.036}}{.000} & \ms{0.040}{.004} & \ms{\best{0.000}}{.000} & \ms{\best{0.248}}{.001} & \ms{\best{0.004}}{.001} \\
\bottomrule
\end{tabular}
\end{table}

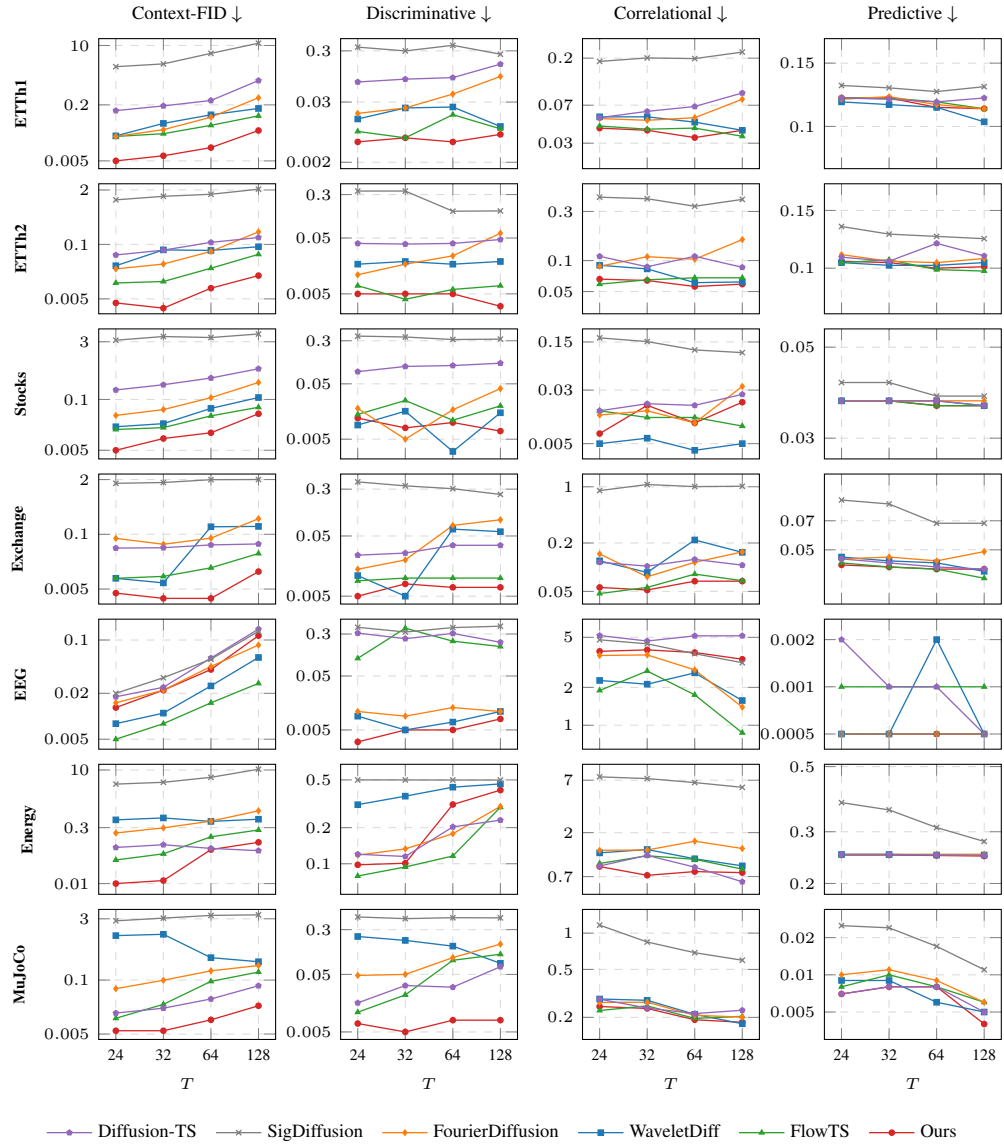
\begin{figure}[t]
  \centering
  \begin{tikzpicture}
  \begin{groupplot}[
      group style={group size=4 by 7, horizontal sep=0.86cm, vertical sep=0.20cm,
                   xticklabels at=edge bottom},
      scale only axis, width=2.34cm, height=1.72cm,
      symbolic x coords={24,32,64,128}, xtick=data, enlarge x limits=0.12,
      ymode=log, log ticks with fixed point,
      grid=both, major grid style={gray!28,dashed}, minor grid style={gray!11},
      tick label style={font=\tiny}, label style={font=\tiny},
      title style={font=\scriptsize, yshift=-0.35em},
      every axis plot/.append style={line width=0.5pt, mark size=1pt},
  ]

  \nextgroupplot[title={Context-FID $\downarrow$}, ylabel={\textbf{ETTh1}}, ymin=0.003, ymax=15.68565, ytick={0.005,0.2,10}, legend to name=appmetriclegend, legend style={legend columns=6, draw=none, fill=none, font=\scriptsize, /tikz/every even column/.append style={column sep=5pt}}]
  \addplot[colOurs, mark=*, forget plot] coordinates {(24,0.005) (32,0.007) (64,0.012) (128,0.037)};
  \addplot[colWave, mark=square*, forget plot] coordinates {(24,0.026) (32,0.059) (64,0.104) (128,0.158)};
  \addplot[colFlow, mark=triangle*, forget plot] coordinates {(24,0.025) (32,0.03) (64,0.052) (128,0.097)};
  \addplot[colFour, mark=diamond*, forget plot] coordinates {(24,0.025) (32,0.039) (64,0.089) (128,0.319)};
  \addplot[colDTS, mark=pentagon*, forget plot] coordinates {(24,0.137) (32,0.187) (64,0.267) (128,0.994)};
  \addplot[colSig, mark=x, mark size=1.4pt, forget plot] coordinates {(24,2.473) (32,2.972) (64,5.948) (128,11.619)};
  \addlegendimage{colDTS, mark=pentagon*, line width=0.5pt, mark size=1pt}
  \addlegendentry{Diffusion-TS}
  \addlegendimage{colSig, mark=x, line width=0.5pt, mark size=1.4pt}
  \addlegendentry{SigDiffusion}
  \addlegendimage{colFour, mark=diamond*, line width=0.5pt, mark size=1pt}
  \addlegendentry{FourierDiffusion}
  \addlegendimage{colWave, mark=square*, line width=0.5pt, mark size=1pt}
  \addlegendentry{WaveletDiff}
  \addlegendimage{colFlow, mark=triangle*, line width=0.5pt, mark size=1pt}
  \addlegendentry{FlowTS}
  \addlegendimage{colOurs, mark=*, line width=0.5pt, mark size=1pt}
  \addlegendentry{Ours}

  \nextgroupplot[title={Discriminative $\downarrow$}, ymin=0.0015, ymax=0.5198, ytick={0.002,0.03,0.3}]
  \addplot[colOurs, mark=*] coordinates {(24,0.005) (32,0.006) (64,0.005) (128,0.007)};
  \addplot[colWave, mark=square*] coordinates {(24,0.014) (32,0.023) (64,0.024) (128,0.01)};
  \addplot[colFlow, mark=triangle*] coordinates {(24,0.008) (32,0.006) (64,0.017) (128,0.009)};
  \addplot[colFour, mark=diamond*] coordinates {(24,0.018) (32,0.023) (64,0.043) (128,0.095)};
  \addplot[colDTS, mark=pentagon*] coordinates {(24,0.074) (32,0.084) (64,0.09) (128,0.165)};
  \addplot[colSig, mark=x, mark size=1.4pt] coordinates {(24,0.355) (32,0.299) (64,0.385) (128,0.258)};

  \nextgroupplot[title={Correlational $\downarrow$}, ymin=0.017, ymax=0.3092, ytick={0.03,0.07,0.2}]
  \addplot[colOurs, mark=*] coordinates {(24,0.042) (32,0.04) (64,0.034) (128,0.04)};
  \addplot[colWave, mark=square*] coordinates {(24,0.054) (32,0.054) (64,0.048) (128,0.04)};
  \addplot[colFlow, mark=triangle*] coordinates {(24,0.044) (32,0.041) (64,0.042) (128,0.035)};
  \addplot[colFour, mark=diamond*] coordinates {(24,0.052) (32,0.05) (64,0.053) (128,0.08)};
  \addplot[colDTS, mark=pentagon*] coordinates {(24,0.053) (32,0.061) (64,0.068) (128,0.092)};
  \addplot[colSig, mark=x, mark size=1.4pt] coordinates {(24,0.186) (32,0.201) (64,0.198) (128,0.229)};

  \nextgroupplot[title={Predictive $\downarrow$}, ymin=0.0763, ymax=0.1755, ytick={0.1,0.15}]
  \addplot[colOurs, mark=*] coordinates {(24,0.12) (32,0.12) (64,0.113) (128,0.112)};
  \addplot[colWave, mark=square*] coordinates {(24,0.117) (32,0.115) (64,0.113) (128,0.103)};
  \addplot[colFlow, mark=triangle*] coordinates {(24,0.12) (32,0.12) (64,0.117) (128,0.112)};
  \addplot[colFour, mark=diamond*] coordinates {(24,0.119) (32,0.121) (64,0.115) (128,0.112)};
  \addplot[colDTS, mark=pentagon*] coordinates {(24,0.12) (32,0.119) (64,0.117) (128,0.12)};
  \addplot[colSig, mark=x, mark size=1.4pt] coordinates {(24,0.13) (32,0.128) (64,0.125) (128,0.129)};

  \nextgroupplot[ylabel={\textbf{ETTh2}}, ymin=0.0022, ymax=2.81745, ytick={0.005,0.1,2}]
  \addplot[colOurs, mark=*] coordinates {(24,0.004) (32,0.003) (64,0.009) (128,0.018)};
  \addplot[colWave, mark=square*] coordinates {(24,0.031) (32,0.074) (64,0.072) (128,0.088)};
  \addplot[colFlow, mark=triangle*] coordinates {(24,0.012) (32,0.013) (64,0.027) (128,0.058)};
  \addplot[colFour, mark=diamond*] coordinates {(24,0.026) (32,0.034) (64,0.068) (128,0.2)};
  \addplot[colDTS, mark=pentagon*] coordinates {(24,0.056) (32,0.073) (64,0.112) (128,0.146)};
  \addplot[colSig, mark=x, mark size=1.4pt] coordinates {(24,1.162) (32,1.413) (64,1.581) (128,2.087)};

  \nextgroupplot[ymin=0.0022, ymax=0.4684, ytick={0.005,0.05,0.3}]
  \addplot[colOurs, mark=*] coordinates {(24,0.005) (32,0.005) (64,0.005) (128,0.003)};
  \addplot[colWave, mark=square*] coordinates {(24,0.017) (32,0.019) (64,0.017) (128,0.019)};
  \addplot[colFlow, mark=triangle*] coordinates {(24,0.007) (32,0.004) (64,0.006) (128,0.007)};
  \addplot[colFour, mark=diamond*] coordinates {(24,0.011) (32,0.017) (64,0.024) (128,0.061)};
  \addplot[colDTS, mark=pentagon*] coordinates {(24,0.04) (32,0.039) (64,0.04) (128,0.047)};
  \addplot[colSig, mark=x, mark size=1.4pt] coordinates {(24,0.347) (32,0.347) (64,0.151) (128,0.153)};

  \nextgroupplot[ymin=0.0304, ymax=0.5575, ytick={0.05,0.1,0.3}]
  \addplot[colOurs, mark=*] coordinates {(24,0.066) (32,0.064) (64,0.056) (128,0.059)};
  \addplot[colWave, mark=square*] coordinates {(24,0.09) (32,0.083) (64,0.061) (128,0.062)};
  \addplot[colFlow, mark=triangle*] coordinates {(24,0.059) (32,0.065) (64,0.068) (128,0.068)};
  \addplot[colFour, mark=diamond*] coordinates {(24,0.088) (32,0.109) (64,0.103) (128,0.16)};
  \addplot[colDTS, mark=pentagon*] coordinates {(24,0.11) (32,0.087) (64,0.11) (128,0.086)};
  \addplot[colSig, mark=x, mark size=1.4pt] coordinates {(24,0.413) (32,0.399) (64,0.337) (128,0.393)};

  \nextgroupplot[ymin=0.0726, ymax=0.1809, ytick={0.1,0.15}]
  \addplot[colOurs, mark=*] coordinates {(24,0.105) (32,0.104) (64,0.1) (128,0.101)};
  \addplot[colWave, mark=square*] coordinates {(24,0.104) (32,0.102) (64,0.102) (128,0.104)};
  \addplot[colFlow, mark=triangle*] coordinates {(24,0.104) (32,0.106) (64,0.099) (128,0.098)};
  \addplot[colFour, mark=diamond*] coordinates {(24,0.11) (32,0.105) (64,0.104) (128,0.107)};
  \addplot[colDTS, mark=pentagon*] coordinates {(24,0.108) (32,0.105) (64,0.119) (128,0.109)};
  \addplot[colSig, mark=x, mark size=1.4pt] coordinates {(24,0.134) (32,0.127) (64,0.125) (128,0.123)};

  \nextgroupplot[ylabel={\textbf{Stocks}}, ymin=0.003, ymax=6.40305, ytick={0.005,0.1,3}]
  \addplot[colOurs, mark=*] coordinates {(24,0.005) (32,0.01) (64,0.014) (128,0.043)};
  \addplot[colWave, mark=square*] coordinates {(24,0.02) (32,0.024) (64,0.059) (128,0.112)};
  \addplot[colFlow, mark=triangle*] coordinates {(24,0.017) (32,0.019) (64,0.038) (128,0.063)};
  \addplot[colFour, mark=diamond*] coordinates {(24,0.039) (32,0.055) (64,0.111) (128,0.272)};
  \addplot[colDTS, mark=pentagon*] coordinates {(24,0.175) (32,0.238) (64,0.354) (128,0.611)};
  \addplot[colSig, mark=x, mark size=1.4pt] coordinates {(24,3.285) (32,4.086) (64,3.851) (128,4.743)};

  \nextgroupplot[ymin=0.0022, ymax=0.4914, ytick={0.005,0.05,0.3}]
  \addplot[colOurs, mark=*] coordinates {(24,0.012) (32,0.008) (64,0.01) (128,0.007)};
  \addplot[colWave, mark=square*] coordinates {(24,0.009) (32,0.016) (64,0.003) (128,0.015)};
  \addplot[colFlow, mark=triangle*] coordinates {(24,0.014) (32,0.025) (64,0.011) (128,0.02)};
  \addplot[colFour, mark=diamond*] coordinates {(24,0.018) (32,0.005) (64,0.017) (128,0.041)};
  \addplot[colDTS, mark=pentagon*] coordinates {(24,0.083) (32,0.103) (64,0.107) (128,0.118)};
  \addplot[colSig, mark=x, mark size=1.4pt] coordinates {(24,0.364) (32,0.351) (64,0.316) (128,0.321)};

  \nextgroupplot[ymin=0.003, ymax=0.2322, ytick={0.005,0.03,0.15}]
  \addplot[colOurs, mark=*] coordinates {(24,0.007) (32,0.018) (64,0.01) (128,0.02)};
  \addplot[colWave, mark=square*] coordinates {(24,0.005) (32,0.006) (64,0.004) (128,0.005)};
  \addplot[colFlow, mark=triangle*] coordinates {(24,0.015) (32,0.012) (64,0.012) (128,0.009)};
  \addplot[colFour, mark=diamond*] coordinates {(24,0.013) (32,0.015) (64,0.01) (128,0.034)};
  \addplot[colDTS, mark=pentagon*] coordinates {(24,0.015) (32,0.019) (64,0.018) (128,0.026)};
  \addplot[colSig, mark=x, mark size=1.4pt] coordinates {(24,0.172) (32,0.152) (64,0.115) (128,0.105)};

  \nextgroupplot[ymin=0.0267, ymax=0.0554, ytick={0.03,0.05}]
  \addplot[colOurs, mark=*] coordinates {(24,0.037) (32,0.037) (64,0.036) (128,0.036)};
  \addplot[colWave, mark=square*] coordinates {(24,0.037) (32,0.037) (64,0.037) (128,0.036)};
  \addplot[colFlow, mark=triangle*] coordinates {(24,0.037) (32,0.037) (64,0.036) (128,0.036)};
  \addplot[colFour, mark=diamond*] coordinates {(24,0.037) (32,0.037) (64,0.037) (128,0.037)};
  \addplot[colDTS, mark=pentagon*] coordinates {(24,0.037) (32,0.037) (64,0.037) (128,0.036)};
  \addplot[colSig, mark=x, mark size=1.4pt] coordinates {(24,0.041) (32,0.041) (64,0.038) (128,0.038)};

  \nextgroupplot[ylabel={\textbf{Exchange}}, ymin=0.0022, ymax=2.70675, ytick={0.005,0.1,2}]
  \addplot[colOurs, mark=*] coordinates {(24,0.004) (32,0.003) (64,0.003) (128,0.013)};
  \addplot[colWave, mark=square*] coordinates {(24,0.009) (32,0.007) (64,0.151) (128,0.154)};
  \addplot[colFlow, mark=triangle*] coordinates {(24,0.009) (32,0.01) (64,0.016) (128,0.035)};
  \addplot[colFour, mark=diamond*] coordinates {(24,0.08) (32,0.058) (64,0.081) (128,0.235)};
  \addplot[colDTS, mark=pentagon*] coordinates {(24,0.047) (32,0.048) (64,0.056) (128,0.059)};
  \addplot[colSig, mark=x, mark size=1.4pt] coordinates {(24,1.629) (32,1.702) (64,1.986) (128,2.005)};

  \nextgroupplot[ymin=0.0037, ymax=0.5346, ytick={0.005,0.05,0.3}]
  \addplot[colOurs, mark=*] coordinates {(24,0.005) (32,0.008) (64,0.007) (128,0.007)};
  \addplot[colWave, mark=square*] coordinates {(24,0.011) (32,0.005) (64,0.065) (128,0.059)};
  \addplot[colFlow, mark=triangle*] coordinates {(24,0.009) (32,0.01) (64,0.01) (128,0.01)};
  \addplot[colFour, mark=diamond*] coordinates {(24,0.014) (32,0.02) (64,0.075) (128,0.093)};
  \addplot[colDTS, mark=pentagon*] coordinates {(24,0.024) (32,0.026) (64,0.035) (128,0.035)};
  \addplot[colSig, mark=x, mark size=1.4pt] coordinates {(24,0.396) (32,0.34) (64,0.305) (128,0.245)};

  \nextgroupplot[ymin=0.0348, ymax=1.44045, ytick={0.05,0.2,1}]
  \addplot[colOurs, mark=*] coordinates {(24,0.056) (32,0.052) (64,0.067) (128,0.067)};
  \addplot[colWave, mark=square*] coordinates {(24,0.12) (32,0.086) (64,0.218) (128,0.153)};
  \addplot[colFlow, mark=triangle*] coordinates {(24,0.047) (32,0.056) (64,0.082) (128,0.068)};
  \addplot[colFour, mark=diamond*] coordinates {(24,0.146) (32,0.076) (64,0.115) (128,0.155)};
  \addplot[colDTS, mark=pentagon*] coordinates {(24,0.115) (32,0.103) (64,0.125) (128,0.106)};
  \addplot[colSig, mark=x, mark size=1.4pt] coordinates {(24,0.896) (32,1.067) (64,1.005) (128,1.015)};

  \nextgroupplot[ymin=0.0267, ymax=0.1202, ytick={0.05,0.07}]
  \addplot[colOurs, mark=*] coordinates {(24,0.042) (32,0.041) (64,0.04) (128,0.04)};
  \addplot[colWave, mark=square*] coordinates {(24,0.046) (32,0.044) (64,0.043) (128,0.039)};
  \addplot[colFlow, mark=triangle*] coordinates {(24,0.043) (32,0.041) (64,0.04) (128,0.036)};
  \addplot[colFour, mark=diamond*] coordinates {(24,0.045) (32,0.046) (64,0.044) (128,0.049)};
  \addplot[colDTS, mark=pentagon*] coordinates {(24,0.045) (32,0.043) (64,0.041) (128,0.04)};
  \addplot[colSig, mark=x, mark size=1.4pt] coordinates {(24,0.089) (32,0.085) (64,0.068) (128,0.068)};

  \nextgroupplot[ylabel={\textbf{EEG}}, ymin=0.0037, ymax=0.18765, ytick={0.005,0.02,0.1}]
  \addplot[colOurs, mark=*] coordinates {(24,0.013) (32,0.022) (64,0.041) (128,0.114)};
  \addplot[colWave, mark=square*] coordinates {(24,0.008) (32,0.011) (64,0.025) (128,0.059)};
  \addplot[colFlow, mark=triangle*] coordinates {(24,0.005) (32,0.008) (64,0.015) (128,0.027)};
  \addplot[colFour, mark=diamond*] coordinates {(24,0.015) (32,0.022) (64,0.045) (128,0.086)};
  \addplot[colDTS, mark=pentagon*] coordinates {(24,0.018) (32,0.024) (64,0.058) (128,0.139)};
  \addplot[colSig, mark=x, mark size=1.4pt] coordinates {(24,0.02) (32,0.032) (64,0.056) (128,0.13)};

  \nextgroupplot[ymin=0.0022, ymax=0.56565, ytick={0.005,0.05,0.3}]
  \addplot[colOurs, mark=*] coordinates {(24,0.003) (32,0.005) (64,0.005) (128,0.008)};
  \addplot[colWave, mark=square*] coordinates {(24,0.009) (32,0.005) (64,0.007) (128,0.011)};
  \addplot[colFlow, mark=triangle*] coordinates {(24,0.106) (32,0.385) (64,0.221) (128,0.176)};
  \addplot[colFour, mark=diamond*] coordinates {(24,0.011) (32,0.009) (64,0.013) (128,0.011)};
  \addplot[colDTS, mark=pentagon*] coordinates {(24,0.312) (32,0.246) (64,0.308) (128,0.211)};
  \addplot[colSig, mark=x, mark size=1.4pt] coordinates {(24,0.4) (32,0.327) (64,0.395) (128,0.419)};

  \nextgroupplot[ymin=0.6452, ymax=6.9714, ytick={1,2,5}]
  \addplot[colOurs, mark=*] coordinates {(24,3.87) (32,3.972) (64,3.788) (128,3.348)};
  \addplot[colWave, mark=square*] coordinates {(24,2.27) (32,2.119) (64,2.608) (128,1.572)};
  \addplot[colFlow, mark=triangle*] coordinates {(24,1.89) (32,2.698) (64,1.744) (128,0.871)};
  \addplot[colFour, mark=diamond*] coordinates {(24,3.585) (32,3.623) (64,2.755) (128,1.394)};
  \addplot[colDTS, mark=pentagon*] coordinates {(24,5.164) (32,4.672) (64,5.136) (128,5.141)};
  \addplot[colSig, mark=x, mark size=1.4pt] coordinates {(24,4.777) (32,4.439) (64,3.706) (128,3.141)};

  \nextgroupplot[ymin=0.0004, ymax=0.0027, ytick={0.0005,0.001,0.002}]
  \addplot[colOurs, mark=*] coordinates {(24,0.0005) (32,0.0005) (64,0.0005) (128,0.0005)};
  \addplot[colWave, mark=square*] coordinates {(24,0.0005) (32,0.0005) (64,0.002) (128,0.0005)};
  \addplot[colFlow, mark=triangle*] coordinates {(24,0.001) (32,0.001) (64,0.001) (128,0.001)};
  \addplot[colFour, mark=diamond*] coordinates {(24,0.0005) (32,0.0005) (64,0.0005) (128,0.0005)};
  \addplot[colDTS, mark=pentagon*] coordinates {(24,0.002) (32,0.001) (64,0.001) (128,0.0005)};
  \addplot[colSig, mark=x, mark size=1.4pt] coordinates {(24,0.0005) (32,0.0005) (64,0.0005) (128,0.0005)};

  \nextgroupplot[ylabel={\textbf{Energy}}, ymin=0.0052, ymax=14.13585, ytick={0.01,0.3,10}]
  \addplot[colOurs, mark=*] coordinates {(24,0.01) (32,0.012) (64,0.079) (128,0.123)};
  \addplot[colWave, mark=square*] coordinates {(24,0.482) (32,0.541) (64,0.438) (128,0.501)};
  \addplot[colFlow, mark=triangle*] coordinates {(24,0.042) (32,0.061) (64,0.172) (128,0.26)};
  \addplot[colFour, mark=diamond*] coordinates {(24,0.217) (32,0.296) (64,0.446) (128,0.837)};
  \addplot[colDTS, mark=pentagon*] coordinates {(24,0.09) (32,0.106) (64,0.085) (128,0.074)};
  \addplot[colSig, mark=x, mark size=1.4pt] coordinates {(24,4.241) (32,4.73) (64,6.403) (128,10.471)};

  \nextgroupplot[ymin=0.0556, ymax=0.675, ytick={0.1,0.2,0.5}]
  \addplot[colOurs, mark=*] coordinates {(24,0.098) (32,0.101) (64,0.312) (128,0.411)};
  \addplot[colWave, mark=square*] coordinates {(24,0.311) (32,0.366) (64,0.435) (128,0.463)};
  \addplot[colFlow, mark=triangle*] coordinates {(24,0.079) (32,0.094) (64,0.116) (128,0.296)};
  \addplot[colFour, mark=diamond*] coordinates {(24,0.118) (32,0.133) (64,0.178) (128,0.301)};
  \addplot[colDTS, mark=pentagon*] coordinates {(24,0.12) (32,0.115) (64,0.203) (128,0.231)};
  \addplot[colSig, mark=x, mark size=1.4pt] coordinates {(24,0.5) (32,0.5) (64,0.499) (128,0.499)};

  \nextgroupplot[ymin=0.4585, ymax=10.23165, ytick={0.7,2,7}]
  \addplot[colOurs, mark=*] coordinates {(24,0.89) (32,0.722) (64,0.789) (128,0.77)};
  \addplot[colWave, mark=square*] coordinates {(24,1.234) (32,1.337) (64,1.071) (128,0.905)};
  \addplot[colFlow, mark=triangle*] coordinates {(24,0.957) (32,1.152) (64,1.057) (128,0.837)};
  \addplot[colFour, mark=diamond*] coordinates {(24,1.312) (32,1.321) (64,1.631) (128,1.369)};
  \addplot[colDTS, mark=pentagon*] coordinates {(24,0.902) (32,1.162) (64,0.877) (128,0.619)};
  \addplot[colSig, mark=x, mark size=1.4pt] coordinates {(24,7.579) (32,7.274) (64,6.615) (128,5.903)};

  \nextgroupplot[ymin=0.1837, ymax=0.509, ytick={0.2,0.3,0.5}]
  \addplot[colOurs, mark=*] coordinates {(24,0.25) (32,0.25) (64,0.249) (128,0.248)};
  \addplot[colWave, mark=square*] coordinates {(24,0.251) (32,0.251) (64,0.25) (128,0.25)};
  \addplot[colFlow, mark=triangle*] coordinates {(24,0.251) (32,0.251) (64,0.251) (128,0.251)};
  \addplot[colFour, mark=diamond*] coordinates {(24,0.251) (32,0.251) (64,0.251) (128,0.251)};
  \addplot[colDTS, mark=pentagon*] coordinates {(24,0.251) (32,0.251) (64,0.25) (128,0.249)};
  \addplot[colSig, mark=x, mark size=1.4pt] coordinates {(24,0.377) (32,0.356) (64,0.31) (128,0.278)};

  \nextgroupplot[ylabel={\textbf{MuJoCo}}, xlabel={$T$}, ymin=0.0037, ymax=5.0463, ytick={0.005,0.1,3}]
  \addplot[colOurs, mark=*] coordinates {(24,0.006) (32,0.006) (64,0.011) (128,0.024)};
  \addplot[colWave, mark=square*] coordinates {(24,1.18) (32,1.269) (64,0.347) (128,0.277)};
  \addplot[colFlow, mark=triangle*] coordinates {(24,0.012) (32,0.026) (64,0.093) (128,0.156)};
  \addplot[colFour, mark=diamond*] coordinates {(24,0.062) (32,0.099) (64,0.167) (128,0.224)};
  \addplot[colDTS, mark=pentagon*] coordinates {(24,0.016) (32,0.021) (64,0.035) (128,0.073)};
  \addplot[colSig, mark=x, mark size=1.4pt] coordinates {(24,2.681) (32,3.117) (64,3.621) (128,3.738)};

  \nextgroupplot[xlabel={$T$}, ymin=0.0037, ymax=0.6737, ytick={0.005,0.05,0.3}]
  \addplot[colOurs, mark=*] coordinates {(24,0.007) (32,0.005) (64,0.008) (128,0.008)};
  \addplot[colWave, mark=square*] coordinates {(24,0.228) (32,0.195) (64,0.155) (128,0.078)};
  \addplot[colFlow, mark=triangle*] coordinates {(24,0.011) (32,0.022) (64,0.088) (128,0.112)};
  \addplot[colFour, mark=diamond*] coordinates {(24,0.048) (32,0.05) (64,0.098) (128,0.168)};
  \addplot[colDTS, mark=pentagon*] coordinates {(24,0.016) (32,0.032) (64,0.03) (128,0.068)};
  \addplot[colSig, mark=x, mark size=1.4pt] coordinates {(24,0.499) (32,0.466) (64,0.484) (128,0.481)};

  \nextgroupplot[xlabel={$T$}, ymin=0.1311, ymax=1.57275, ytick={0.2,0.5,1}]
  \addplot[colOurs, mark=*] coordinates {(24,0.246) (32,0.237) (64,0.191) (128,0.182)};
  \addplot[colWave, mark=square*] coordinates {(24,0.284) (32,0.277) (64,0.212) (128,0.177)};
  \addplot[colFlow, mark=triangle*] coordinates {(24,0.228) (32,0.249) (64,0.196) (128,0.204)};
  \addplot[colFour, mark=diamond*] coordinates {(24,0.266) (32,0.267) (64,0.208) (128,0.201)};
  \addplot[colDTS, mark=pentagon*] coordinates {(24,0.283) (32,0.239) (64,0.215) (128,0.229)};
  \addplot[colSig, mark=x, mark size=1.4pt] coordinates {(24,1.165) (32,0.844) (64,0.688) (128,0.596)};

  \nextgroupplot[xlabel={$T$}, ymin=0.003, ymax=0.0338, ytick={0.005,0.01,0.02}]
  \addplot[colOurs, mark=*] coordinates {(24,0.007) (32,0.008) (64,0.008) (128,0.004)};
  \addplot[colWave, mark=square*] coordinates {(24,0.009) (32,0.009) (64,0.006) (128,0.005)};
  \addplot[colFlow, mark=triangle*] coordinates {(24,0.008) (32,0.01) (64,0.008) (128,0.006)};
  \addplot[colFour, mark=diamond*] coordinates {(24,0.01) (32,0.011) (64,0.009) (128,0.006)};
  \addplot[colDTS, mark=pentagon*] coordinates {(24,0.007) (32,0.008) (64,0.008) (128,0.005)};
  \addplot[colSig, mark=x, mark size=1.4pt] coordinates {(24,0.025) (32,0.024) (64,0.017) (128,0.011)};

  \end{groupplot}
  \end{tikzpicture}

  \vspace{0.5em}
  \pgfplotslegendfromname{appmetriclegend}

  \caption{Evolution of each metric with the window length $T$.}
  \label{fig:app-grid}
\end{figure}

\FloatBarrier
\subsection{Qualitative comparison on all datasets}
\label{app:qualitative}

\Cref{fig:app-tsne,fig:app-density} show the two qualitative diagnostics for every method and dataset at $T=24$.

\begin{figure}[!htbp]
  \centering
  \setlength{\tabcolsep}{1pt}
  \renewcommand{\arraystretch}{0.6}
  \settsneh{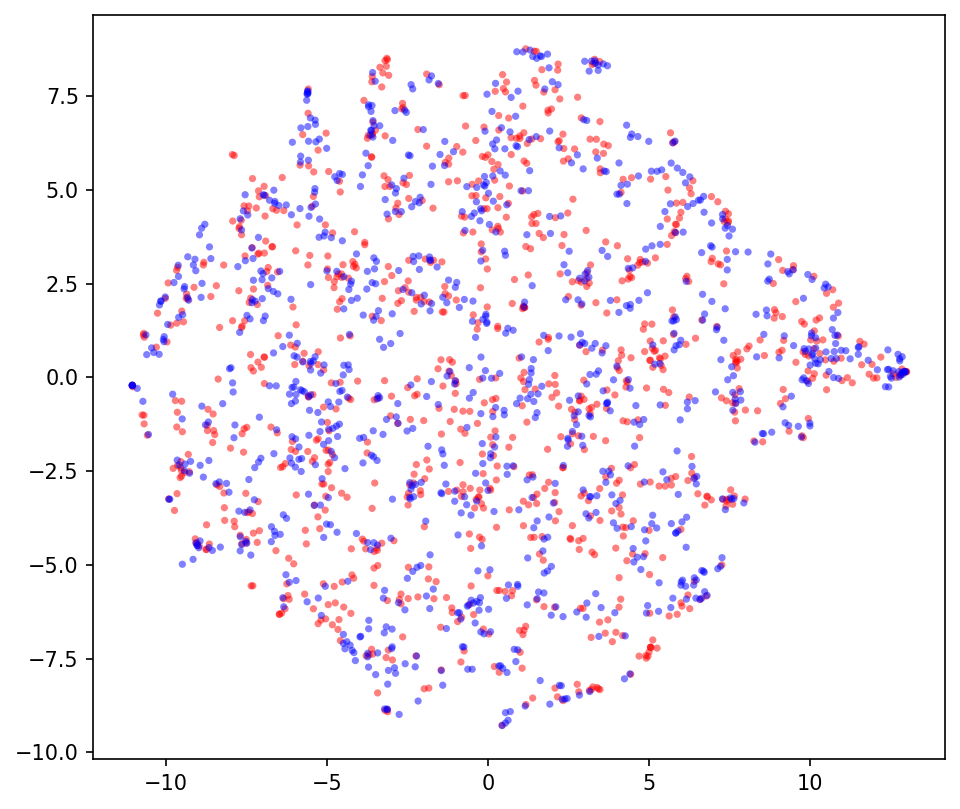}
  \begin{tabular}{c@{\hspace{2pt}}cccccc}
    \rowlab{ETTh1}    & \tsnecell{etth1}{ours}    & \tsnecell{etth1}{waveletdiff}    & \tsnecell{etth1}{flowts}    & \tsnecell{etth1}{fourier}    & \tsnecell{etth1}{difftts}    & \tsnecell{etth1}{sigdiff} \\
    \rowlab{ETTh2}    & \tsnecell{etth2}{ours}    & \tsnecell{etth2}{waveletdiff}    & \tsnecell{etth2}{flowts}    & \tsnecell{etth2}{fourier}    & \tsnecell{etth2}{difftts}    & \tsnecell{etth2}{sigdiff} \\
    \rowlab{Stocks}   & \tsnecell{stocks}{ours}   & \tsnecell{stocks}{waveletdiff}   & \tsnecell{stocks}{flowts}   & \tsnecell{stocks}{fourier}   & \tsnecell{stocks}{difftts}   & \tsnecell{stocks}{sigdiff} \\
    \rowlab{Exchange} & \tsnecell{exch}{ours} & \tsnecell{exch}{waveletdiff} & \tsnecell{exch}{flowts} & \tsnecell{exch}{fourier} & \tsnecell{exch}{difftts} & \tsnecell{exch}{sigdiff} \\
    \rowlab{EEG}      & \tsnecell{eeg}{ours}      & \tsnecell{eeg}{waveletdiff}      & \tsnecell{eeg}{flowts}      & \tsnecell{eeg}{fourier}      & \tsnecell{eeg}{difftts}      & \tsnecell{eeg}{sigdiff} \\
    \rowlab{Energy}   & \tsnecell{energy}{ours}   & \tsnecell{energy}{waveletdiff}   & \tsnecell{energy}{flowts}   & \tsnecell{energy}{fourier}   & \tsnecell{energy}{difftts}   & \tsnecell{energy}{sigdiff} \\
    \rowlab{MuJoCo}   & \tsnecell{mujoco}{ours}   & \tsnecell{mujoco}{waveletdiff}   & \tsnecell{mujoco}{flowts}   & \tsnecell{mujoco}{fourier}   & \tsnecell{mujoco}{difftts}   & \tsnecell{mujoco}{sigdiff} \\[2pt]
     & {\scriptsize (a) Ours} & {\scriptsize (b) WaveletDiff} & {\scriptsize (c) FlowTS} & {\scriptsize (d) FourierDiffusion} & {\scriptsize (e) Diffusion-TS} & {\scriptsize (f) SigDiffusion} \\
  \end{tabular}
  \caption{t-SNE embeddings of real (red) and generated (blue) windows for all methods and datasets at $T=24$.}
  \label{fig:app-tsne}
\end{figure}

\begin{figure}[!htbp]
  \centering
  \setlength{\tabcolsep}{1pt}
  \renewcommand{\arraystretch}{0.6}
  \setlength{\panelw}{\dimexpr(\textwidth-\rowlabw-\pdlabw-13pt)/6\relax}
  \setgridh{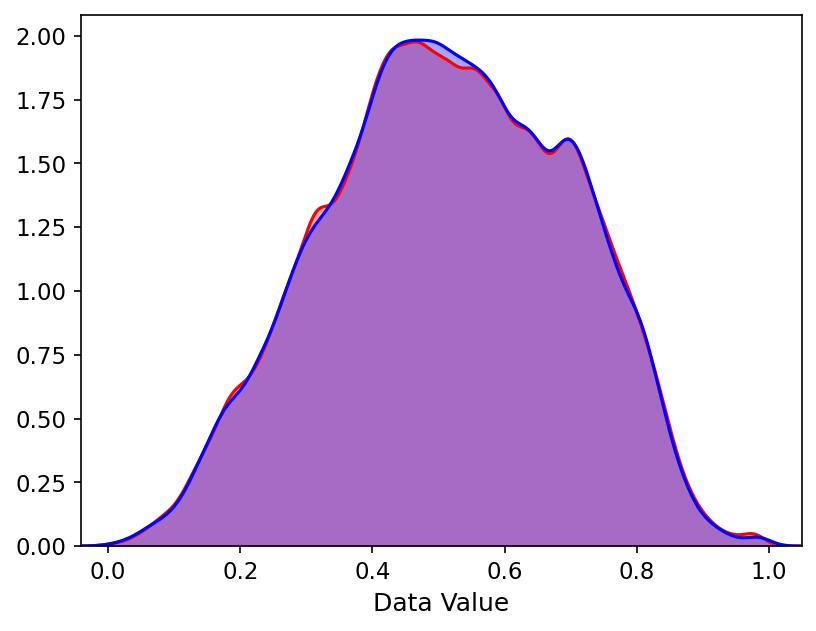}
  \begin{tabular}{@{}c@{\hspace{1pt}}c@{\hspace{2pt}}cccccc@{}}
    \rowlab{ETTh1} & \pdlab & \denscell{etth1}{ours} & \denscell{etth1}{waveletdiff} & \denscell{etth1}{flowts} & \denscell{etth1}{fourier} & \denscell{etth1}{difftts} & \denscell{etth1}{sigdiff} \\
    \rowlab{ETTh2} & \pdlab & \denscell{etth2}{ours} & \denscell{etth2}{waveletdiff} & \denscell{etth2}{flowts} & \denscell{etth2}{fourier} & \denscell{etth2}{difftts} & \denscell{etth2}{sigdiff} \\
    \rowlab{Stocks} & \pdlab & \denscell{stocks}{ours} & \denscell{stocks}{waveletdiff} & \denscell{stocks}{flowts} & \denscell{stocks}{fourier} & \denscell{stocks}{difftts} & \denscell{stocks}{sigdiff} \\
    \rowlab{Exchange} & \pdlab & \denscell{exch}{ours} & \denscell{exch}{waveletdiff} & \denscell{exch}{flowts} & \denscell{exch}{fourier} & \denscell{exch}{difftts} & \denscell{exch}{sigdiff} \\
    \rowlab{EEG} & \pdlab & \denscell{eeg}{ours} & \denscell{eeg}{waveletdiff} & \denscell{eeg}{flowts} & \denscell{eeg}{fourier} & \denscell{eeg}{difftts} & \denscell{eeg}{sigdiff} \\
    \rowlab{Energy} & \pdlab & \denscell{energy}{ours} & \denscell{energy}{waveletdiff} & \denscell{energy}{flowts} & \denscell{energy}{fourier} & \denscell{energy}{difftts} & \denscell{energy}{sigdiff} \\
    \rowlab{MuJoCo} & \pdlab & \denscell{mujoco}{ours} & \denscell{mujoco}{waveletdiff} & \denscell{mujoco}{flowts} & \denscell{mujoco}{fourier} & \denscell{mujoco}{difftts} & \denscell{mujoco}{sigdiff} \\[2pt]
     & & {\scriptsize (a) Ours} & {\scriptsize (b) WaveletDiff} & {\scriptsize (c) FlowTS} & {\scriptsize (d) FourierDiffusion} & {\scriptsize (e) Diffusion-TS} & {\scriptsize (f) SigDiffusion} \\
  \end{tabular}
  \caption{Probability distributions of real (red) and generated (blue) data for all methods and datasets at $T=24$.}
  \label{fig:app-density}
\end{figure}


\FloatBarrier
\subsection{Method component ablation}
\label{app:full_abl}

\Cref{tab:ablations} reports one-factor-at-a-time ablations of the main model configuration at $T=24$. The full model uses \texttt{db2} wavelet coefficients, natural level scaling, logit-normal sampling of the training time $t$, and channel-identity embeddings $e_f$. Each ablation changes exactly one of these components while keeping the remaining architecture, training, and sampling settings fixed.

For \emph{standardized levels}, each wavelet level is divided by its standard deviation from \Cref{tab:level-energy} before training and rescaled before the inverse transform, instead of retaining the natural level scaling described in \Cref{sec:method}. The \emph{uniform $t$} ablation replaces logit-normal time sampling with uniform sampling.

For the \emph{without $e_f$} ablation, removing the channel-identity embedding from \Cref{eq:token} makes every parameter shared across tokens, with none depending on the channel. Hence
$\vth(zP,t)=\vth(z,t)P$ for every permutation matrix $P$ of the channels. Since the $\DWT$ acts channel-wise and the prior is invariant under $P$, the generated law is exchangeable in the channels: all off-diagonal lag-zero correlations share a common value $\rho$. No choice of $\rho$ can match a real correlation matrix that is not equicorrelated, so \Cref{eq:corr-score} is bounded below by $\tfrac{1}{10}\sum_{i<j}\lvert\rho^{\mathrm{real}}_{ij}-\rho^{\star}\rvert$, where $\rho^{\star}$ is the median of the real off-diagonal correlations.

Finally, the \emph{time-domain} ablation replaces the \texttt{db2} transform by the identity, $\W=I$, and trains the same model directly on the time-domain signal.

\begin{table}[t]
\centering
\footnotesize
\setlength{\tabcolsep}{3.5pt}
\resizebox{\textwidth}{!}{%
\begin{tabular}{l *{7}{c}}
\toprule
\textbf{Variant}
& ETTh1 & ETTh2 & Stocks & Exchange & EEG & Energy & MuJoCo \\
\midrule
\multicolumn{8}{l}{\textit{\hspace{0.2cm}Context-FID} ($\downarrow$)}\\
(a) Without $e_f$
  & 1.488\pmstd{0.099} & 0.977\pmstd{0.168}
  & 1.763\pmstd{0.321} & 0.938\pmstd{0.072}
  & \underline{0.011}\pmstd{0.001} & 2.617\pmstd{0.145}
  & 1.390\pmstd{0.139} \\
(b) Std. levels
  & 0.021\pmstd{0.002} & 0.014\pmstd{0.001}
  & 0.018\pmstd{0.004} & 0.010\pmstd{0.001}
  & \textbf{0.010}\pmstd{0.001} & 0.019\pmstd{0.002}
  & 0.008\pmstd{0.000} \\
(c) Time domain
  & \underline{0.006}\pmstd{0.001} & \underline{0.005}\pmstd{0.001}
  & 0.014\pmstd{0.001} & 0.007\pmstd{0.000}
  & 0.014\pmstd{0.001} & \underline{0.012}\pmstd{0.000}
  & \textbf{0.005}\pmstd{0.000} \\
(d) Uniform $t$
  & 0.007\pmstd{0.000} & 0.008\pmstd{0.001}
  & \underline{0.013}\pmstd{0.001} & \underline{0.005}\pmstd{0.000}
  & 0.014\pmstd{0.002} & 0.016\pmstd{0.001}
  & 0.010\pmstd{0.001} \\
\rowcolor{ourscolor}
Full model
  & \textbf{0.005}\pmstd{0.001} & \textbf{0.004}\pmstd{0.000}
  & \textbf{0.005}\pmstd{0.002} & \textbf{0.004}\pmstd{0.001}
  & 0.013\pmstd{0.002} & \textbf{0.010}\pmstd{0.003}
  & \underline{0.006}\pmstd{0.000} \\
  
\midrule
\multicolumn{8}{l}{\textit{\hspace{0.2cm}Discriminative score} ($\downarrow$)}\\
(a) Without $e_f$
  & 0.245\pmstd{0.112} & 0.258\pmstd{0.100}
  & 0.299\pmstd{0.076} & 0.335\pmstd{0.014}
  & 0.060\pmstd{0.021} & 0.496\pmstd{0.001}
  & 0.434\pmstd{0.023} \\
(b) Std. levels
  & 0.009\pmstd{0.005} & \textbf{0.005}\pmstd{0.004}
  & 0.023\pmstd{0.010} & 0.010\pmstd{0.010}
  & \underline{0.005}\pmstd{0.004} & 0.133\pmstd{0.009}
  & \underline{0.009}\pmstd{0.004} \\
(c) Time domain
  & \underline{0.006}\pmstd{0.004} & \underline{0.008}\pmstd{0.004}
  & 0.016\pmstd{0.006} & 0.006\pmstd{0.002}
  & \underline{0.005}\pmstd{0.004} & 0.118\pmstd{0.004}
  & 0.010\pmstd{0.006} \\
(d) Uniform $t$
  & 0.009\pmstd{0.004} & \textbf{0.005}\pmstd{0.003}
  & \textbf{0.011}\pmstd{0.008} & \textbf{0.004}\pmstd{0.002}
  & \underline{0.005}\pmstd{0.003} & \textbf{0.092}\pmstd{0.019}
  & \textbf{0.007}\pmstd{0.002} \\
\rowcolor{ourscolor}
Full model
  & \textbf{0.005}\pmstd{0.004} & \textbf{0.005}\pmstd{0.006}
  & \underline{0.012}\pmstd{0.009} & \underline{0.005}\pmstd{0.003}
  & \textbf{0.003}\pmstd{0.002} & \underline{0.098}\pmstd{0.018}
  & \textbf{0.007}\pmstd{0.004} \\

\midrule
\multicolumn{8}{l}{\textit{\hspace{0.2cm}Correlational score} ($\downarrow$)}\\
(a) Without $e_f$
  & 0.394\pmstd{0.008} & 0.574\pmstd{0.015}
  & 1.023\pmstd{0.024} & 0.816\pmstd{0.027}
  & 4.143\pmstd{0.353} & 11.385\pmstd{0.085}
  & 0.922\pmstd{0.023} \\
(b) Std. levels
  & 0.064\pmstd{0.015} & 0.079\pmstd{0.015}
  & 0.008\pmstd{0.002} & 0.076\pmstd{0.026}
  & \underline{4.002}\pmstd{0.303} & 1.087\pmstd{0.331}
  & \textbf{0.223}\pmstd{0.035} \\
(c) Time domain
  & 0.045\pmstd{0.008} & 0.078\pmstd{0.022}
  & 0.008\pmstd{0.005} & 0.075\pmstd{0.038}
  & 4.363\pmstd{0.318} & 1.069\pmstd{0.287}
  & 0.251\pmstd{0.019} \\
(d) Uniform $t$
  & \textbf{0.039}\pmstd{0.017} & \underline{0.074}\pmstd{0.023}
  & \textbf{0.004}\pmstd{0.004} & \underline{0.069}\pmstd{0.028}
  & 4.429\pmstd{0.227} & \underline{0.973}\pmstd{0.130}
  & 0.265\pmstd{0.032} \\
\rowcolor{ourscolor}
Full model
  & \underline{0.042}\pmstd{0.015} & \textbf{0.066}\pmstd{0.020}
  & \underline{0.007}\pmstd{0.005} & \textbf{0.056}\pmstd{0.017}
  & \textbf{3.870}\pmstd{0.422} & \textbf{0.890}\pmstd{0.176}
  & \underline{0.246}\pmstd{0.033} \\

\midrule
\multicolumn{8}{l}{\textit{\hspace{0.2cm}Predictive score} ($\downarrow$)}\\
(a) Without $e_f$
  & 0.138\pmstd{0.001} & 0.165\pmstd{0.001}
  & \underline{0.090}\pmstd{0.035} & 0.126\pmstd{0.003}
  & \textbf{0.000}\pmstd{0.000} & 0.268\pmstd{0.001}
  & \underline{0.040}\pmstd{0.000} \\
(b) Std. levels
  & \textbf{0.119}\pmstd{0.005} & \underline{0.107}\pmstd{0.003}
  & \textbf{0.037}\pmstd{0.000} & 0.048\pmstd{0.003}
  & \textbf{0.000}\pmstd{0.000} & 0.252\pmstd{0.000}
  & \textbf{0.007}\pmstd{0.000} \\
(c) Time domain
  & 0.122\pmstd{0.003} & \underline{0.107}\pmstd{0.004}
  & \textbf{0.037}\pmstd{0.000} & \underline{0.042}\pmstd{0.006}
  & \textbf{0.000}\pmstd{0.000} & \textbf{0.250}\pmstd{0.000}
  & \textbf{0.007}\pmstd{0.000} \\
(d) Uniform $t$
  & 0.121\pmstd{0.002} & \textbf{0.105}\pmstd{0.004}
  & \textbf{0.037}\pmstd{0.000} & \textbf{0.040}\pmstd{0.005}
  & \textbf{0.000}\pmstd{0.000} & \underline{0.251}\pmstd{0.000}
  & \textbf{0.007}\pmstd{0.000} \\
\rowcolor{ourscolor}
Full model
  & \underline{0.120}\pmstd{0.003} & \textbf{0.105}\pmstd{0.004}
  & \textbf{0.037}\pmstd{0.000} & \underline{0.042}\pmstd{0.004}
  & \textbf{0.000}\pmstd{0.000} & \textbf{0.250}\pmstd{0.000}
  & \textbf{0.007}\pmstd{0.001} \\

\bottomrule
\end{tabular}}
\caption{One-factor-at-a-time ablations at $T=24$. The full model uses
\texttt{db2} coefficients, natural level scaling, logit-normal time sampling,
and channel-identity embeddings $e_f$. Variants (a)--(d) respectively remove
$e_f$, standardize the wavelet levels, replace the wavelet transform with the
identity, and sample $t$ uniformly. Results are mean $\pm$ standard deviation; Best values are in bold and second-best values are underlined.
The shaded rows show the full model.}
\label{tab:ablations}
\end{table}

\clearpage
\FloatBarrier
\subsection{Sampler settings}
\label{app:sampler}

The main experiments use deterministic Euler integration with $N=100$ uniformly spaced steps, as described in \Cref{eq:euler}. For the sampler ablations, we additionally vary both the number and
placement of integration steps.

\paragraph{Time-grid parameterization.}
Following~\citet{hu2024flowts}, we parameterize non-uniform grids through the increasing bijection of $[0,1]$
\begin{equation}
\label{eq:tshift}
  f_\alpha(\tau)
  = \frac{\alpha\,\tau}{1 + (\alpha-1)\,\tau},
  \qquad \alpha > 0.
\end{equation}
We set $t_k=f_\alpha(k/N)$. Values $\alpha>1$ concentrate steps near $t=1$, while $0<\alpha<1$ concentrates them near $t=0$; $\alpha=1$ recovers the uniform grid used by default.

\begin{algorithm}
\caption{Sampling by deterministic Euler integration in the wavelet domain}
\label{alg:sampling}
\begin{algorithmic}[1]
\Require trained velocity field $v_{\theta}$, inverse operator $\W^{-1}$, steps $N$, shift $\alpha$
\State $t_k \gets f_\alpha(k/N)$ for $k = 0,\dots,N$
  \Comment{time grid}
\State $z \gets z_0$ with i.i.d.\ $\N(0,1)$ entries
\For{$k = 0,\dots,N-1$}
  \State $z \gets z + (t_{k+1} - t_k)\; v_{\theta}(z,\, t_k)$
    \Comment{explicit Euler step}
\EndFor
\State $\hat x \gets \W^{-1}\, z$ \Comment{inverse wavelet transform}
\State \Return $\hat x$
\end{algorithmic}
\end{algorithm}

\paragraph{Sampler ablations.}
\Cref{tab:ablation-shift,tab:ablation-steps} report all four metrics at $T=24$ while varying the sampler. \Cref{tab:ablation-shift} varies the time-grid shift $\alpha$ at fixed $N=100$, whereas \Cref{tab:ablation-steps} varies the number of Euler steps $N$ on the uniform grid ($\alpha=1$). All configurations use the same trained model, so these comparisons isolate inference-time choices. The dagger denotes the default configuration.

\begin{table}[H]
  \centering
  \footnotesize
  \setlength{\tabcolsep}{4pt}
  \begin{tabular}{l *{6}{c}}
    \toprule
    & \multicolumn{6}{c}{Time shift $\alpha$} \\
    \cmidrule(lr){2-7}
    \textbf{Dataset} & $0.33$ & $0.5$ & $1.0^\dagger$ & $2.0$ & $3.0$ & $5.0$ \\
    \midrule
    \multicolumn{7}{l}{\textit{\hspace{0.2cm}Context-FID} ($\downarrow$)} \\
    \addlinespace[2pt]
    ETTh1         & 0.007\pmstd{0.000} & 0.008\pmstd{0.000} & \textbf{0.005}\pmstd{0.001} & 0.011\pmstd{0.001} & 0.010\pmstd{0.000} & 0.021\pmstd{0.001} \\
    ETTh2         & 0.008\pmstd{0.000} & 0.005\pmstd{0.000} & \textbf{0.004}\pmstd{0.000} & \textbf{0.004}\pmstd{0.000} & 0.005\pmstd{0.000} & 0.007\pmstd{0.000} \\
    Stocks        & 0.008\pmstd{0.001} & 0.011\pmstd{0.001} & \textbf{0.005}\pmstd{0.002} & 0.006\pmstd{0.001} & 0.010\pmstd{0.001} & 0.023\pmstd{0.003} \\
    Exchange & 0.005\pmstd{0.001} & 0.005\pmstd{0.000} & 0.004\pmstd{0.001} & \textbf{0.003}\pmstd{0.000} & 0.007\pmstd{0.000} & 0.008\pmstd{0.001} \\
    EEG           & \textbf{0.012}\pmstd{0.001} & 0.015\pmstd{0.002} & 0.013\pmstd{0.002} & 0.014\pmstd{0.004} & 0.019\pmstd{0.003} & 0.015\pmstd{0.002} \\
    Energy        & 0.012\pmstd{0.001} & 0.011\pmstd{0.000} & \textbf{0.010}\pmstd{0.003} & 0.012\pmstd{0.001} & 0.011\pmstd{0.002} & 0.016\pmstd{0.002} \\
    MuJoCo        & 0.007\pmstd{0.001} & 0.007\pmstd{0.000} & \textbf{0.006}\pmstd{0.000} & 0.007\pmstd{0.000} & 0.008\pmstd{0.000} & 0.016\pmstd{0.001} \\
    \midrule
    \multicolumn{7}{l}{\textit{\hspace{0.2cm}Discriminative score} ($\downarrow$)} \\
    \addlinespace[2pt]
    ETTh1         & 0.006\pmstd{0.002} & 0.007\pmstd{0.003} & \textbf{0.005}\pmstd{0.004} & \textbf{0.005}\pmstd{0.003} & 0.007\pmstd{0.006} & 0.010\pmstd{0.002} \\
    ETTh2         & 0.009\pmstd{0.007} & 0.007\pmstd{0.004} & 0.005\pmstd{0.006} & \textbf{0.004}\pmstd{0.003} & \textbf{0.004}\pmstd{0.003} & 0.005\pmstd{0.003} \\
    Stocks        & 0.013\pmstd{0.006} & 0.020\pmstd{0.011} & \textbf{0.012}\pmstd{0.009} & 0.015\pmstd{0.008} & 0.015\pmstd{0.009} & 0.024\pmstd{0.003} \\
    Exchange & 0.011\pmstd{0.007} & 0.008\pmstd{0.007} & 0.005\pmstd{0.003} & \textbf{0.003}\pmstd{0.002} & 0.013\pmstd{0.006} & 0.013\pmstd{0.008} \\
    EEG           & \textbf{0.003}\pmstd{0.002} & 0.008\pmstd{0.003} & \textbf{0.003}\pmstd{0.002} & 0.005\pmstd{0.004} & 0.006\pmstd{0.002} & 0.008\pmstd{0.003} \\
    Energy        & 0.102\pmstd{0.013} & 0.094\pmstd{0.021} & 0.098\pmstd{0.018} & \textbf{0.091}\pmstd{0.015} & 0.109\pmstd{0.010} & 0.099\pmstd{0.019} \\
    MuJoCo        & 0.006\pmstd{0.003} & 0.005\pmstd{0.002} & 0.007\pmstd{0.004} & \textbf{0.004}\pmstd{0.006} & 0.005\pmstd{0.004} & 0.011\pmstd{0.006} \\
    \midrule
    \multicolumn{7}{l}{\textit{\hspace{0.2cm}Correlational score} ($\downarrow$)} \\
    \addlinespace[2pt]
    ETTh1         & 0.038\pmstd{0.005} & \textbf{0.031}\pmstd{0.003} & 0.042\pmstd{0.015} & 0.037\pmstd{0.008} & 0.037\pmstd{0.006} & 0.034\pmstd{0.006} \\
    ETTh2         & \textbf{0.055}\pmstd{0.016} & 0.057\pmstd{0.009} & 0.066\pmstd{0.020} & 0.074\pmstd{0.027} & 0.080\pmstd{0.037} & 0.086\pmstd{0.027} \\
    Stocks        & \textbf{0.004}\pmstd{0.001} & 0.006\pmstd{0.003} & 0.007\pmstd{0.005} & 0.006\pmstd{0.003} & 0.014\pmstd{0.004} & 0.009\pmstd{0.004} \\
    Exchange & 0.060\pmstd{0.022} & 0.073\pmstd{0.027} & \textbf{0.056}\pmstd{0.017} & 0.059\pmstd{0.010} & 0.061\pmstd{0.024} & 0.065\pmstd{0.018} \\
    EEG           & 4.921\pmstd{0.386} & 4.108\pmstd{0.426} & \textbf{3.870}\pmstd{0.422} & 3.991\pmstd{0.817} & 3.950\pmstd{0.256} & 4.018\pmstd{0.298} \\
    Energy        & 0.852\pmstd{0.112} & 0.863\pmstd{0.160} & 0.890\pmstd{0.176} & \textbf{0.822}\pmstd{0.127} & 0.843\pmstd{0.148} & 0.870\pmstd{0.194} \\
    MuJoCo        & 0.250\pmstd{0.025} & 0.248\pmstd{0.038} & \textbf{0.246}\pmstd{0.033} & 0.272\pmstd{0.034} & 0.259\pmstd{0.021} & 0.263\pmstd{0.031} \\
    \midrule
    \multicolumn{7}{l}{\textit{\hspace{0.2cm}Predictive score} ($\downarrow$)} \\
    \addlinespace[2pt]
    ETTh1         & 0.121\pmstd{0.002} & 0.121\pmstd{0.002} & 0.120\pmstd{0.003} & 0.119\pmstd{0.001} & \textbf{0.118}\pmstd{0.004} & 0.121\pmstd{0.005} \\
    ETTh2         & 0.108\pmstd{0.004} & 0.106\pmstd{0.001} & \textbf{0.105}\pmstd{0.004} & 0.107\pmstd{0.006} & 0.110\pmstd{0.005} & 0.107\pmstd{0.005} \\
    Stocks        & 0.037\pmstd{0.000} & 0.037\pmstd{0.000} & 0.037\pmstd{0.000} & 0.037\pmstd{0.000} & 0.037\pmstd{0.000} & 0.037\pmstd{0.000} \\
    Exchange & 0.044\pmstd{0.003} & 0.042\pmstd{0.005} & 0.042\pmstd{0.004} & 0.045\pmstd{0.004} & 0.043\pmstd{0.004} & \textbf{0.040}\pmstd{0.004} \\
    EEG           & 0.000\pmstd{0.000} & 0.000\pmstd{0.000} & 0.000\pmstd{0.000} & 0.000\pmstd{0.000} & 0.000\pmstd{0.000} & 0.000\pmstd{0.000} \\
    Energy        & 0.250\pmstd{0.000} & 0.250\pmstd{0.000} & 0.250\pmstd{0.000} & 0.250\pmstd{0.000} & 0.250\pmstd{0.000} & 0.250\pmstd{0.000} \\
    MuJoCo        & 0.007\pmstd{0.001} & 0.007\pmstd{0.001} & 0.007\pmstd{0.001} & 0.007\pmstd{0.000} & 0.007\pmstd{0.001} & 0.007\pmstd{0.000} \\
    \bottomrule
  \end{tabular}
  \caption{Effect of the sampling time shift $\alpha$ at $T=24$, with $N=100$ Euler steps,
    all from the same trained model. $^\dagger$\,Default setting. Results are mean $\pm$ std; the best value in each row is in bold.}
  \label{tab:ablation-shift}
\end{table}

\begin{table}[H]
  \centering
  \footnotesize
  \setlength{\tabcolsep}{4pt}
  \begin{tabular}{l *{5}{c}}
    \toprule
    & \multicolumn{5}{c}{Euler steps $N$} \\
    \cmidrule(lr){2-6}
    \textbf{Dataset} & $25$ & $50$ & $100^\dagger$ & $200$ & $400$ \\
    \midrule
    \multicolumn{6}{l}{\textit{\hspace{0.2cm}Context-FID} ($\downarrow$)} \\
    \addlinespace[2pt]
    ETTh1         & 0.019\pmstd{0.002} & 0.011\pmstd{0.001} & \textbf{0.005}\pmstd{0.001} & 0.007\pmstd{0.000} & 0.008\pmstd{0.001} \\
    ETTh2         & 0.009\pmstd{0.001} & 0.005\pmstd{0.001} & \textbf{0.004}\pmstd{0.000} & 0.005\pmstd{0.000} & 0.009\pmstd{0.000} \\
    Stocks        & 0.021\pmstd{0.003} & 0.004\pmstd{0.001} & 0.005\pmstd{0.002} & \textbf{0.003}\pmstd{0.000} & 0.010\pmstd{0.001} \\
    Exchange & 0.009\pmstd{0.000} & \textbf{0.003}\pmstd{0.000} & 0.004\pmstd{0.001} & 0.005\pmstd{0.000} & 0.006\pmstd{0.000} \\
    EEG           & 0.016\pmstd{0.002} & 0.015\pmstd{0.002} & \textbf{0.013}\pmstd{0.002} & 0.015\pmstd{0.002} & 0.016\pmstd{0.003} \\
    Energy        & 0.013\pmstd{0.001} & 0.012\pmstd{0.000} & \textbf{0.010}\pmstd{0.003} & 0.012\pmstd{0.000} & 0.012\pmstd{0.001} \\
    MuJoCo        & 0.013\pmstd{0.002} & 0.007\pmstd{0.001} & \textbf{0.006}\pmstd{0.000} & 0.007\pmstd{0.001} & 0.007\pmstd{0.001} \\
    \midrule
    \multicolumn{6}{l}{\textit{\hspace{0.2cm}Discriminative score} ($\downarrow$)} \\
    \addlinespace[2pt]
    ETTh1         & 0.007\pmstd{0.004} & 0.006\pmstd{0.003} & \textbf{0.005}\pmstd{0.004} & 0.006\pmstd{0.004} & \textbf{0.005}\pmstd{0.002} \\
    ETTh2         & 0.006\pmstd{0.003} & 0.007\pmstd{0.003} & \textbf{0.005}\pmstd{0.006} & 0.007\pmstd{0.003} & 0.008\pmstd{0.004} \\
    Stocks        & 0.009\pmstd{0.004} & \textbf{0.006}\pmstd{0.003} & 0.012\pmstd{0.009} & 0.011\pmstd{0.009} & 0.020\pmstd{0.008} \\
    Exchange & 0.013\pmstd{0.009} & 0.006\pmstd{0.004} & \textbf{0.005}\pmstd{0.003} & \textbf{0.005}\pmstd{0.003} & 0.006\pmstd{0.005} \\
    EEG           & 0.010\pmstd{0.007} & 0.009\pmstd{0.003} & \textbf{0.003}\pmstd{0.002} & \textbf{0.003}\pmstd{0.002} & 0.005\pmstd{0.003} \\
    Energy        & 0.110\pmstd{0.023} & 0.102\pmstd{0.009} & 0.098\pmstd{0.018} & \textbf{0.094}\pmstd{0.013} & 0.104\pmstd{0.018} \\
    MuJoCo        & 0.016\pmstd{0.003} & \textbf{0.005}\pmstd{0.003} & 0.007\pmstd{0.004} & 0.008\pmstd{0.005} & 0.008\pmstd{0.003} \\
    \midrule
    \multicolumn{6}{l}{\textit{\hspace{0.2cm}Correlational score} ($\downarrow$)} \\
    \addlinespace[2pt]
    ETTh1         & 0.048\pmstd{0.012} & 0.051\pmstd{0.005} & \textbf{0.042}\pmstd{0.015} & 0.046\pmstd{0.013} & 0.058\pmstd{0.018} \\
    ETTh2         & \textbf{0.059}\pmstd{0.007} & \textbf{0.059}\pmstd{0.016} & 0.066\pmstd{0.020} & 0.061\pmstd{0.007} & 0.071\pmstd{0.013} \\
    Stocks        & 0.011\pmstd{0.005} & 0.007\pmstd{0.003} & 0.007\pmstd{0.005} & \textbf{0.005}\pmstd{0.003} & \textbf{0.005}\pmstd{0.005} \\
    Exchange & 0.070\pmstd{0.021} & 0.077\pmstd{0.030} & \textbf{0.056}\pmstd{0.017} & 0.073\pmstd{0.032} & 0.072\pmstd{0.020} \\
    EEG           & 4.535\pmstd{0.204} & 4.109\pmstd{0.385} & \textbf{3.870}\pmstd{0.422} & 4.018\pmstd{0.239} & 3.911\pmstd{0.257} \\
    Energy        & 0.902\pmstd{0.101} & 0.877\pmstd{0.160} & 0.890\pmstd{0.176} & \textbf{0.823}\pmstd{0.158} & 0.842\pmstd{0.136} \\
    MuJoCo        & 0.267\pmstd{0.028} & 0.266\pmstd{0.020} & \textbf{0.246}\pmstd{0.033} & 0.249\pmstd{0.017} & 0.261\pmstd{0.028} \\
    \midrule
    \multicolumn{6}{l}{\textit{\hspace{0.2cm}Predictive score} ($\downarrow$)} \\
    \addlinespace[2pt]
    ETTh1         & 0.119\pmstd{0.003} & \textbf{0.118}\pmstd{0.002} & 0.120\pmstd{0.003} & 0.122\pmstd{0.003} & 0.119\pmstd{0.003} \\
    ETTh2         & \textbf{0.105}\pmstd{0.002} & 0.107\pmstd{0.004} & \textbf{0.105}\pmstd{0.004} & 0.106\pmstd{0.005} & 0.107\pmstd{0.005} \\
    Stocks        & 0.037\pmstd{0.000} & 0.037\pmstd{0.000} & 0.037\pmstd{0.000} & 0.037\pmstd{0.000} & 0.037\pmstd{0.000} \\
    Exchange & 0.044\pmstd{0.005} & 0.040\pmstd{0.005} & 0.042\pmstd{0.004} & \textbf{0.039}\pmstd{0.003} & 0.041\pmstd{0.003} \\
    EEG           & 0.000\pmstd{0.000} & 0.000\pmstd{0.000} & 0.000\pmstd{0.000} & 0.000\pmstd{0.000} & 0.000\pmstd{0.000} \\
    Energy        & 0.250\pmstd{0.000} & 0.250\pmstd{0.000} & 0.250\pmstd{0.000} & 0.250\pmstd{0.000} & 0.250\pmstd{0.000} \\
    MuJoCo        & \textbf{0.007}\pmstd{0.000} & 0.008\pmstd{0.002} & \textbf{0.007}\pmstd{0.001} & 0.010\pmstd{0.003} & \textbf{0.007}\pmstd{0.001} \\
    \bottomrule
  \end{tabular}
  \caption{Effect of the number of Euler steps $N$ at $T=24$, with time shift $\alpha=1$,
    all from the same trained model. $^\dagger$\,Default setting. Results are mean $\pm$ std; the best value in each row is in bold.}
  \label{tab:ablation-steps}
\end{table}

Halving the budget from $100$ to $50$ steps leaves the discriminative and predictive scores within one standard deviation of their $N=100$ values on six of seven datasets, the exception being EEG, where the discriminative score triples. The average Context-FID rises from $0.007$ to $0.008$.

\clearpage
\subsection{Choice of wavelet transform.} \label{app:wavelet_choice}
We compare four fixed wavelet families: Daubechies (\texttt{db2}), Coiflets (\texttt{coif1}), biorthogonal (\texttt{bior2.2}) and reverse biorthogonal (\texttt{rbio2.2}). As a fifth setting, following \Cref{sec:method:learnable}, we let the transform be learned jointly with the velocity field, initializing the free angles at their Daubechies values and any departure from it is driven by the flow-matching loss alone.

\begin{table}[H]
  \centering
  \footnotesize
  \setlength{\tabcolsep}{4pt}
  \begin{tabular}{l *{7}{c}}
    \toprule
    \textbf{Transform} & ETTh1 & ETTh2 & Stocks
    & \begin{tabular}[b]{@{}c@{}}Exchange\end{tabular}
    & EEG & Energy & MuJoCo \\
    \midrule
    \multicolumn{8}{l}{\textit{\hspace{0.2cm}Context-FID} ($\downarrow$)} \\
    \addlinespace[2pt]
    \texttt{db2}     & 0.005\pmstd{0.001} & 0.004\pmstd{0.000} & \textbf{0.005}\pmstd{0.002} & \textbf{0.004}\pmstd{0.001} & 0.013\pmstd{0.002} & 0.010\pmstd{0.003} & 0.006\pmstd{0.000} \\
    \texttt{coif1}   & 0.005\pmstd{0.000} & \textbf{0.003}\pmstd{0.000} & 0.009\pmstd{0.001} & 0.009\pmstd{0.001} & 0.014\pmstd{0.001} & \textbf{0.008}\pmstd{0.000} & 0.007\pmstd{0.001} \\
    \texttt{bior2.2} & 0.006\pmstd{0.000} & 0.005\pmstd{0.000} & 0.006\pmstd{0.000} & \textbf{0.004}\pmstd{0.001} & 0.014\pmstd{0.001} & 0.013\pmstd{0.001} & 0.006\pmstd{0.001} \\
    \texttt{rbio2.2} & \textbf{0.004}\pmstd{0.000} & 0.005\pmstd{0.000} & 0.019\pmstd{0.001} & 0.006\pmstd{0.001} & \textbf{0.010}\pmstd{0.002} & 0.013\pmstd{0.001} & 0.006\pmstd{0.000} \\
    Learned          & \textbf{0.004}\pmstd{0.000} & \textbf{0.003}\pmstd{0.000} & 0.008\pmstd{0.001} & 0.008\pmstd{0.000} & 0.011\pmstd{0.002} & 0.009\pmstd{0.001} & \textbf{0.005}\pmstd{0.000} \\
    \midrule
    \multicolumn{8}{l}{\textit{\hspace{0.2cm}Discriminative score} ($\downarrow$)} \\
    \addlinespace[2pt]
    \texttt{db2}     & 0.005\pmstd{0.004} & 0.005\pmstd{0.006} & 0.012\pmstd{0.009} & \textbf{0.005}\pmstd{0.003} & 0.003\pmstd{0.002} & \textbf{0.098}\pmstd{0.018} & 0.007\pmstd{0.004} \\
    \texttt{coif1}   & 0.005\pmstd{0.003} & 0.007\pmstd{0.004} & 0.012\pmstd{0.011} & 0.012\pmstd{0.003} & \textbf{0.002}\pmstd{0.001} & 0.103\pmstd{0.011} & \textbf{0.005}\pmstd{0.006} \\
    \texttt{bior2.2} & 0.006\pmstd{0.005} & \textbf{0.003}\pmstd{0.003} & 0.015\pmstd{0.006} & 0.007\pmstd{0.006} & 0.006\pmstd{0.005} & 0.106\pmstd{0.009} & 0.006\pmstd{0.008} \\
    \texttt{rbio2.2} & \textbf{0.003}\pmstd{0.004} & 0.005\pmstd{0.001} & 0.014\pmstd{0.009} & 0.008\pmstd{0.003} & 0.004\pmstd{0.002} & 0.111\pmstd{0.009} & 0.008\pmstd{0.007} \\
    Learned          & 0.004\pmstd{0.003} & 0.005\pmstd{0.003} & \textbf{0.010}\pmstd{0.005} & 0.007\pmstd{0.002} & 0.004\pmstd{0.003} & 0.110\pmstd{0.014} & 0.006\pmstd{0.003} \\
    \midrule
    \multicolumn{8}{l}{\textit{\hspace{0.2cm}Correlational score} ($\downarrow$)} \\
    \addlinespace[2pt]
    \texttt{db2}     & 0.042\pmstd{0.015} & 0.066\pmstd{0.020} & 0.007\pmstd{0.005} & 0.056\pmstd{0.017} & \textbf{3.870}\pmstd{0.422} & \textbf{0.890}\pmstd{0.176} & 0.246\pmstd{0.033} \\
    \texttt{coif1}   & 0.045\pmstd{0.014} & 0.088\pmstd{0.020} & 0.005\pmstd{0.002} & \textbf{0.045}\pmstd{0.020} & 4.777\pmstd{0.504} & 0.901\pmstd{0.139} & 0.267\pmstd{0.025} \\
    \texttt{bior2.2} & 0.041\pmstd{0.008} & \textbf{0.056}\pmstd{0.017} & 0.006\pmstd{0.004} & 0.056\pmstd{0.019} & 4.754\pmstd{0.331} & 1.033\pmstd{0.234} & \textbf{0.228}\pmstd{0.046} \\
    \texttt{rbio2.2} & \textbf{0.039}\pmstd{0.009} & 0.066\pmstd{0.012} & 0.010\pmstd{0.005} & 0.067\pmstd{0.017} & 4.447\pmstd{0.288} & 1.242\pmstd{0.210} & 0.308\pmstd{0.035} \\
    Learned          & 0.041\pmstd{0.013} & 0.068\pmstd{0.018} & \textbf{0.004}\pmstd{0.002} & 0.056\pmstd{0.013} & 4.401\pmstd{0.343} & 1.018\pmstd{0.370} & 0.269\pmstd{0.017} \\
    \midrule
    \multicolumn{8}{l}{\textit{\hspace{0.2cm}Predictive score} ($\downarrow$)} \\
    \addlinespace[2pt]
    \texttt{db2}     & 0.120\pmstd{0.003} & \textbf{0.105}\pmstd{0.004} & 0.037\pmstd{0.000} & \textbf{0.042}\pmstd{0.004} & 0.000\pmstd{0.000} & 0.250\pmstd{0.000} & 0.007\pmstd{0.001} \\
    \texttt{coif1}   & \textbf{0.114}\pmstd{0.005} & \textbf{0.105}\pmstd{0.004} & 0.037\pmstd{0.000} & 0.044\pmstd{0.005} & 0.000\pmstd{0.000} & 0.250\pmstd{0.000} & 0.007\pmstd{0.000} \\
    \texttt{bior2.2} & 0.122\pmstd{0.002} & 0.106\pmstd{0.002} & 0.037\pmstd{0.000} & 0.043\pmstd{0.007} & 0.000\pmstd{0.000} & 0.250\pmstd{0.000} & 0.007\pmstd{0.000} \\
    \texttt{rbio2.2} & 0.119\pmstd{0.001} & \textbf{0.105}\pmstd{0.002} & 0.037\pmstd{0.000} & 0.044\pmstd{0.007} & 0.000\pmstd{0.000} & 0.250\pmstd{0.000} & 0.007\pmstd{0.001} \\
    Learned          & 0.121\pmstd{0.001} & \textbf{0.105}\pmstd{0.004} & 0.037\pmstd{0.000} & \textbf{0.042}\pmstd{0.004} & 0.000\pmstd{0.000} & 0.250\pmstd{0.000} & 0.007\pmstd{0.001} \\
    \bottomrule
  \end{tabular}
  \caption{Effect of the wavelet transform at $T=24$. Results are mean $\pm$ std; the best transform for each dataset and
    metric is in bold.}
  \label{tab:ablation-transform}
\end{table}

No transform wins across the board: each of the five is best somewhere, and the winner changes from one metric to another within a dataset and from one dataset to another within a metric (\Cref{tab:ablation-transform}). The learned transform gives the model the capacity to move the basis toward a representation better suited to the data, but training makes little use of it. The angles never leave a $2^\circ$ neighborhood of the \texttt{db2} initialization, the direction of the drift is not consistent across datasets, and the final filters differ from \texttt{db2} by at most $5\%$ in $\ell^2$ (\Cref{fig:angle-drift}). We conclude that the method is largely insensitive to the choice of filter: what matters is the multi-resolution representation, not the particular basis within it.

\begin{figure}[H]
    \centering
    \begin{tikzpicture}
        \node[inner sep=0pt, anchor=south west] (im) at (0,0)
            {\includegraphics[width=\linewidth]{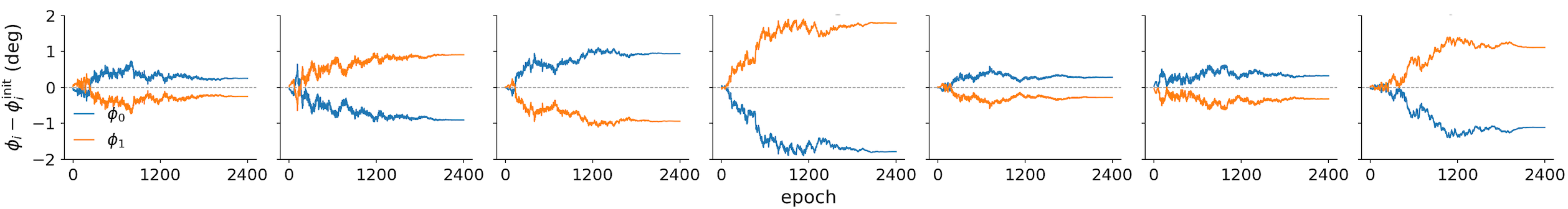}};
        \begin{scope}[x={(im.south east)}, y={(im.north west)}]
            \foreach \xc/\nm in {0.1017/{ETTh1}, 0.2394/{ETTh2}, 0.3770/{Stocks},
                                 0.5147/{Exchange}, 0.6523/{EEG},
                                 0.7900/{Energy}, 0.9276/{MuJoCo}}
                \node[anchor=south, font=\scriptsize, inner sep=1pt, yshift=3pt]
                    at (\xc,0.90) {\nm};
        \end{scope}
    \end{tikzpicture}
    \caption{Deviation of the lattice angles from their \texttt{db2} initialization during training at $T=24$.}
    \label{fig:angle-drift}
\end{figure}

\end{document}

%% file: math_commands.tex
\usepackage{amsmath,amsfonts,bm}

\def\eqref#1{equation~\ref{#1}}

\def\1{\bm{1}}

\DeclareMathAlphabet{\mathsfit}{\encodingdefault}{\sfdefault}{m}{sl}
\SetMathAlphabet{\mathsfit}{bold}{\encodingdefault}{\sfdefault}{bx}{n}

\newcommand{\pdata}{p_{\rm{data}}}

\newcommand{\E}{\mathbb{E}}

\newcommand{\R}{\mathbb{R}}



%% file: figures/SNR_figure.tex
\ifdefined\snrwidth\else\def\snrwidth{8.5cm}\fi
\ifdefined\snrheight\else\def\snrheight{5.5cm}\fi
\ifdefined\snrylabel\else\def\snrylabel{$\mathrm{SNR}_j(t)$}\fi
\ifdefined\snrlegend\else\def\snrlegend{1}\fi
\ifdefined\snrlegendsep\else\def\snrlegendsep{0.35cm}\fi
\ifdefined\snrlegendsample\else\def\snrlegendsample{0.6cm}\fi
\def\snryes{1}\def\snrlegendentry#1{\ifx\snrlegend\snryes\addlegendentry{#1}\fi}
\def\snrall{all}\def\snrno{0}\def\snrfound{0}\ifdefined\snrdataset\else\def\snrdataset{all}\fi

\def\snrplot#1#2#3#4#5{\begin{tikzpicture}
\begin{axis}[
    width=\snrwidth,
    height=\snrheight,
    ymode=log,
    ymin=1e-6,
    ymax=1e6,
    xmin=0,
    xmax=1,
    domain=1e-5:0.99999,
    samples=500,
    set layers,
    axis line style={on layer=axis foreground},
    xlabel={$t$},
    ylabel={\snrylabel},
    label style={font=\small},
    tick label style={font=\footnotesize},
    grid=both,
    grid style={gray!20,dashed},
    ytick={1e-4,1,1e4},
    yticklabels={$10^{-4}$,$1$,$10^{4}$},
    minor ytick={1e-2,1e2},
    xtick={0,0.25,0.5,0.75,1},
    legend style={
        at={(0.5,1.03)},
        anchor=south,
        draw=none,
        fill=none,
        font=\footnotesize,
        /tikz/every even column/.append style={column sep=\snrlegendsep}
    },
    legend image code/.code={\draw[##1] (0cm,0cm) -- (\snrlegendsample,0cm);},
    legend columns=4,
    legend cell align=left,
]
\addplot[cA,    line width=0.9pt] {(#2*x/(1-x))^2};
\snrlegendentry{$a_3$}
\addplot[cC,    line width=0.8pt] {(#3*x/(1-x))^2};
\snrlegendentry{$d_3$}
\addplot[cC!60, line width=0.8pt] {(#4*x/(1-x))^2};
\snrlegendentry{$d_2$}
\addplot[cC!35, line width=0.8pt] {(#5*x/(1-x))^2};
\snrlegendentry{$d_1$}
\addplot[black, dashed, line width=0.4pt, domain=0:1, samples=2, forget plot] {1};
\end{axis}
\end{tikzpicture}}\relax

\def\snrcase#1#2#3#4#5{\def\snrthis{#1}\ifx\snrdataset\snrall
  \def\snrfound{1}\snrplot{#1}{#2}{#3}{#4}{#5}\else\ifx\snrdataset\snrthis
  \def\snrfound{1}\snrplot{#1}{#2}{#3}{#4}{#5}\fi\fi\ignorespaces}\relax

\snrcase {ETTh1}         {2.57} {0.90} {0.42} {0.24}
\snrcase {ETTh2}         {2.76} {0.40} {0.24} {0.17}
\snrcase {Stocks}        {2.75} {0.30} {0.23} {0.17}
\snrcase {Exchange Rate} {2.82} {0.12} {0.07} {0.04}
\snrcase {EEG}           {1.41} {0.95} {0.93} {0.92}
\snrcase {Energy}        {2.66} {0.51} {0.37} {0.31}
\snrcase {MuJoCo}        {2.67} {0.66} {0.38} {0.21}

\ifx\snrfound\snrno\errmessage{SNR_figure: unknown dataset '\snrdataset'. Use ETTh1, ETTh2, Stocks, Exchange Rate, EEG, Energy or MuJoCo}\fi

%% file: figures/figure1_new.tex
\begin{tikzpicture}[
  panel/.style   = {draw=cN!55, line width=0.4pt, fill=white},
  sigline/.style = {line width=0.45pt, line join=round},
  coefbox/.style = {rectangle, rounded corners=1.5pt, draw=cN!55, fill=white,
                    line width=0.4pt, minimum width=1.5cm, minimum height=1.6cm,
                    inner sep=0pt},
  flow/.style    = {-{Stealth[length=2.2mm]}, line width=0.5pt, draw=cN!85},
  vert/.style    = {-{Stealth[length=1.8mm]}, line width=0.45pt, draw=cN!75},
  vertd/.style   = {-{Stealth[length=1.8mm]}, line width=0.4pt, draw=cN!50,
                    dash pattern=on 2pt off 1.6pt},
  planeb/.style  = {rounded corners=1.5pt, draw=cN!28, fill=white, line width=0.35pt},
  planem/.style  = {rounded corners=1.5pt, draw=cN!42, fill=white, line width=0.38pt},
  tag/.style     = {font=\scriptsize, text=cN, inner sep=1.5pt},
  ax/.style      = {font=\scriptsize, text=cN, inner sep=2pt}]

\def\sigAa{(0.000,0.664) (0.037,0.709) (0.073,0.630) (0.110,0.785) (0.146,0.584) (0.183,0.539) (0.219,0.670) (0.256,0.485) (0.292,0.382) (0.329,0.803) (0.365,0.388) (0.402,0.253) (0.438,0.591) (0.475,0.474) (0.511,0.743) (0.548,0.311) (0.584,0.543) (0.621,0.403) (0.657,0.725) (0.694,0.338) (0.730,0.809) (0.767,0.728) (0.803,0.720) (0.840,0.361) (0.876,0.459) (0.913,0.333) (0.949,0.389) (0.986,0.381) (1.022,0.559) (1.059,0.753) (1.095,0.838) (1.132,0.752) (1.168,0.511) (1.205,0.701) (1.241,0.425) (1.278,0.624) (1.314,0.672) (1.351,0.620) (1.387,0.333) (1.424,0.659) (1.460,0.570) (1.497,0.597) (1.533,0.585) (1.570,0.291) (1.606,0.451) (1.643,0.603) (1.679,0.137) (1.716,0.705) (1.752,0.032) (1.789,0.974) (1.825,0.411) (1.862,0.233) (1.898,0.698) (1.935,0.544) (1.971,0.438) (2.008,0.340) (2.044,0.753) (2.081,0.812) (2.117,0.834) (2.154,0.665) (2.190,0.509) (2.227,0.851) (2.263,0.380) (2.300,0.650)}
\def\sigAb{(0.000,0.450) (0.037,0.402) (0.073,0.321) (0.110,0.000) (0.146,0.267) (0.183,0.367) (0.219,0.476) (0.256,0.432) (0.292,0.583) (0.329,0.306) (0.365,0.626) (0.402,0.572) (0.438,0.653) (0.475,0.334) (0.511,0.308) (0.548,0.696) (0.584,0.856) (0.621,0.732) (0.657,0.539) (0.694,0.817) (0.730,0.346) (0.767,0.477) (0.803,0.115) (0.840,0.534) (0.876,0.511) (0.913,0.590) (0.949,0.508) (0.986,0.808) (1.022,0.364) (1.059,0.503) (1.095,0.484) (1.132,0.533) (1.168,0.361) (1.205,0.508) (1.241,0.493) (1.278,0.445) (1.314,0.615) (1.351,0.593) (1.387,0.687) (1.424,0.399) (1.460,0.218) (1.497,0.581) (1.533,0.568) (1.570,0.688) (1.606,0.691) (1.643,0.269) (1.679,0.413) (1.716,0.678) (1.752,0.471) (1.789,0.600) (1.825,0.327) (1.862,0.624) (1.898,0.581) (1.935,0.567) (1.971,0.634) (2.008,0.664) (2.044,0.625) (2.081,0.520) (2.117,0.381) (2.154,0.521) (2.190,0.485) (2.227,0.552) (2.263,0.758) (2.300,0.247)}
\def\sigAc{(0.000,0.492) (0.037,0.824) (0.073,0.469) (0.110,0.417) (0.146,0.508) (0.183,0.096) (0.219,0.295) (0.256,0.553) (0.292,0.691) (0.329,0.479) (0.365,0.826) (0.402,0.715) (0.438,0.465) (0.475,0.432) (0.511,0.635) (0.548,0.434) (0.584,0.949) (0.621,0.530) (0.657,0.246) (0.694,0.593) (0.730,0.549) (0.767,0.615) (0.803,0.137) (0.840,0.638) (0.876,0.538) (0.913,0.696) (0.949,0.321) (0.986,0.687) (1.022,0.231) (1.059,0.645) (1.095,0.536) (1.132,0.406) (1.168,0.367) (1.205,0.698) (1.241,0.381) (1.278,0.547) (1.314,0.913) (1.351,0.425) (1.387,0.436) (1.424,0.349) (1.460,0.441) (1.497,0.667) (1.533,0.476) (1.570,0.366) (1.606,0.529) (1.643,0.627) (1.679,0.282) (1.716,0.834) (1.752,0.749) (1.789,0.308) (1.825,0.256) (1.862,0.471) (1.898,0.437) (1.935,0.451) (1.971,0.658) (2.008,0.495) (2.044,0.364) (2.081,0.636) (2.117,0.810) (2.154,0.579) (2.190,0.593) (2.227,0.545) (2.263,0.773) (2.300,0.349)}
\def\sigBa{(0.000,0.568) (0.037,0.673) (0.073,0.627) (0.110,0.755) (0.146,0.590) (0.183,0.591) (0.219,0.737) (0.256,0.557) (0.292,0.415) (0.329,0.748) (0.365,0.385) (0.402,0.264) (0.438,0.517) (0.475,0.350) (0.511,0.551) (0.548,0.185) (0.584,0.410) (0.621,0.278) (0.657,0.537) (0.694,0.208) (0.730,0.678) (0.767,0.669) (0.803,0.688) (0.840,0.375) (0.876,0.484) (0.913,0.431) (0.949,0.530) (0.986,0.524) (1.022,0.647) (1.059,0.798) (1.095,0.885) (1.132,0.815) (1.168,0.564) (1.205,0.663) (1.241,0.378) (1.278,0.552) (1.314,0.600) (1.351,0.527) (1.387,0.233) (1.424,0.509) (1.460,0.472) (1.497,0.546) (1.533,0.553) (1.570,0.294) (1.606,0.459) (1.643,0.659) (1.679,0.314) (1.716,0.828) (1.752,0.224) (1.789,1.050) (1.825,0.589) (1.862,0.455) (1.898,0.834) (1.935,0.640) (1.971,0.507) (2.008,0.419) (2.044,0.777) (2.081,0.788) (2.117,0.748) (2.154,0.573) (2.190,0.458) (2.227,0.789) (2.263,0.378) (2.300,0.599)}
\def\sigBb{(0.000,0.631) (0.037,0.588) (0.073,0.509) (0.110,0.211) (0.146,0.426) (0.183,0.490) (0.219,0.564) (0.256,0.509) (0.292,0.614) (0.329,0.345) (0.365,0.605) (0.402,0.539) (0.438,0.593) (0.475,0.295) (0.511,0.263) (0.548,0.602) (0.584,0.734) (0.621,0.627) (0.657,0.458) (0.694,0.707) (0.730,0.303) (0.767,0.436) (0.803,0.135) (0.840,0.527) (0.876,0.532) (0.913,0.623) (0.949,0.559) (0.986,0.838) (1.022,0.464) (1.059,0.598) (1.095,0.581) (1.132,0.623) (1.168,0.465) (1.205,0.581) (1.241,0.563) (1.278,0.500) (1.314,0.630) (1.351,0.594) (1.387,0.649) (1.424,0.373) (1.460,0.192) (1.497,0.487) (1.533,0.448) (1.570,0.536) (1.606,0.533) (1.643,0.146) (1.679,0.271) (1.716,0.504) (1.752,0.324) (1.789,0.457) (1.825,0.237) (1.862,0.508) (1.898,0.491) (1.935,0.504) (1.971,0.585) (2.008,0.640) (2.044,0.629) (2.081,0.562) (2.117,0.467) (2.154,0.601) (2.190,0.583) (2.227,0.646) (2.263,0.830) (2.300,0.374)}
\def\sigBc{(0.000,0.429) (0.037,0.680) (0.073,0.338) (0.110,0.275) (0.146,0.354) (0.183,0.012) (0.219,0.217) (0.256,0.482) (0.292,0.642) (0.329,0.488) (0.365,0.810) (0.402,0.713) (0.438,0.477) (0.475,0.417) (0.511,0.556) (0.548,0.342) (0.584,0.761) (0.621,0.379) (0.657,0.131) (0.694,0.453) (0.730,0.449) (0.767,0.548) (0.803,0.172) (0.840,0.645) (0.876,0.578) (0.913,0.721) (0.949,0.381) (0.986,0.675) (1.022,0.244) (1.059,0.574) (1.095,0.455) (1.132,0.332) (1.168,0.306) (1.205,0.621) (1.241,0.384) (1.278,0.576) (1.314,0.942) (1.351,0.554) (1.387,0.587) (1.424,0.516) (1.460,0.583) (1.497,0.753) (1.533,0.551) (1.570,0.417) (1.606,0.530) (1.643,0.600) (1.679,0.301) (1.716,0.802) (1.752,0.760) (1.789,0.415) (1.825,0.409) (1.862,0.629) (1.898,0.617) (1.935,0.630) (1.971,0.794) (2.008,0.620) (2.044,0.466) (2.081,0.665) (2.117,0.786) (2.154,0.566) (2.190,0.579) (2.227,0.554) (2.263,0.785) (2.300,0.454)}
\def\sigCa{(0.000,0.407) (0.037,0.582) (0.073,0.598) (0.110,0.654) (0.146,0.583) (0.183,0.657) (0.219,0.792) (0.256,0.662) (0.292,0.492) (0.329,0.610) (0.365,0.414) (0.402,0.343) (0.438,0.402) (0.475,0.196) (0.511,0.242) (0.548,0.068) (0.584,0.228) (0.621,0.141) (0.657,0.238) (0.694,0.079) (0.730,0.435) (0.767,0.544) (0.803,0.600) (0.840,0.433) (0.876,0.533) (0.913,0.608) (0.949,0.751) (0.986,0.750) (1.022,0.756) (1.059,0.804) (1.095,0.874) (1.132,0.846) (1.168,0.638) (1.205,0.570) (1.241,0.340) (1.278,0.433) (1.314,0.468) (1.351,0.380) (1.387,0.144) (1.424,0.277) (1.460,0.331) (1.497,0.461) (1.533,0.497) (1.570,0.353) (1.606,0.489) (1.643,0.714) (1.679,0.641) (1.716,0.949) (1.752,0.598) (1.789,1.045) (1.825,0.853) (1.862,0.819) (1.898,0.974) (1.935,0.763) (1.971,0.620) (2.008,0.567) (2.044,0.755) (2.081,0.688) (2.117,0.560) (2.154,0.417) (2.190,0.394) (2.227,0.628) (2.263,0.409) (2.300,0.502)}
\def\sigCb{(0.000,0.890) (0.037,0.866) (0.073,0.807) (0.110,0.617) (0.146,0.700) (0.183,0.690) (0.219,0.692) (0.256,0.634) (0.292,0.643) (0.329,0.449) (0.365,0.553) (0.402,0.484) (0.438,0.483) (0.475,0.289) (0.511,0.254) (0.548,0.435) (0.584,0.492) (0.621,0.439) (0.657,0.346) (0.694,0.492) (0.730,0.287) (0.767,0.393) (0.803,0.259) (0.840,0.515) (0.876,0.563) (0.913,0.652) (0.949,0.631) (0.986,0.810) (1.022,0.637) (1.059,0.731) (1.095,0.719) (1.132,0.741) (1.168,0.642) (1.205,0.684) (1.241,0.663) (1.278,0.593) (1.314,0.629) (1.351,0.579) (1.387,0.562) (1.424,0.368) (1.460,0.229) (1.497,0.348) (1.533,0.278) (1.570,0.295) (1.606,0.282) (1.643,0.041) (1.679,0.109) (1.716,0.235) (1.752,0.141) (1.789,0.249) (1.825,0.163) (1.862,0.328) (1.898,0.357) (1.935,0.410) (1.971,0.495) (2.008,0.574) (2.044,0.609) (2.081,0.619) (2.117,0.616) (2.154,0.710) (2.190,0.724) (2.227,0.765) (2.263,0.870) (2.300,0.608)}
\def\sigCc{(0.000,0.351) (0.037,0.416) (0.073,0.175) (0.110,0.110) (0.146,0.153) (0.183,0.000) (0.219,0.167) (0.256,0.381) (0.292,0.538) (0.329,0.512) (0.365,0.717) (0.402,0.665) (0.438,0.507) (0.475,0.419) (0.511,0.424) (0.548,0.240) (0.584,0.411) (0.621,0.175) (0.657,0.044) (0.694,0.252) (0.730,0.309) (0.767,0.437) (0.803,0.309) (0.840,0.627) (0.876,0.629) (0.913,0.714) (0.949,0.509) (0.986,0.621) (1.022,0.330) (1.059,0.450) (1.095,0.344) (1.132,0.261) (1.168,0.262) (1.205,0.478) (1.241,0.421) (1.278,0.609) (1.314,0.890) (1.351,0.749) (1.387,0.809) (1.424,0.780) (1.460,0.793) (1.497,0.835) (1.533,0.661) (1.570,0.523) (1.606,0.532) (1.643,0.542) (1.679,0.383) (1.716,0.686) (1.752,0.722) (1.789,0.608) (1.825,0.677) (1.862,0.853) (1.898,0.879) (1.935,0.886) (1.971,0.943) (2.008,0.794) (2.044,0.641) (2.081,0.678) (2.117,0.686) (2.154,0.537) (2.190,0.544) (2.227,0.562) (2.263,0.743) (2.300,0.635)}
\def\sigDa{(0.000,0.302) (0.037,0.488) (0.073,0.553) (0.110,0.533) (0.146,0.561) (0.183,0.678) (0.219,0.767) (0.256,0.715) (0.292,0.567) (0.329,0.468) (0.365,0.468) (0.402,0.460) (0.438,0.338) (0.475,0.154) (0.511,0.055) (0.548,0.091) (0.584,0.153) (0.621,0.128) (0.657,0.061) (0.694,0.090) (0.730,0.253) (0.767,0.432) (0.803,0.504) (0.840,0.507) (0.876,0.572) (0.913,0.735) (0.949,0.877) (0.986,0.881) (1.022,0.786) (1.059,0.734) (1.095,0.770) (1.132,0.785) (1.168,0.671) (1.205,0.479) (1.241,0.357) (1.278,0.357) (1.314,0.372) (1.351,0.294) (1.387,0.172) (1.424,0.147) (1.460,0.264) (1.497,0.406) (1.533,0.456) (1.570,0.451) (1.606,0.524) (1.643,0.710) (1.679,0.888) (1.716,0.938) (1.752,0.896) (1.789,0.900) (1.825,0.989) (1.862,1.050) (1.898,0.972) (1.935,0.804) (1.971,0.690) (2.008,0.682) (2.044,0.673) (2.081,0.559) (2.117,0.391) (2.154,0.313) (2.190,0.374) (2.227,0.463) (2.263,0.467) (2.300,0.425)}
\def\sigDb{(0.000,1.012) (0.037,1.010) (0.073,0.985) (0.110,0.938) (0.146,0.886) (0.183,0.816) (0.219,0.756) (0.256,0.711) (0.292,0.635) (0.329,0.558) (0.365,0.501) (0.402,0.449) (0.438,0.401) (0.475,0.347) (0.511,0.320) (0.548,0.317) (0.584,0.296) (0.621,0.302) (0.657,0.299) (0.694,0.318) (0.730,0.338) (0.767,0.392) (0.803,0.436) (0.840,0.508) (0.876,0.579) (0.913,0.642) (0.949,0.664) (0.986,0.709) (1.022,0.753) (1.059,0.788) (1.095,0.784) (1.132,0.784) (1.168,0.761) (1.205,0.728) (1.241,0.712) (1.278,0.654) (1.314,0.600) (1.351,0.552) (1.387,0.477) (1.424,0.407) (1.460,0.341) (1.497,0.278) (1.533,0.200) (1.570,0.152) (1.606,0.135) (1.643,0.083) (1.679,0.083) (1.716,0.085) (1.752,0.088) (1.789,0.147) (1.825,0.198) (1.862,0.229) (1.898,0.288) (1.935,0.361) (1.971,0.426) (2.008,0.505) (2.044,0.570) (2.081,0.642) (2.117,0.719) (2.154,0.753) (2.190,0.789) (2.227,0.802) (2.263,0.811) (2.300,0.785)}
\def\sigDc{(0.000,0.332) (0.037,0.221) (0.073,0.132) (0.110,0.082) (0.146,0.082) (0.183,0.132) (0.219,0.222) (0.256,0.334) (0.292,0.445) (0.329,0.535) (0.365,0.585) (0.402,0.586) (0.438,0.538) (0.475,0.450) (0.511,0.340) (0.548,0.230) (0.584,0.143) (0.621,0.096) (0.657,0.100) (0.694,0.154) (0.730,0.248) (0.767,0.366) (0.803,0.485) (0.840,0.584) (0.876,0.644) (0.913,0.657) (0.949,0.622) (0.986,0.549) (1.022,0.457) (1.059,0.366) (1.095,0.298) (1.132,0.272) (1.168,0.295) (1.205,0.369) (1.241,0.481) (1.278,0.615) (1.314,0.747) (1.351,0.855) (1.387,0.921) (1.424,0.936) (1.460,0.899) (1.497,0.820) (1.533,0.718) (1.570,0.614) (1.606,0.531) (1.643,0.487) (1.679,0.491) (1.716,0.544) (1.752,0.637) (1.789,0.750) (1.825,0.864) (1.862,0.955) (1.898,1.005) (1.935,1.006) (1.971,0.958) (2.008,0.869) (2.044,0.758) (2.081,0.647) (2.117,0.558) (2.154,0.508) (2.190,0.509) (2.227,0.558) (2.263,0.648) (2.300,0.760)}

\def\bandL{-0.40} \def\bandR{13.30}   
\def\bandB{-1.20} \def\bandT{1.55}    
\def\bumpL{10.30} \def\bumpT{3.40}    

\begin{scope}[on background layer]
  \path[fill=cN!10, rounded corners=8pt]
    (\bandL,\bandB) -- (\bandL,\bandT) -- (\bumpL,\bandT) --
    (\bumpL,\bumpT) -- (\bandR,\bumpT) -- (\bandR,\bandB) -- cycle;
\end{scope}

\begin{scope}[shift={(0.0,2.15)}]
  \draw[panel] (0,0) rectangle (2.3,1.05);
  \draw[sigline,draw=cS!85] plot coordinates {\sigAa};
  \draw[sigline,draw=cS!50] plot coordinates {\sigAb};
  \draw[sigline,draw=cS!28] plot coordinates {\sigAc};
\end{scope}

\begin{scope}[shift={(3.55,2.15)}]
  \draw[panel] (0,0) rectangle (2.3,1.05);
  \draw[sigline,draw=cS!85] plot coordinates {\sigBa};
  \draw[sigline,draw=cS!50] plot coordinates {\sigBb};
  \draw[sigline,draw=cS!28] plot coordinates {\sigBc};
\end{scope}

\begin{scope}[shift={(7.1,2.15)}]
  \draw[panel] (0,0) rectangle (2.3,1.05);
  \draw[sigline,draw=cS!85] plot coordinates {\sigCa};
  \draw[sigline,draw=cS!50] plot coordinates {\sigCb};
  \draw[sigline,draw=cS!28] plot coordinates {\sigCc};
\end{scope}

\begin{scope}[shift={(10.65,2.15)}]
  \draw[panel] (0,0) rectangle (2.3,1.05);
  \draw[sigline,draw=cS!85] plot coordinates {\sigDa};
  \draw[sigline,draw=cS!50] plot coordinates {\sigDb};
  \draw[sigline,draw=cS!28] plot coordinates {\sigDc};
\end{scope}

\foreach \cx in {1.15,4.70,8.25,11.80}{%
  \begin{scope}[shift={(\cx,0.45)}]
    \draw[planeb] (-0.53,-0.58) rectangle (0.97,1.02);
    \draw[planem] (-0.64,-0.69) rectangle (0.86,0.91);
  \end{scope}}

\node[coefbox] (c0) at (1.15,0.45)  {};
\node[coefbox] (c1) at (4.70,0.45)  {};
\node[coefbox] (c2) at (8.25,0.45)  {};
\node[coefbox] (c3) at (11.80,0.45) {};

\begin{scope}[shift={(1.15,0.45)}]
  \draw[line width=0.25pt, draw=cN!22] (-0.63,0.6) -- (0.63,0.6);
  \draw[line width=0.25pt, draw=cN!22] (-0.63,0.3) -- (0.63,0.3);
  \draw[line width=0.25pt, draw=cN!22] (-0.63,0.0) -- (0.63,0.0);
  \draw[line width=0.25pt, draw=cN!22] (-0.63,-0.6) -- (0.63,-0.6);
  \draw[line width=0.55pt, draw=cS] (-0.420,0.6) -- (-0.420,+0.499);
  \draw[line width=0.55pt, draw=cS] (0.000,0.6) -- (0.000,+0.540);
  \draw[line width=0.55pt, draw=cS] (0.420,0.6) -- (0.420,+0.557);
  \draw[line width=0.55pt, draw=cS] (-0.420,0.3) -- (-0.420,+0.224);
  \draw[line width=0.55pt, draw=cS] (0.000,0.3) -- (0.000,+0.214);
  \draw[line width=0.55pt, draw=cS] (0.420,0.3) -- (0.420,+0.309);
  \draw[line width=0.55pt, draw=cS] (-0.525,0.0) -- (-0.525,+0.064);
  \draw[line width=0.55pt, draw=cS] (-0.315,0.0) -- (-0.315,-0.028);
  \draw[line width=0.55pt, draw=cS] (-0.105,0.0) -- (-0.105,-0.057);
  \draw[line width=0.55pt, draw=cS] (0.105,0.0) -- (0.105,+0.038);
  \draw[line width=0.55pt, draw=cS] (0.315,0.0) -- (0.315,+0.056);
  \draw[line width=0.55pt, draw=cS] (0.525,0.0) -- (0.525,-0.033);
  \draw[line width=0.55pt, draw=cS] (-0.578,-0.6) -- (-0.578,-0.553);
  \draw[line width=0.55pt, draw=cS] (-0.473,-0.6) -- (-0.473,-0.530);
  \draw[line width=0.55pt, draw=cS] (-0.367,-0.6) -- (-0.367,-0.626);
  \draw[line width=0.55pt, draw=cS] (-0.263,-0.6) -- (-0.263,-0.660);
  \draw[line width=0.55pt, draw=cS] (-0.158,-0.6) -- (-0.158,-0.564);
  \draw[line width=0.55pt, draw=cS] (-0.052,-0.6) -- (-0.052,-0.571);
  \draw[line width=0.55pt, draw=cS] (0.052,-0.6) -- (0.052,-0.528);
  \draw[line width=0.55pt, draw=cS] (0.157,-0.6) -- (0.157,-0.679);
  \draw[line width=0.55pt, draw=cS] (0.262,-0.6) -- (0.262,-0.653);
  \draw[line width=0.55pt, draw=cS] (0.367,-0.6) -- (0.367,-0.661);
  \draw[line width=0.55pt, draw=cS] (0.473,-0.6) -- (0.473,-0.695);
  \draw[line width=0.55pt, draw=cS] (0.578,-0.6) -- (0.578,-0.580);
  \fill[cN!70] (0,-0.22) circle (0.5pt);
  \fill[cN!70] (0,-0.3) circle (0.5pt);
  \fill[cN!70] (0,-0.38) circle (0.5pt);
\end{scope}

\begin{scope}[shift={(4.7,0.45)}]
  \draw[line width=0.25pt, draw=cN!22] (-0.63,0.6) -- (0.63,0.6);
  \draw[line width=0.25pt, draw=cN!22] (-0.63,0.3) -- (0.63,0.3);
  \draw[line width=0.25pt, draw=cN!22] (-0.63,0.0) -- (0.63,0.0);
  \draw[line width=0.25pt, draw=cN!22] (-0.63,-0.6) -- (0.63,-0.6);
  \draw[line width=0.55pt, draw=cS] (-0.420,0.6) -- (-0.420,+0.514);
  \draw[line width=0.55pt, draw=cS] (0.000,0.6) -- (0.000,+0.519);
  \draw[line width=0.55pt, draw=cS] (0.420,0.6) -- (0.420,+0.656);
  \draw[line width=0.55pt, draw=cS] (-0.420,0.3) -- (-0.420,+0.216);
  \draw[line width=0.55pt, draw=cS] (0.000,0.3) -- (0.000,+0.249);
  \draw[line width=0.55pt, draw=cS] (0.420,0.3) -- (0.420,+0.338);
  \draw[line width=0.55pt, draw=cS] (-0.525,0.0) -- (-0.525,+0.027);
  \draw[line width=0.55pt, draw=cS] (-0.315,0.0) -- (-0.315,-0.028);
  \draw[line width=0.55pt, draw=cS] (-0.105,0.0) -- (-0.105,-0.035);
  \draw[line width=0.55pt, draw=cS] (0.105,0.0) -- (0.105,-0.014);
  \draw[line width=0.55pt, draw=cS] (0.315,0.0) -- (0.315,+0.049);
  \draw[line width=0.55pt, draw=cS] (0.525,0.0) -- (0.525,+0.026);
  \draw[line width=0.55pt, draw=cS] (-0.578,-0.6) -- (-0.578,-0.555);
  \draw[line width=0.55pt, draw=cS] (-0.473,-0.6) -- (-0.473,-0.551);
  \draw[line width=0.55pt, draw=cS] (-0.367,-0.6) -- (-0.367,-0.617);
  \draw[line width=0.55pt, draw=cS] (-0.263,-0.6) -- (-0.263,-0.638);
  \draw[line width=0.55pt, draw=cS] (-0.158,-0.6) -- (-0.158,-0.575);
  \draw[line width=0.55pt, draw=cS] (-0.052,-0.6) -- (-0.052,-0.576);
  \draw[line width=0.55pt, draw=cS] (0.052,-0.6) -- (0.052,-0.542);
  \draw[line width=0.55pt, draw=cS] (0.157,-0.6) -- (0.157,-0.664);
  \draw[line width=0.55pt, draw=cS] (0.262,-0.6) -- (0.262,-0.641);
  \draw[line width=0.55pt, draw=cS] (0.367,-0.6) -- (0.367,-0.642);
  \draw[line width=0.55pt, draw=cS] (0.473,-0.6) -- (0.473,-0.674);
  \draw[line width=0.55pt, draw=cS] (0.578,-0.6) -- (0.578,-0.585);
  \fill[cN!70] (0,-0.22) circle (0.5pt);
  \fill[cN!70] (0,-0.3) circle (0.5pt);
  \fill[cN!70] (0,-0.38) circle (0.5pt);
\end{scope}

\begin{scope}[shift={(8.25,0.45)}]
  \draw[line width=0.25pt, draw=cN!22] (-0.63,0.6) -- (0.63,0.6);
  \draw[line width=0.25pt, draw=cN!22] (-0.63,0.3) -- (0.63,0.3);
  \draw[line width=0.25pt, draw=cN!22] (-0.63,0.0) -- (0.63,0.0);
  \draw[line width=0.25pt, draw=cN!22] (-0.63,-0.6) -- (0.63,-0.6);
  \draw[line width=0.55pt, draw=cS] (-0.420,0.6) -- (-0.420,+0.531);
  \draw[line width=0.55pt, draw=cS] (0.000,0.6) -- (0.000,+0.501);
  \draw[line width=0.55pt, draw=cS] (0.420,0.6) -- (0.420,+0.701);
  \draw[line width=0.55pt, draw=cS] (-0.420,0.3) -- (-0.420,+0.208);
  \draw[line width=0.55pt, draw=cS] (0.000,0.3) -- (0.000,+0.330);
  \draw[line width=0.55pt, draw=cS] (0.420,0.3) -- (0.420,+0.355);
  \draw[line width=0.55pt, draw=cS] (-0.525,0.0) -- (-0.525,-0.043);
  \draw[line width=0.55pt, draw=cS] (-0.315,0.0) -- (-0.315,-0.027);
  \draw[line width=0.55pt, draw=cS] (-0.105,0.0) -- (-0.105,+0.014);
  \draw[line width=0.55pt, draw=cS] (0.105,0.0) -- (0.105,-0.046);
  \draw[line width=0.55pt, draw=cS] (0.315,0.0) -- (0.315,+0.042);
  \draw[line width=0.55pt, draw=cS] (0.525,0.0) -- (0.525,+0.055);
  \draw[line width=0.55pt, draw=cS] (-0.578,-0.6) -- (-0.578,-0.558);
  \draw[line width=0.55pt, draw=cS] (-0.473,-0.6) -- (-0.473,-0.582);
  \draw[line width=0.55pt, draw=cS] (-0.367,-0.6) -- (-0.367,-0.597);
  \draw[line width=0.55pt, draw=cS] (-0.263,-0.6) -- (-0.263,-0.590);
  \draw[line width=0.55pt, draw=cS] (-0.158,-0.6) -- (-0.158,-0.593);
  \draw[line width=0.55pt, draw=cS] (-0.052,-0.6) -- (-0.052,-0.582);
  \draw[line width=0.55pt, draw=cS] (0.052,-0.6) -- (0.052,-0.559);
  \draw[line width=0.55pt, draw=cS] (0.157,-0.6) -- (0.157,-0.646);
  \draw[line width=0.55pt, draw=cS] (0.262,-0.6) -- (0.262,-0.627);
  \draw[line width=0.55pt, draw=cS] (0.367,-0.6) -- (0.367,-0.612);
  \draw[line width=0.55pt, draw=cS] (0.473,-0.6) -- (0.473,-0.649);
  \draw[line width=0.55pt, draw=cS] (0.578,-0.6) -- (0.578,-0.590);
  \fill[cN!70] (0,-0.22) circle (0.5pt);
  \fill[cN!70] (0,-0.3) circle (0.5pt);
  \fill[cN!70] (0,-0.38) circle (0.5pt);
\end{scope}

\begin{scope}[shift={(11.8,0.45)}]
  \draw[line width=0.25pt, draw=cN!22] (-0.63,0.6) -- (0.63,0.6);
  \draw[line width=0.25pt, draw=cN!22] (-0.63,0.3) -- (0.63,0.3);
  \draw[line width=0.25pt, draw=cN!22] (-0.63,0.0) -- (0.63,0.0);
  \draw[line width=0.25pt, draw=cN!22] (-0.63,-0.6) -- (0.63,-0.6);
  \draw[line width=0.55pt, draw=cS] (-0.420,0.6) -- (-0.420,+0.548);
  \draw[line width=0.55pt, draw=cS] (0.000,0.6) -- (0.000,+0.488);
  \draw[line width=0.55pt, draw=cS] (0.420,0.6) -- (0.420,+0.730);
  \draw[line width=0.55pt, draw=cS] (-0.420,0.3) -- (-0.420,+0.202);
  \draw[line width=0.55pt, draw=cS] (0.000,0.3) -- (0.000,+0.368);
  \draw[line width=0.55pt, draw=cS] (0.420,0.3) -- (0.420,+0.368);
  \draw[line width=0.55pt, draw=cS] (-0.525,0.0) -- (-0.525,-0.070);
  \draw[line width=0.55pt, draw=cS] (-0.315,0.0) -- (-0.315,-0.027);
  \draw[line width=0.55pt, draw=cS] (-0.105,0.0) -- (-0.105,+0.041);
  \draw[line width=0.55pt, draw=cS] (0.105,0.0) -- (0.105,-0.065);
  \draw[line width=0.55pt, draw=cS] (0.315,0.0) -- (0.315,+0.034);
  \draw[line width=0.55pt, draw=cS] (0.525,0.0) -- (0.525,+0.074);
  \draw[line width=0.55pt, draw=cS] (-0.578,-0.6) -- (-0.578,-0.560);
  \draw[line width=0.55pt, draw=cS] (-0.473,-0.6) -- (-0.473,-0.632);
  \draw[line width=0.55pt, draw=cS] (-0.367,-0.6) -- (-0.367,-0.583);
  \draw[line width=0.55pt, draw=cS] (-0.263,-0.6) -- (-0.263,-0.560);
  \draw[line width=0.55pt, draw=cS] (-0.158,-0.6) -- (-0.158,-0.618);
  \draw[line width=0.55pt, draw=cS] (-0.052,-0.6) -- (-0.052,-0.590);
  \draw[line width=0.55pt, draw=cS] (0.052,-0.6) -- (0.052,-0.579);
  \draw[line width=0.55pt, draw=cS] (0.157,-0.6) -- (0.157,-0.625);
  \draw[line width=0.55pt, draw=cS] (0.262,-0.6) -- (0.262,-0.605);
  \draw[line width=0.55pt, draw=cS] (0.367,-0.6) -- (0.367,-0.568);
  \draw[line width=0.55pt, draw=cS] (0.473,-0.6) -- (0.473,-0.615);
  \draw[line width=0.55pt, draw=cS] (0.578,-0.6) -- (0.578,-0.596);
  \fill[cN!70] (0,-0.22) circle (0.5pt);
  \fill[cN!70] (0,-0.3) circle (0.5pt);
  \fill[cN!70] (0,-0.38) circle (0.5pt);
\end{scope}

\draw[vertd] (1.15,1.58)  -- (1.15,2.08);
\draw[vertd] (4.70,1.58)  -- (4.70,2.08);
\draw[vertd] (8.25,1.58)  -- (8.25,2.08);
\draw[vert]  (11.80,1.58) -- (11.80,2.08);
\node[tag,anchor=east] at (1.08,1.83)  {IDWT};
\node[tag,anchor=east] at (4.63,1.83)  {IDWT};
\node[tag,anchor=east] at (8.18,1.83)  {IDWT};
\node[tag,anchor=east] at (11.73,1.83) {IDWT};

\node[font=\small, text=cN, inner sep=2pt] (e1) at (2.925,0.45)  {$\cdots$};
\node[font=\small, text=cN, inner sep=2pt] (e2) at (10.025,0.45) {$\cdots$};

\draw[flow] (2.15,0.45) -- (e1);
\draw[flow] (e1) -- (c1);
\draw[flow] (5.70,0.45) -- node[tag,above,align=center,pos=0.44]{learned \\ velocity\\[-1pt] $\vth(z_{t_k},t_k)$} (c2);
\draw[flow] (9.25,0.45) -- (e2);
\draw[flow] (e2) -- (c3);

\node[font=\small, text=cN] at (1.15,-0.58)  {$z_0$};
\node[font=\small, text=cN] at (4.70,-0.58)  {$z_{t_k}$};
\node[font=\small, text=cN] at (8.25,-0.58)  {$z_{t_{k+1}}$};
\node[font=\small, text=cN] at (11.80,-0.58) {$z_1$};
\node[ax, anchor=north] at (1.15,-0.78)  {pure noise};
\node[ax, anchor=north] at (11.80,-0.78) {generated sample};

\node[tag,anchor=east] at (0.33,1.05)  {$a_J$};
\node[tag,anchor=east] at (0.33,0.75)  {$d_J$};
\node[tag,anchor=east] at (0.33,0.45)  {$d_{J-1}$};
\node[tag,anchor=east] at (0.33,-0.15) {$d_1$};

\draw[decorate, decoration={brace, amplitude=3.5pt, raise=1pt},
      line width=0.4pt, draw=cN!85] (2.12,-0.13) -- (1.90,-0.35);
\node[font=\tiny, text=cN, anchor=west, inner sep=1pt] at (2.15,-0.36) {$F$ channels};

\node[ax, anchor=east] at (-0.60,2.68) {time domain};
\node[ax, anchor=east] at (-0.60,0.45) {wavelet domain};

\end{tikzpicture}

%% file: figures/figure_architecture.tex
\begin{tikzpicture}[
  io/.style   = {draw=cN!65, fill=cN!8, rounded corners=2pt, line width=0.4pt,
                 align=center, font=\scriptsize, inner sep=3pt, minimum height=6.5mm},
  mod/.style  = {draw=cS!65, fill=cS!8, rounded corners=2pt, line width=0.4pt,
                 align=center, font=\scriptsize, inner sep=3pt, minimum height=6.0mm},
  modC/.style = {draw=cS!65, fill=cS!8, rounded corners=2pt, line width=0.4pt,
                 align=center, font=\scriptsize, inner sep=2.2pt, minimum height=5.2mm},
  emb/.style  = {draw=cC!65, fill=cC!10, rounded corners=2pt, line width=0.4pt,
                 align=center, font=\scriptsize, inner sep=3pt, minimum height=6.5mm},
  ada/.style  = {draw=cA!70, fill=cA!12, rounded corners=2pt, line width=0.4pt,
                 align=center, font=\scriptsize, inner sep=3pt, minimum height=6.5mm},
  opC/.style  = {circle, draw=cC!70, fill=cC!10, line width=0.4pt, inner sep=0pt,
                 minimum size=4mm, font=\tiny},
  opS/.style  = {circle, draw=cS!70, fill=cS!8, line width=0.4pt, inner sep=0pt,
                 minimum size=3.0mm, font=\tiny},
  a/.style    = {-{Stealth[length=1.7mm]}, line width=0.4pt, draw=cN!80},
  ac/.style   = {-{Stealth[length=1.7mm]}, line width=0.4pt, draw=cC!75},
  as/.style   = {-{Stealth[length=1.7mm]}, line width=0.4pt, draw=cS!70},
  ta/.style   = {-{Stealth[length=1.7mm]}, line width=0.4pt, draw=cA!85},
  lab/.style  = {font=\scriptsize, text=cN, inner sep=1.5pt},
  capt/.style = {font=\small, inner sep=2pt}
]


\node[io, minimum width=30mm] (a1) at (1.9,0.00)
  {coefficients $z\in\R^{T\times F}$};

\node[emb, minimum width=30mm] (a2) at (1.9,-1.05)
  {token construction (b)};

\node[mod, minimum width=30mm] (a3) at (1.9,-2.10)
  {transformer block (c)};

\node[ada, minimum width=30mm] (a4) at (1.9,-3.15)
  {adaptive read-out};

\node[emb, minimum width=30mm] (a5) at (1.9,-4.20)
  {per-level heads $\mathbf{H}_j$};

\node[io, minimum width=30mm] (a6) at (1.9,-5.25)
  {velocity $\vth(z,t)\in\R^{T\times F}$};

\draw[a] (a1) -- (a2);
\draw[a] (a2) -- (a3);
\draw[a] (a3) -- (a4);
\draw[a] (a4) -- (a5);
\draw[a] (a5.south) -- (a6.north);

\draw[as]
  (a3.south east)
  .. controls +(0.43,-0.07) and +(0.43,0.07) ..
  (a3.north east);

\node[lab, anchor=west] at (3.70,-2.10) {$\times B$};

\draw[ta] (0.10,-2.10) -- (0.35,-2.10);
\draw[ta] (0.10,-3.15) -- (0.35,-3.15);

\node[lab, anchor=east, text=cA] at (0.05,-2.10) {$t$};
\node[lab, anchor=east, text=cA] at (0.05,-3.15) {$t$};

\node[capt] at (1.9,-6.15) {(a) velocity network};


\foreach \x/\h/\lb in {
  5.6/0.26/{$z_{I_{J+1},f}$},
  7.0/0.26/{$z_{I_J,f}$},
  9.2/0.50/{$z_{I_2,f}$},
  10.6/1.00/{$z_{I_1,f}$}
}{
  \draw[draw=cC!32, fill=cC!16, line width=0.30pt]
       (\x-0.08,-0.76) rectangle (\x+0.36,-0.76+\h);

  \draw[draw=cC!50, fill=cC!30, line width=0.32pt]
       (\x-0.15,-0.83) rectangle (\x+0.29,-0.83+\h);

  \draw[draw=cC!70, fill=cC!45, line width=0.35pt]
       (\x-0.22,-0.90) rectangle (\x+0.22,-0.90+\h);

  \node[lab] at (\x,-1.17) {\lb};
}

\node[lab, text=cN] at (8.1,-0.75) {$\cdots$};

\draw[
  decorate,
  decoration={brace, amplitude=3pt, raise=1pt},
  line width=0.35pt,
  draw=cN!85
]
  (8.95,-0.48) -- (9.17,-0.22);

\node[font=\tiny, text=cN, anchor=south east, inner sep=1pt]
  at (9.00,-0.18) {$F$ channels};

\node[lab, anchor=south] at (5.6,-0.30)
  {heights $\propto L_j$};

\foreach \x/\nm/\j in {
  5.6/EA/{J+1},
  7.0/EB/{J},
  9.2/EC/{2},
  10.6/ED/{1}
}{
  \node[emb, minimum width=13mm] (\nm) at (\x,-1.85)
       {$\mathbf{W}^{\mathrm{emb}}_{\j}$};

  \draw[ac] (\x,-1.33) -- (\x,-1.54);
}

\node[lab, text=cN] at (8.1,-1.85) {$\cdots$};

\node[opC] (sum) at (8.1,-2.85) {$+$};

\foreach \x in {5.6,7.0,9.2,10.6}{
  \draw[ac] (\x,-2.17) -- (sum);
}

\node[emb, minimum width=20mm] (ef) at (10.8,-2.85)
  {channel id $e_f$};

\draw[ac] (ef) -- (sum);

\node[emb, minimum width=32mm] (tok) at (8.1,-4.00)
  {token $u^{(0)}_f\in\R^{d_{\mathrm{m}}}$};

\draw[ac] (sum) -- (tok);

\node[emb, minimum width=44mm] (stk) at (8.1,-5.25)
  {$u^{(0)}=\bigl(u^{(0)}_1,\dots,u^{(0)}_F\bigr)^{\!\top}
   \in\R^{F\times d_{\mathrm{m}}}$};

\draw[ac]
  (tok) --
  node[lab,right,inner sep=3pt] {stack over $f=1,\dots,F$}
  (stk);

\node[capt] at (8.1,-6.15) {(b) token construction};


\def\xc{14.1}

\node[modC, minimum width=24mm] (b0) at (\xc,0.00)
  {$u^{(i)}\in\R^{F\times d_{\mathrm{m}}}$};

\node[modC, minimum width=24mm] (b1) at (\xc,-0.82)
  {AdaLN-Zero};

\node[modC, minimum width=24mm] (b2) at (\xc,-1.68)
  {self-attention\\over $F$ tokens};

\node[opS] (b3) at (\xc,-2.46) {$+$};

\node[modC, minimum width=24mm] (b4) at (\xc,-3.22)
  {AdaLN-Zero};

\node[modC, minimum width=24mm] (b5) at (\xc,-4.02)
  {MLP};

\node[opS] (b6) at (\xc,-4.76) {$+$};

\node[modC, minimum width=24mm] (b7) at (\xc,-5.45)
  {$u^{(i+1)}$};

\draw[as] (b0) -- (b1);
\draw[as] (b1) -- (b2);
\draw[as] (b2) -- node[lab,right] {$g_1$} (b3);
\draw[as] (b3) -- (b4);
\draw[as] (b4) -- (b5);
\draw[as] (b5) -- node[lab,right] {$g_2$} (b6);
\draw[as] (b6) -- (b7);

\def\xskip{15.70}

\draw[as]
  (b0.east)
  -- (\xskip,0.00)
  -- (\xskip,-2.46)
  -- (b3.east);

\draw[as]
  (b3.east)
  -- (\xskip,-2.46)
  -- (\xskip,-4.76)
  -- (b6.east);

\draw[ta] ($(b1.west)+(-0.36,0)$) -- (b1.west);
\node[lab, anchor=east, text=cA]
  at ($(b1.west)+(-0.41,0)$) {$t$};

\draw[ta] ($(b4.west)+(-0.36,0)$) -- (b4.west);
\node[lab, anchor=east, text=cA]
  at ($(b4.west)+(-0.41,0)$) {$t$};

\node[capt] at (\xc,-6.15)
  {(c) transformer block};

\end{tikzpicture}